\documentclass[review]{elsarticle}

\usepackage{graphicx}
\usepackage{amssymb}
\usepackage{amsfonts}
\usepackage{algorithm}
\usepackage{algorithmic}
\usepackage{theorem}

\newtheorem{theorem}{Theorem}
\newtheorem{corollary}{Corollary}
\newtheorem{lemma}{Lemma}

\newtheorem{definition}{Definition}
\newtheorem{proposition}{Proposition}
\newtheorem{example}{Example}

\newenvironment{proof}{\textbf{Proof.}\ }{\hspace{\stretch{1}}$\square$\\}

\journal{Artificial Intelligence}

\begin{document}

\begin{frontmatter}

\title{Theoretical Foundation of Co-Training and Disagreement-Based Algorithms}

\author{Wei Wang}

\author{Zhi-Hua Zhou\corref{mycorrespondingauthor}}
\cortext[mycorrespondingauthor]{Corresponding author}
\ead{zhouzh@nju.edu.cn}

\address{National Key Laboratory for Novel Software Technology,\\
        Nanjing University, Nanjing 210023, China}

\begin{abstract}
Disagreement-based approaches generate multiple classifiers and exploit the disagreement among them with unlabeled data to improve learning performance. Co-training is a representative paradigm of them, which trains two classifiers separately on two sufficient and redundant views; while for the applications where there is only one view, several successful variants of co-training with two different classifiers on single-view data instead of two views have been proposed. For these disagreement-based approaches, there are several important issues which still are unsolved, in this article we present theoretical analyses to address these issues, which provides a theoretical foundation of co-training and disagreement-based approaches.
\end{abstract}

\begin{keyword}
machine learning \sep semi-supervised learning \sep disagreement-based learning \sep co-training \sep multi-view classification \sep combination


\end{keyword}

\end{frontmatter}

\section{Introduction}\label{sec:Introduction}
Learning from labeled training data is well-established in traditional machine learning, but labeling the data is time-consuming, sometimes may be very expensive since it requires human efforts. In many practical applications, unlabeled data can be obtained abundantly and cheaply. For example, in the task of web page classification, it is easy to get abundant unlabeled web pages in the Internet, while a few labeled ones are available since obtaining the labels requires human interaction. Many semi-supervised learning approaches which exploit unlabeled data to complement labeled data for improving learning performance have been developed. Those approaches can be roughly categorized into four classes, i.e., generative approaches, S3VMs
(Semi-Supervised Support Vector Machines), graph-based approaches and
disagreement-based approaches. Generative approaches use a generative model and typically employ EM to model the label estimation or parameter estimation process \cite{Miller97,NigamMTM00,Shahshahani:Landgrebe1994}. S3VMs use unlabeled data to adjust the SVM decision boundary learned from labeled data such that it goes through the less dense region while keeping the labeled data classified correctly \cite{Chapelle:Zien2005,Joachims99,Sindhwani:Keerthi2006}. Graph-based approaches define a graph on the training data and enforce the label smoothness over the graph as a regularization term
\cite{Belkin:Niyogi:Sindhwani2006,Blum:Chawla2001,Zhou:Bousquet:Lal:Weston:Scholkopf2004,Zhu03}.
Disagreement-based approaches generate multiple classifiers and exploit the
disagreement among them with unlabeled data, i.e., letting the multiple
classifiers label unlabeled instances to augment the training data \cite{Blum:Mitchell1998,Goldman:2000,Zhou:Li2005tkde,Zhou:Li2009}.

Research on disagreement-based approaches started from Blum and Mitchell's seminal work of co-training \cite{Blum:Mitchell1998}, which is a representative paradigm of
disagreement-based approaches. When co-training was proposed, Blum and Mitchell assumed that there exist two \textit{views} (i.e., two disjoint sets of features), each of which is sufficient for learning the target concept. For example, the web page classification task has two views, i.e., the text appearing on the page itself and the anchor text attached to hyper-links pointing to this page \cite{Blum:Mitchell1998}; the speech recognition task also has two views, i.e., sound and lip motion \cite{DeSa98}. Co-training learns two classifiers with initial labeled data on the two views respectively and lets them label unlabeled instances for each other to augment the training data. Unfortunately, in real-world applications, the requirement of two views is hard to satisfy. Although Nigam and Ghani \cite{Nigam:Ghani2000} have shown that a feature split can be used to enable co-training to work when there are many redundant features, it is more desirable to develop algorithms that can be applied to single-view data. Several successful variants of co-training have been proposed along this direction. For example, Goldman and Zhou \cite{Goldman:2000} proposed a method which generates two classifiers by using two different learning algorithms that can partition the example space into a number of equivalence classes; Zhou and Li \cite{Zhou:Li2007TKDE} proposed a semi-supervised regression method which generates two regressors by using different parameter configurations for the same learning algorithm. Different relevant algorithms have been developed with different names and the name disagreement-based semi-supervised learning was coined \cite{Zhou:Li2009,ZhouZH2008} to reflect the fact that co-training and its variants are actually in the same family, and the key for the learning process to proceed is to maintain a large disagreement among the classifiers\footnote{If there is no disagreement among the classifiers, the learning process would degenerate into self-training.}. Co-training \cite{Blum:Mitchell1998} is the famous algorithm which relies on two views, while the algorithms which rely on multiple classifiers generated from single-view data is referred to as single-view disagreement-based approaches. In disagreement-based approaches, multiple classifiers are trained for the same task and the disagreement among them is exploited during the learning process. Here, unlabeled data serve as a kind of ``platform'' for information exchange. If one classifier is much more confident on a disagreed unlabeled instance than other classifier(s), then this classifier will teach other(s) with this instance\footnote{In real-world applications, the disagreement-based approaches may consist of more than two classifiers, in this article we mainly focus on the two-classifier setting.}. It does not matter where these classifiers come from, they can be trained on multi-view data with the same learning algorithm or on single-view data with different learning algorithms. The disagreement-based algorithms have achieved success in many domains such as natural language processing \cite{Hwa2003,Sarkar2001,Steedman2003} and image retrieval \cite{Zhou:Chen:Dai2006,Wang:Zhou2008}.

There is another famous semi-supervised learning approach called \textit{co-regularization}
\cite{Brefeld2006,Farquhar2006,Sindhwani:Niyogi:Belkin2005}, which also exploits unlabeled data with two views. It directly minimizes the error rate on labeled data and the disagreement on unlabeled data over two views with the intuition that the optimal classifiers in the two views are compatible with each other. It is worth noting that co-training exploits unlabeled data very differently from co-regularization, and no pseudo-labels are assigned to unlabeled instances in co-regularization.

By comparing these discriminative semi-supervised learning approaches, it can be found that unlabeled data help in two distinct ways. One is starting with unlabeled data to generate a constraint (regularization) and then learning a classifier with labeled data and the constraint (regularization), i.e., S3VMs, graph-based approaches and co-regularization. In S3VMs, unlabeled data are exploited as a constraint such that the SVM decision boundary goes through the less dense region; in graph-based approaches, unlabeled data are exploited to construct a graph Laplacian regularization; while in co-regularization, unlabeled data are exploited to generate a regularization of disagreement over two views. Balcan and Blum \cite{Balcan:Blum2010} provided a unified framework for these approaches to interpret why and when unlabeled data can help, in which they defined a notion of compatibility and assumed that the target concept should have high compatibility with the underlying data distribution. Unlabeled data are exploited to estimate the
compatibility of all hypotheses and the size of hypothesis space could be reduced by removing the hypotheses which have low compatibility. Then labeled data is used to find a good hypothesis within the reduced hypothesis space, which will lead to good sample-size bounds. The other is starting with labeled data to generate multiple weak classifiers and then letting them label unlabeled instances to augment the training data, i.e., disagreement-based approaches. There has been a long-term theoretical study on this.
When Blum and Mitchell \cite{Blum:Mitchell1998} proposed co-training, they proved that when the two views are conditionally independent, co-training can boost the performance of weak classifiers to arbitrarily high by exploiting unlabeled data. Dasgupta et al.
\cite{Dasgupta:Littman:McAllester2002} analyzed the generalization bound for co-training with two conditionally independent views
and showed that the error rate of co-training is bounded by the disagreement between two co-trained classifiers. To relax the conditional independence assumption, Abney \cite{Abney2002} found that weak dependence can also lead to successful co-training. Later, Balcan et al. \cite{Balcan:Blum:Yang2005} pointed out that if a PAC classifier can be obtained on each view, the conditional independence assumption or even weak dependence assumption is unnecessary; a weaker assumption of ``expansion'' over the underlying data distribution is sufficient for co-training to succeed. However, all these results focus on co-training which relies on two views, there are several important issues on disagreement-based approaches which still are unsolved.

\subsection{Our Focus and Main Results}\label{sec:MainResults}

We present a theoretical foundation of co-training and disagreement-based approaches to address the unsolved issues in this article. The issues and results can be summarized as follows. (1) One basic issue is why and when the disagreement-based approaches can improve learning performance by exploiting unlabeled data. We present a general analysis, which shows that if the two initial classifiers trained with the initial labeled data have large disagreement, the disagreement-based approaches can improve learning performance (Section~\ref{sec:upper-bound}); (2) it is often observed that the performance of the classifiers in disagreement-based approaches can not be improved further after a number of rounds in empirical studies. Up to now, there is no theoretical explanation to this. We prove that the disagreement and the error rates of the classifiers will converge after a number of rounds, which theoretically explains why the classifiers can not be improved further (Section~\ref{sec:no-further-improvement}); (3) all previous theoretical analyses focused on the sufficient condition, so a fundamental issue may arise: what is the sufficient and necessary condition for co-training to succeed? To the best of our knowledge, this has not been touched. We present a theoretical graph-based analysis on co-training, based on which we get the sufficient and necessary condition for co-training to succeed (Section~\ref{sec:sufficient-necessary}); (4) all previous theoretical analyses assumed that each view is sufficient for learning the target concept. So another issue may arise: what can co-training do with insufficient views? We present a theoretical analysis on co-training with insufficient views which is much more challenging but practical, especially when the two views provide diverse information (Section~\ref{sec:insufficient-co-training}); (5) the classifiers in disagreement-based approaches are usually combined to make predictions, unfortunately, there is no theoretical analysis about this. We study margin-based classifiers and present a theoretical analysis to explain why and when the combination can be better than the individual classifiers (Section~\ref{sec:combination}).

\subsection{Organization}\label{sec:Organization}
The rest of this article is organized as follows. We present a general analysis on disagreement-based approaches to explain why they can improve learning performance by exploiting unlabeled data in Section~\ref{sec:upper-bound} and why the learning performance can not be improved further after a number of rounds in Section~\ref{sec:no-further-improvement}. In Section~\ref{sec:sufficient-necessary}, we study the sufficient and necessary condition and give an interesting implication, such as combination of weight matrices. We analyze co-training with insufficient views in Section~\ref{sec:insufficient-co-training}, and analyze when the combination can be better than the individual classifiers in Section~\ref{sec:combination}. Finally, we make a conclusion in Section~\ref{sec:conclusions}.

\section{General Analysis on Disagreement-Based Approaches}\label{sec:general-analysis}
Given the labeled data $L$, unlabeled data $U$ and two hypothesis spaces $\mathcal{H}_{1}$ and $\mathcal{H}_{2}$, in this section we consider the following disagreement-based process whose pseudo-codes
are in Algorithm~\ref{alg:co-training-process}.

\noindent \textbf{Disagreement-Based Process:} \textit{Given the labeled data $L$ and unlabeled data $U$, at first, we train two
initial classifiers $h_{1}^{0} \in {\mathcal{H}}_{1}$ and $h_{2}^{0} \in
{\mathcal{H}}_{2}$ using $L$ which contains $l$ labeled examples with two
different learning algorithms (if the data have two views, we can train two
classifiers $h_{1}^{0} \in {\mathcal{H}}_{1}$ and $h_{2}^{0} \in
{\mathcal{H}}_{2}$ using $L$ in the two views with the
same learning algorithm, respectively). Then, $h_{1}^{0}$ selects $u$ unlabeled instances from $U$ to label and puts these newly labeled examples
into $\sigma_{2}$ which contains the initial labeled examples in $L$; at
the same time, $h_{2}^{0}$ selects $u$ unlabeled instances from $U$ to label and puts these newly labeled examples into $\sigma_{1}$
which contains the initial labeled examples in $L$. Then, $h_{1}^{1} \in
{\mathcal{H}}_{1}$ and $h_{2}^{1} \in {\mathcal{H}}_{2}$ are trained with $\sigma_{1}$
and $\sigma_{2}$, respectively. After that, $h_{1}^{1}$ selects $u$ unlabeled
instances to label and uses these newly labeled examples to update $\sigma_{2}$; while
$h_{2}^{1}$ also selects $u$ unlabeled instances to label and uses these
newly labeled examples to update $\sigma_{1}$. Such a process is repeated for a pre-set
number of learning round.}

\begin{algorithm}[t]
\caption{Disagreement-Based Process}\label{alg:co-training-process}
\begin{algorithmic}
   \STATE {\bfseries Input:} Labeled data $L$, unlabeled data $U$, two hypothesis spaces $\mathcal{H}_{1}$ and $\mathcal{H}_{2}$, and pre-set number of learning round $s$.
   \STATE {\bfseries Output:} $h_{1}^{s}$ and $h_{2}^{s}$.
   \STATE {\bfseries Initialize:} Set $\sigma_{1}=\sigma_{2}=L$;
   \FOR{$i=0,1,\cdots,s$}
   \STATE Train a classifier $h_{v}^{i} \in {\mathcal{H}}_{v}$ ($v=1,2$) with $\sigma_{v}$ by minimizing the empirical risk;
   \STATE $h_{v}^{i}$ selects $u$ unlabeled instances from $U$ to label, then add them into $\sigma_{3-v}$ and delete them from $U$.
   \ENDFOR
\end{algorithmic}
\end{algorithm}

It is easy to see that Algorithm~\ref{alg:co-training-process} reassembles the main process of existing disagreement-based approaches including co-training\footnote{Algorithm~\ref{alg:co-training-process} is almost the same as co-training in \cite{Blum:Mitchell1998} except one place: Algorithm~\ref{alg:co-training-process} uses $L$ and the examples labeled by classifier $h_{v}$ ($v=1,2$) to
retrain classifier $h_{3-v}$, while co-training in \cite{Blum:Mitchell1998} uses $L$ and the
examples labeled by both classifiers $h_{1}$ and $h_{2}$ to retrain each of them. To
exclude the examples labeled by a classifier itself is helpful in reducing the risk of over-fitting and many recent algorithms use the paradigm described in Algorithm~\ref{alg:co-training-process}.} \cite{Blum:Mitchell1998} which requires that the data should have two views and single-view disagreement-based algorithms \cite{Goldman:2000,Zhou:Li2007TKDE,Chawla2005}. The key procedure is that one classifier labels some unlabeled instances for the other, it does not matter where the two classifiers come from. The two
classifiers can be trained on two-view data with the same learning algorithm or on single-view data with two different learning algorithms.

\subsection{Upper Bounds on Error rates of Classifiers}\label{sec:upper-bound}
Suppose that $\mathcal {X}$ is the instance space, $\mathcal {Y}=\{-1, +1\}$ is the label space, $L=\{
(x^{1}, y^{1}), \cdots, (x^{l}, y^{l}) \} \subset {\mathcal {X} \times \mathcal {Y}}$ are the labeled data, $U=\{ x^{l+1}, x^{l+2}, \cdots, x^{l+|U|} \}\\ \subset
\mathcal{X}$ are the unlabeled data. Suppose that the labeled data $L$ independently and identically come from some unknown distribution $\mathcal {D}$, whose marginal distribution on $\mathcal {X}$ is $\mathcal {D}_{\mathcal {X}}$, and the unlabeled data $U$ independently and identically come from $\mathcal {D}_{\mathcal {X}}$. $\mathcal {H}_{v}:  \mathcal{X} \to \mathcal{Y}$ ($v=1,2$) denotes the hypothesis space. Suppose that $|{\mathcal{H}}_{v}|$ is finite \footnote{If ${\mathcal{H}}_{v}$ is infinite hypothesis space with finite VC-dimension $d_{v}$, we can replace ${\mathcal {H}}_{v}$ with its $\epsilon$-cover ${\mathcal
{H}}_{v}^{\epsilon}$: a set of hypotheses ${\mathcal {H}}_{v}^{\epsilon}\subseteq
{\mathcal {H}}_{v}$ such that for any $h_{v}\in{\mathcal {H}}_{v}$ there is a $h_{v}^{\epsilon}\in{\mathcal {H}}_{v}^{\epsilon}$ satisfying
$P(h_{v}(x) \neq h_{v}^{\epsilon}(x))\leq\epsilon$. It is well-known that such an ${\mathcal
{H}}_{v}^{\epsilon}$ always exists with size $|{\mathcal
{H}}_{v}^{\epsilon}|\leq2(\frac{2e}{\epsilon}\ln\frac{2e}{\epsilon})^{d_{v}}$
\cite{Haussler1992}.} and the target concept (ground truth) $c$ which is perfectly consistent with the distribution $\mathcal {D}$ belongs to ${\mathcal{H}}_{1}$ and ${\mathcal{H}}_{2}$, and the error rate $\texttt{err}(h_{v})$ of $h_{v} \in
\mathcal{H}_{v}$ and the disagreement $d(f,g)$ between two hypotheses $f$ and $g$ are defined as follows.
\begin{eqnarray*}
    \texttt{err}(h_{v})&\!\!=\!\!&P_{(x,y)\in\mathcal {X}\times\mathcal {Y}}\big(h_{v}(x) \neq y\big);\\
    d(f,g)&\!\!=\!\!&P_{x\in \mathcal {X}}(f(x)\neq g(x)).
\end{eqnarray*}

In disagreement-based approaches shown in Algorithm~\ref{alg:co-training-process}, one classifier selects some unlabeled instances to label for the other. Here comes the question: how to select these unlabeled instances?
In co-training \cite{Blum:Mitchell1998}, Blum and Mitchell did not specify how to select unlabeled instances (see Page 8, Table 1 in their paper), though in their experiments they selected the most confident unlabeled instances to label with the intuition that confident instances bring high label quality. Nevertheless, there is no guarantee that selecting confident instances is better than selecting random instances, and it needs strong assumptions to characterize the relationship between confidence and label quality. Sometimes, selecting confident instances may be no better than selecting random instances if the confidence is unreliable. Actually, if the learning paradigm believes that labeling confident instances is helpful, it can select confident instances to label; otherwise, it can select random instances to label. Thus, selecting random instances to label can be thought of as the worst case. To make our theoretical analysis general and without any specific assumption on the learning process, we will prove upper bounds on the error rates of the classifiers by considering that each classifier selects random instances to label. In fact, these upper bounds also hold in the case where selecting confident instances helps; this is easy to understand: labeling confident instances may reduce label noise, according to the standard PAC learning theory, learning from the data with small label noise is no harder than learning from the data with large label noise.

Two classifiers trained with different views or different learning algorithms have different biases, it is the intuition why the disagreement-based approaches can work. Two classifiers having different biases implies that they classify some
unlabeled instances with different labels. The disagreement $d(f,g)$ can be used to estimate the difference. In disagreement-based approaches, $f$ selects some unlabeled instances from $U$ to label and adds them into the training data of $g$. If these newly labeled
examples are helpful in updating $g$, $f$ should know some information that $g$ does not know, i.e., $f$ can correctly classify some unlabeled instances which are mistakenly
classified by $g$. Obviously, this helpful information is a part of the
disagreement between $f$ and $g$. Unfortunately, sometimes $f$ may provide some
mistakenly classified examples to $g$ due to its non-perfect performance, and these mistakenly classified examples from $f$ would degrade the performance of $g$. So we must carefully characterize the newly labeled examples.

Let $\varepsilon_{v}^{i}$ denote the error rate of $h_{v}^{i}$, i.e., $\varepsilon_{v}^{i}=\texttt{err}(h_{v}^{i})$ ($i=0,\ldots,s$ and $v=1,2$). In the beginning, two classifiers $h_{1}^{0}$
and $h_{2}^{0}$ are trained with the initial $l$ labeled examples by minimizing the
empirical risk. In real-world applications, it has been found that these initial labeled examples play an important role, i.e., when $l$ is too
large, $h_{1}^{0}$ and $h_{2}^{0}$ are so good that they could
improve each other hardly; while $l$ is too small, $h_{1}^{0}$ and
$h_{2}^{0}$ are so weak that they may degenerate each other due to large noise in
the newly labeled examples \cite{Pierce2001,KrogelS03}. In our analysis, we suppose that by minimizing the empirical risk on $l$ labeled examples we can train two classifiers $h_{1}^{0}$ and $h_{2}^{0}$ with $\varepsilon_{1}^{0}<\frac{1}{2}$ and $\varepsilon_{2}^{0}<\frac{1}{2}$, and the value of $l$ satisfies the model of learning from noisy examples in \cite{AngluinLaird1988} with noise rate $\eta$ shown in Equation \ref{eq:intial-l}. The reason why we use this noise model to characterize $l$ is that updating the classifiers is a process of learning from noisy examples.
\begin{eqnarray}
\label{eq:intial-l}
    l\geq\frac{2}{(\varepsilon_{1}^{0})^{2}(1-2\eta)^{2}}\ln\frac{2|{\mathcal{H}}_{1}|}{\delta},
    ~~~~~~
    l\geq\frac{2}{(\varepsilon_{2}^{0})^{2}(1-2\eta)^{2}}\ln\frac{2|{\mathcal{H}}_{2}|}{\delta}.
\end{eqnarray}
Considering that the initial labeled examples are clean, i.e., $\eta=0$, we get that $l$ should be no less than
\begin{eqnarray*}
    \max\Big[\frac{2}{(\varepsilon_{1}^{0})^{2}}\ln\frac{2|{\mathcal{H}}_{1}|}{\delta},~
    \frac{2}{(\varepsilon_{2}^{0})^{2}}\ln\frac{2|{\mathcal{H}}_{2}|}{\delta}\Big].
\end{eqnarray*}
Let $\xi_{v}^{i}$ denote the
upper bound on the error rate of $h_{1}^{i}$, i.e., $\varepsilon_{v}^{i}\leq\xi_{v}^{i}$. In order to show whether the performance of the classifiers could be improved, we need to analyze the upper bounds $\xi_{1}^{i}$ and $\xi_{2}^{i}$ for $i\geq 1$. In detail, considering the $i$-th round, $h_{v}^{i} \in {\mathcal{H}}_{v}$ randomly selects $u$ unlabeled instances from $U$ to label and adds them into the training data $\sigma_{3-v}$,
then $h_{3-v}^{i+1}$ is trained with $\sigma_{3-v}$. The disagreement
$d(h_{v}^{i},h_{3-v}^{i+1})$ is a kind of ``distance'' between
$h_{v}^{i}$ and $h_{3-v}^{i+1}$, and can be estimated conveniently when there are a large mount of unlabeled instances. This ``distance'' will help us bound the performance of $h_{3-v}^{i+1}$ with respect to the performance of
$h_{v}^{i}$. Iteratively, we can bound the performance of
$h_{v}^{i}$ for $i\geq 1$. Based on this intuition, we give the following upper bounds for $\xi_{1}^{i}$ and $\xi_{2}^{i}$, and discuss the insight we can get from the bounds.

\begin{theorem}\label{theorem:upper-bound}
In Algorithm~\ref{alg:co-training-process}, suppose one classifier randomly selects unlabeled instances to label for the other,
\begin{eqnarray*}
    \Theta_{i}=\sum_{k=0}^{i-1}\big(d(h_{1}^{i},h_{2}^{k})-\varepsilon_{2}^{k}\big),~~~~
    \Delta_{i}=\sum_{k=0}^{i-1}\big(d(h_{1}^{k},h_{2}^{i})-\varepsilon_{1}^{k}\big),
\end{eqnarray*}
$\xi_{1}^{i}= \frac{\varepsilon_{1}^{0}\sqrt{l^{2}+i\cdot u\cdot l}}{l}-\frac{u\cdot\Theta_{i}}{l}$ and $\xi_{2}^{i}= \frac{\varepsilon_{2}^{0}\sqrt{l^{2}+i\cdot u\cdot l}}{l}-\frac{u\cdot\Delta_{i}}{l}$, if $\Theta_{i}> \frac{i\cdot\varepsilon_{1}^{0}}{2}$ and $\Delta_{i}> \frac{i\cdot\varepsilon_{2}^{0}}{2}$, then $\xi_{1}^{i}< \varepsilon_{1}^{0}$, $\xi_{2}^{i}< \varepsilon_{2}^{0}$, and the following bounds on the error rates of $h_{1}^{i}$ and $h_{2}^{i}$ hold.
\begin{eqnarray}
\label{eq:upper-bound-1}
  P\big(\texttt{err}(h_{1}^{i})\leq \xi_{1}^{i} \big)&\!\!\geq\!\!& 1-\delta,\\
\label{eq:upper-bound-2}
  P\big(\texttt{err}(h_{2}^{i})\leq \xi_{2}^{i} \big)&\!\!\geq\!\!& 1-\delta.
\end{eqnarray}
\end{theorem}

\begin{proof}
First, it is easy to verify that $\xi_{1}^{i}< \varepsilon_{1}^{0}$ and $\xi_{2}^{i}< \varepsilon_{2}^{0}$ for $\Theta_{i}> \frac{i\cdot\varepsilon_{1}^{0}}{2}$ and $\Delta_{i}> \frac{i\cdot\varepsilon_{2}^{0}}{2}$.
In the proof, $SD1_{i}=\sum_{k=0}^{i-1}d(h_{1}^{k},h_{2}^{i})$ denotes the sum of disagreement and $SE1_{i}=\sum_{k=0}^{i-1}\varepsilon_{1}^{k}$ denotes the sum of error rate w.r.t. $h_{1}$
after $i$ rounds; $SD2_{i}=\sum_{k=0}^{i-1}d(h_{1}^{i},h_{2}^{k})$ denotes the sum of
disagreement and $SE2_{i}=\sum_{k=0}^{i-1}\varepsilon_{2}^{k}$ denotes the sum of error rate w.r.t. $h_{2}$ after $i$ rounds.

After $i$ rounds, both the training data $\sigma_{1}$ and $\sigma_{2}$ consist of $l$ labeled examples and $u\cdot i$ newly labeled examples. First, we analyze the inconsistency between any $h_{2} \in {\mathcal{H}}_{2}$ and $\sigma_{2}$. Let $\big\{(x^{1},y^{1}), \ldots, (x^{l},y^{l}), (x^{1}_{0},y^{1}_{0}), \ldots,
(x^{u}_{0},y^{u}_{0}), \ldots, (x^{1}_{i-1},y^{1}_{i-1}),\\ \ldots,
(x^{u}_{i-1},y^{u}_{i-1})\big\}$ denote the $(l+u\cdot i)$ examples in $\sigma_{2}$,
where $\big\{(x^{1}_{k},y^{1}_{k}), \ldots, (x^{u}_{k},y^{u}_{k})\big\}$
denotes the $u$ newly labeled examples labeled by $h_{1}^{k}$ and $y_{k}^{u}$
is the pseudo-label of $x_{k}^{u}$ assigned by $h_{1}^{k}$ $(k=0, \ldots, i-1)$. Let $X_{1},
\ldots, X_{l+iu}$ be random variables taking on values 0 or 1, where
$X_{t}=1$ means that for $x^{t} \in \sigma_{2}$, $h_{2}$ makes a different prediction on $x^{t}$ from its pseudo-label $(t=1, \ldots, l+iu)$. Considering that $\big\{(x_{k}^{1},y_{k}^{1}), \ldots, (x_{k}^{u},y_{k}^{u})\big\}$ are randomly selected by $h_{1}^{k}$, so $X_{1}, \ldots, X_{l+iu}$ are independent random variables. We let $p_{t}=P(X_{t}=1)$ and $X=\sum_{t=1}^{l+iu}X_{t}$. Obviously, $X$ equals to the number of inconsistent examples between $h_{2}$ and $\sigma_{2}$.

Since the initial $l$ labeled examples $\big\{(x^{1},y^{1}), \ldots, (x^{l},y^{l})\big\}$ in $\sigma_{2}$ are independently and identically drawn from the distribution $\mathcal {D}$ and the disagreement between $h_{2}\in{\mathcal{H}}_{2}$ and the target concept $c$ is $d(h_{2},c)$, for any $x^{t}\in \{(x^{1},y^{1}), \ldots, (x^{l},y^{l})\}$,
$p_{t}=d(h_{2},c)$; since the newly labeled examples
$\big\{(x_{k}^{1},y_{k}^{1}),\\ \ldots, (x_{k}^{u},y_{k}^{u})\big\}$ $(k=0, \ldots, i-1)$ are selected randomly by
$h_{1}^{k}$ and the disagreement between $h_{2}\in{\mathcal{H}}_{2}$ and $h_{1}^{k}$ is
$d(h_{2},h_{1}^{k})$, for any $x^{t}\in \{(x_{1}^{k},y_{1}^{k}),
\ldots, (x_{u}^{k},y_{u}^{k})\}$, $p_{t}=d(h_{2},h_{1}^{k})$. So we get the expectation $E(X)$ of $X$ is:
\begin{eqnarray}\label{eq:expectation-X}
    E(X)=E\Big(\sum_{t=1}^{l+iu}X_{t}\Big)=\sum_{t=1}^{l+iu}p_{t}=l\cdot d(h_{2},c)+u\cdot \sum_{k=0}^{i-1}d(h_{2},h_{1}^{k}).
\end{eqnarray}

Then, we analyze the inconsistency between the target concept $c$ and $\sigma_{2}$. Let
random variables $X_{1}', \ldots, X_{l+iu}'$ be independent random variables
taking on values 0 or 1, where $X_{t}'=1$ $(t=1, \ldots, l+iu)$ means that for $x^{t} \in \sigma_{2}$, the target concept $c$ makes a different prediction on
$x^{t}$ from its pseudo-label. We let $q_{t}=P(X_{t}'=1)$ and
$X'=\sum_{t=1}^{l+iu}X_{t}'$. Obviously, $X'$ equals to the number of
inconsistent examples between $c$ and $\sigma_{2}$. Since the newly labeled examples
$\big\{(x_{k}^{1},y_{k}^{1}), \ldots, (x_{k}^{u},y_{k}^{u})\big\}$ $(k=0, \ldots, i-1)$ are selected randomly by $h_{1}^{k}$ and the disagreement between the target concept $c$ and $h_{1}^{k}$ is $d(c,h_{1}^{k})$, similarly
to Equation~\ref{eq:expectation-X} we get the expectation $E(X')$ of $X'$ is:
\begin{eqnarray}\label{eq:expectation-X'}
   E(X')=E\Big(\sum_{t=1}^{l+iu}X_{t}'\Big)=\sum_{t=1}^{l+iu}q_{t}
   =u\cdot \sum_{k=0}^{i-1}d(c,h_{1}^{k})=u\cdot \sum_{k=0}^{i-1}\varepsilon_{1}^{k}.
\end{eqnarray}

According to minimizing the empirical risk, the algorithm will search out the
classifier which has the lowest observed inconsistent examples with the training data
$\sigma_{2}$. If we want to achieve a `good' classifier whose error rate is no larger than $\xi_{2}^{i}$ with probability at least $1-\delta$ by minimizing the empirical risk, $\sigma_{2}$ should be sufficient to
guarantee that the classifier whose error rate is larger than $\xi_{2}^{i}$ has a lower observed inconsistent examples with $\sigma_{2}$ than the target concept $c$ with
probability no larger than $\delta$.

Thus, for $h_{2}^{i}\in {\mathcal{H}}_{2}$, if $d(h_{2}^{i},c)>\xi_{2}^{i}$, from Equations~\ref{eq:expectation-X} and \ref{eq:expectation-X'} we get
\begin{eqnarray}
\nonumber
 E(X)-E(X')&\!>\!&l\cdot \xi_{2}^{i}+u\cdot \sum_{k=0}^{i-1}d(h_{2}^{i},h_{1}^{k})-u\cdot \sum_{k=0}^{i-1}\varepsilon_{1}^{k}\\
\label{eq:large-bi}
     &\!=\!&\varepsilon_{2}^{0}\sqrt{l^{2}+i\cdot u\cdot l}.
\end{eqnarray}
It means that if
$d(h_{2}^{i},c)>\xi_{2}^{i}$, the expected inconsistent examples between
$h_{2}^{i}$ and $\sigma_{2}$ is at least $\varepsilon_{2}^{0}\sqrt{l^{2}+i\cdot u\cdot l}$ larger than that between the target concept $c$
and $\sigma_{2}$. If $h_{2}^{i}$ minimizes the empirical risk on $\sigma_{2}$, the number of observed inconsistent examples between $h_{2}^{i}$ and $\sigma_{2}$ is no larger than that between the target concept $c$ and $\sigma_{2}$, i.e., $X\leq X'$. In this case, either $X'\geq E(X')+ \frac{\varepsilon_{2}^{0}\sqrt{l^{2}+i\cdot u\cdot l}}{2}$ or
$X\leq E(X')+ \frac{\varepsilon_{2}^{0}\sqrt{l^{2}+i\cdot u\cdot l}}{2}$
holds. Considering that there are at most
$|{\mathcal{H}}_{2}|-1$ classifiers whose error rates are larger than $\xi_{2}^{i}$, so if Equations~\ref{eq:x-x'1} and \ref{eq:x-x'2} hold, it can be guaranteed that the classifier whose error rate is larger than $\xi_{2}^{i}$ has a lower
observed inconsistent examples with $\sigma_{2}$ than the target concept $c$ with probability no larger than $\delta$.
\begin{eqnarray}
\label{eq:x-x'1}
    P\big(X'\geq E(X')+ \frac{\varepsilon_{2}^{0}\sqrt{l^{2}+i\cdot u\cdot l}}{2}\big)&\!\!\leq\!\!&\frac{\delta}{2};\\
\label{eq:x-x'2}
    P\big(X\leq
    E(X')+ \frac{\varepsilon_{2}^{0}\sqrt{l^{2}+i\cdot u\cdot l}}{2}\big)&\!\!\leq\!\!&
    \frac{\delta}{2|{\mathcal{H}}_{2}|}.
\end{eqnarray}
So with Hoeffding bounds \cite{Hoeffding:1963} we get
\begin{eqnarray}
\label{eq:delta-1}
    P\big(X'\geq E(X')+ \frac{\varepsilon_{2}^{0}\sqrt{l^{2}+i\cdot u\cdot l}}{2}\big)
    \leq\exp\Big(-\frac{(\varepsilon_{2}^{0}\sqrt{l^{2}+i\cdot u\cdot l})^{2}}{2(l+u\cdot i)}\Big),
\end{eqnarray}
\begin{eqnarray}
\label{eq:delta-2}
    P\big(X\leq E(X')+ \frac{\varepsilon_{2}^{0}\sqrt{l^{2}+i\cdot u\cdot l}}{2}\big)
    \leq\exp\Big(-\frac{(\varepsilon_{2}^{0}\sqrt{l^{2}+i\cdot u\cdot l})^{2}}{2(l+u\cdot i)}\Big).
\end{eqnarray}
Since $l\geq
\frac{2}{(\varepsilon_{2}^{0})^{2}}\ln\frac{2|{\mathcal{H}}_{2}|}{\delta}$, we get $\exp\big(-\frac{(\varepsilon_{2}^{0}\sqrt{l^{2}+i\cdot u\cdot l})^{2}}{2(l+u\cdot i)}\big)\leq
\frac{\delta}{2|{\mathcal{H}}_{2}|}$. So with Equations~\ref{eq:delta-1} and \ref{eq:delta-2} we find that Equations~\ref{eq:x-x'1} and \ref{eq:x-x'2} hold. Thus we have that
$P\big(\texttt{err}(h_{2}^{i})\leq \xi_{2}^{i}\big)\geq1-\delta$ holds. Similarly, we have that $P\big(\texttt{err}(h_{1}^{i})\leq \xi_{1}^{i}\big)\geq1-\delta$ holds.
\end{proof}

\noindent \textbf{Remark:} The bounds in Theorem~\ref{theorem:upper-bound} seem somewhat
complicated to understand, we give an explanation in the following comprehensive way: in
the $i$-th round, the training data $\sigma_{1}$ for
$h_{1}^{i}$ contain the initial $l$ labeled examples and $u\cdot i$ newly labeled examples
from $h_{2}^{0},\ldots,h_{2}^{i-1}$. The classifier $h_{1}^{0}$ with error rate $\varepsilon_{1}^{0}$ can be trained with the initial $l$ labeled examples, now we investigate the contribution that $h_{2}^{0},\ldots,h_{2}^{i-1}$ make to retraining $h_{1}^{i}$. $d(h_{1}^{i},h_{2}^{k})$ measures the information that $h_{2}^{k}$ ($k<i$) knows while
$h_{1}^{i}$ does not know. Wiping off the possibly wrong information from
$h_{2}^{k}$ bounded by its error rate $\varepsilon_{2}^{k}$, $d(h_{1}^{i},h_{2}^{k})-\varepsilon_{2}^{k}$
is an estimation of the helpful information that $h_{2}^{k}$ offers to $h_{1}^{i}$. So
$\Theta_{i}=\sum_{k=0}^{i-1}\big(d(h_{1}^{i},h_{2}^{k})-\varepsilon_{2}^{k}\big)$ measures the helpful information provided by $h_{2}^{0},\ldots,h_{2}^{i-1}$. This is the intuition why the bounds are meaningful. $d(h_{1}^{k},h_{2}^{i})$ and $d(h_{1}^{i},h_{2}^{k})$ can be estimated conveniently with unlabeled data, then we can calculate $\xi_{1}^{k}$ and $\xi_{2}^{k}$ according to Equations~\ref{eq:upper-bound-1} and \ref{eq:upper-bound-2}. Using them as the approximations of $\varepsilon_{1}^{k}$ and
$\varepsilon_{2}^{k}$, we can get $\varepsilon_{1}^{i}$, $\varepsilon_{2}^{i}$, $\Theta_{i}$ and $\Delta_{i}$ in an iterative way.

Generally, as the disagreement-based process goes on, the disagreement will decrease since $h_{v}^{k+1}$ has $u$ training examples from $h_{3-v}^{k}$. Here, we give Theorem~\ref{theorem:disagreement} on the disagreement.

\begin{theorem}\label{theorem:disagreement}
In Algorithm~\ref{alg:co-training-process}, Equation~\ref{eq:disagreement-relation} on the disagreement between the classifiers holds.
\begin{eqnarray}
\label{eq:disagreement-relation}
  d(h_{1}^{k},h_{2}^{k})\geq d(h_{1}^{k},h_{2}^{k+1});~~~~~~d(h_{1}^{k},h_{2}^{k})\geq d(h_{1}^{k+1},h_{2}^{k}).
\end{eqnarray}
\end{theorem}

\begin{proof} Let $D(h_{1}, \sigma_{1})$ denote the number of inconsistent examples
between $h_{1}$ and the training data $\sigma_{1}$, and let $\sigma_{1}^{k}$ denote $\sigma_{1}$ in the $k$-th round. $h_{1}^{k}\in \mathcal {H}_{1}$ is trained on $\sigma_{1}^{k}$ by minimizing the empirical risk, so for any $h_{1}\in \mathcal {H}_{1}$, we have
\begin{eqnarray}\label{eq:dis2-of-k}
   D(h_{1}^{k}, \sigma_{1}^{k})\leq D(h_{1}, \sigma_{1}^{k}).
\end{eqnarray}
$h_{2}^{k}$ randomly selects $u$ unlabeled instances to label and adds them into $\sigma_{1}^{k}$ to get
$\sigma_{1}^{k+1}$. Let $S_{2}^{k}$ denote these $u$ newly labeled examples, i.e., $\sigma_{1}^{k+1}=\sigma_{1}^{k}\cup S_{2}^{k}$, the $h_{1}^{k+1}\in \mathcal {H}_{1}$ is trained on
$\sigma_{1}^{k+1}$ by minimizing the empirical risk, so for any $h_{1}^{k}\in \mathcal
{H}_{1}$, we have
\begin{eqnarray}\label{eq:dis2-of-k+1}
    D(h_{1}^{k+1}, \sigma_{1}^{k+1})\leq D(h_{1}^{k}, \sigma_{1}^{k+1}).
\end{eqnarray}
Let $D_{|S_{2}^{k}}(h_{1}^{k+1},h_{2}^{k})$ denote the number of inconsistent examples
between $h_{1}^{k+1}$ and $h_{2}^{k}$ on $S_{2}^{k}$,
from Equation~\ref{eq:dis2-of-k+1} we get
\begin{eqnarray}\label{eq:dis2-of-k+1-2}
  D(h_{1}^{k+1}, \sigma_{1}^{k})+D_{|S_{2}^{k}}(h_{1}^{k+1},h_{2}^{k})\leq
  D(h_{1}^{k}, \sigma_{1}^{k})+D_{|S_{2}^{k}}(h_{1}^{k},h_{2}^{k}).
\end{eqnarray}
With Equations~\ref{eq:dis2-of-k} and \ref{eq:dis2-of-k+1-2} we have
\begin{eqnarray}\label{eq:dis2-of-k+1-3}
  D_{|S_{2}^{k}}(h_{1}^{k+1},h_{2}^{k})\leq D_{|S_{2}^{k}}(h_{1}^{k},h_{2}^{k}).
\end{eqnarray}
Since the examples in $S_{2}^{k}$ are drawn randomly from the distribution,
with Equation~\ref{eq:dis2-of-k+1-3} we get
\begin{eqnarray}\label{eq:dis2-of-k+1-4}
  \mathbb{E}\big(\frac{D_{|S_{2}^{k}}(h_{1}^{k+1},h_{2}^{k})}{|S_{2}^{k}|}\big)\leq
  \mathbb{E}\big(\frac{D_{|S_{2}^{k}}(h_{1}^{k},h_{2}^{k})}{|S_{2}^{k}|}\big),
\end{eqnarray}
i.e., $d(h_{1}^{k+1},h_{2}^{k})\leq d(h_{1}^{k},h_{2}^{k})$. Similarly, we get $d(h_{1}^{k},h_{2}^{k+1})\leq d(h_{1}^{k},h_{2}^{k})$.
\end{proof}

\noindent \textbf{Remark:} Theorem~\ref{theorem:upper-bound} shows that when
$\Theta_{i}> \frac{i\cdot\varepsilon_{1}^{0}}{2}$ and $\Delta_{i}> \frac{i\cdot\varepsilon_{2}^{0}}{2}$, the error rate of $h_{v}^{i}$ ($v=1,2$) is smaller than that of $h_{v}^{0}$. If we require $\xi_{1}^{1}<\varepsilon_{1}^{0}$ and $\xi_{2}^{1}<\varepsilon_{2}^{0}$, $\Theta_{1}$ and $\Delta_{1}$ should be larger than $\frac{\varepsilon_{1}^{0}}{2}$ and $\frac{\varepsilon_{2}^{0}}{2}$, respectively, i.e., $d(h_{1}^{1},h_{2}^{0})>\varepsilon_{2}^{0}+\frac{\varepsilon_{1}^{0}}{2}$ and $d(h_{1}^{0},h_{2}^{1})>\varepsilon_{1}^{0}+\frac{\varepsilon_{2}^{0}}{2}$. Now, Theorem~\ref{theorem:disagreement} indicates that $d(h_{1}^{0},h_{2}^{0})\geq d(h_{1}^{1},h_{2}^{0})$ and $d(h_{1}^{0},h_{2}^{0})\geq d(h_{1}^{0},h_{2}^{1})$. It is easy to know that the disagreement $d(h_{1}^{0},h_{2}^{0})$
between the two initial classifiers $h_{1}^{0}$ and $h_{2}^{0}$ should be at least larger than $\max\big[\varepsilon_{1}^{0}+\frac{\varepsilon_{2}^{0}}{2},\varepsilon_{2}^{0}+\frac{\varepsilon_{1}^{0}}{2}\big]$. Therefore, it could be recognized that the two views used in co-training \cite{Blum:Mitchell1998}, the two different learning algorithms used in \cite{Goldman:2000}, and
the two different parameter configurations used in \cite{Zhou:Li2007TKDE} are actually exploited to make the two initial classifiers have
large disagreement. It does not matter whether the disagreement comes from the two views or not. This explains why the disagreement-based approaches can work.

In the following parts of this section we will discuss whether the condition in Theorem~\ref{theorem:upper-bound} can be satisfied in the applications with or without two views.

\subsubsection{Co-Training}\label{sec:two-view-co-training}

In real-world applications, co-training can be implemented when there exist two views. First, we show that the condition in Theorem~\ref{theorem:upper-bound} could be satisfied in co-training. An extreme case is that the two views are exactly the same and the two-view setting degenerates into the single-view setting. To analyze co-training, we should know some prior knowledge about the two views. There have been some theoretical analyses on co-training, i.e., conditional independence analysis and expansion analysis.

\emph{Conditionally Independent Views.}
When Blum and Mitchell \cite{Blum:Mitchell1998} proposed co-training, they assumed there exist two sufficient and redundant views in the data. If the two views are conditionally independent to each other,
they proved that co-training can boost the performance of weak classifiers to arbitrarily high by using unlabeled data. Here, we give Theorem~\ref{theorem:two-views-independence} for co-training with conditionally independent views.

\begin{theorem}\label{theorem:two-views-independence}
Suppose $0\leq\varepsilon_{1}^{0}, \varepsilon_{2}^{0}\leq\zeta\leq\frac{1}{6}$ and the data have two conditionally independent views, the condition $\Theta_{i}> \frac{i\cdot\varepsilon_{1}^{0}}{2}$ and $\Delta_{i}> \frac{i\cdot\varepsilon_{2}^{0}}{2}$ in Theorem~\ref{theorem:upper-bound} could hold before $\varepsilon_{1}^{i}$ (the error rate of $h_{1}^{i}$) and $\varepsilon_{2}^{i}$ (the error rate of $h_{2}^{i}$) decrease to $\frac{1}{2(1-2\zeta)}\varepsilon_{1}^{0}$ and $\frac{1}{2(1-2\zeta)}\varepsilon_{2}^{0}$, respectively. Here $\frac{1}{2}\leq\frac{1}{2(1-2\zeta)}\leq\frac{3}{4}$.
\end{theorem}

\begin{proof}
In the two-view setting, for an example $(x,y)$, let $h_{1}^{i}(x)$ denote the label predicted by the classifier in the $i$-th round of the first view and let $h_{2}^{i}(x)$ denote the label predicted by the classifier in the $i$-th round of the second view. The two views are conditionally independent to each other means that the classifier in the first view is independent of the classifier in the second view to make predictions. So we have
\begin{eqnarray*}
    d(h_{1}^{i},h_{2}^{k})&\!=\!&P(h_{1}^{i}(x)\neq h_{2}^{k}(x))\\
    &\!=\!&P(h_{1}^{i}(x)=y)P(h_{2}^{k}(x)\neq y)+P(h_{1}^{i}(x)\neq y)P(h_{2}^{k}(x)=y)\\
    &\!=\!&(1-\varepsilon_{1}^{i})\varepsilon_{2}^{k}+
    \varepsilon_{1}^{i}(1-\varepsilon_{2}^{k})\\
    &\!=\!&\varepsilon_{2}^{k}+(1-2\zeta)\varepsilon_{1}^{i}~~~~~~\textrm{w.r.t.}~\varepsilon_{2}^{k}<\varepsilon_{2}^{0}
\end{eqnarray*}
Thus, when $\varepsilon_{1}^{i}>\frac{1}{2(1-2\zeta)}\varepsilon_{1}^{0}$, we get $\Theta_{i}=\sum_{k=0}^{i-1}\big(d(h_{1}^{i},h_{2}^{k})-\varepsilon_{2}^{k}\big)>\frac{i\cdot\varepsilon_{1}^{0}}{2}$.
Similarly, when $\varepsilon_{2}^{i}>\frac{1}{2(1-2\zeta)}\varepsilon_{2}^{0}$, we get $\Delta_{i}=\sum_{k=0}^{i-1}\big(d(h_{1}^{k},h_{2}^{i})-\varepsilon_{1}^{k}\big)>\frac{i\cdot\varepsilon_{2}^{0}}{2}$.
\end{proof}

\noindent \textbf{Remark:} Theorem~\ref{theorem:two-views-independence} shows that if the two views are conditionally independent, the condition in Theorem~\ref{theorem:upper-bound} holds in a number of learning rounds (before the error rates decrease to some degree).
It could not guarantee that the condition in Theorem~\ref{theorem:upper-bound} always holds in the learning process, which is different from that co-training can boost the performance to arbitrarily high with the conditional independence condition. This is understandable because our theorem provides a general analysis and does not depend on any strong condition. In fact, we will prove that the performance could not always be improved as the learning process goes on in Section~\ref{sec:no-further-improvement}.

\emph{Expanding Views.}
Balcan et al. \cite{Balcan:Blum:Yang2005} pointed out that ``expansion'' of the underlying data distribution is sufficient for co-training to succeed. Furthermore, they also assumed that the classifier in each view is never ``confident but wrong'', which means that if the classifier makes a prediction, the prediction is correct. Here, we give Theorem~\ref{theorem:two-view-expansion} for co-training with expanding views.

\begin{theorem}\label{theorem:two-view-expansion}
Suppose the ``expansion'' assumption holds and the classifier in each view is never ``confident but wrong'', the condition $\Theta_{i}> \frac{i\cdot\varepsilon_{1}^{0}}{2}$ and $\Delta_{i}> \frac{i\cdot\varepsilon_{2}^{0}}{2}$ in Theorem~\ref{theorem:upper-bound} holds.
\end{theorem}

\begin{proof}
Let $S_{1}$ denote the examples predicted correctly by the first view and let $S_{2}$ denote the examples predicted correctly by the second view. $P(S_{1}\oplus S_{2})$ denotes the probability mass on the examples predicted correctly by just one view. The ``expansion'' assumption means that for some $\alpha>0$ and any $S_{1}$, $S_{2}$, $P(S_{1}\oplus S_{2})\geq \alpha \min\big[P(S_{1}\wedge S_{2}), P(\overline{S_{1}}\wedge \overline{S_{2}})\big]$ holds. It implies that $P(S_{1}\oplus S_{2})$ has a lower bound, i.e., $P(S_{1}\oplus S_{2})>0$. It is easy to find that $d(h_{1}^{i},h_{2}^{k})=P(S_{1}^{i}\oplus S_{2}^{k})>0$.

The classifier in each view is never ``confident but wrong'' means that if the classifier makes a prediction, the prediction is correct, i.e., $\varepsilon_{1}^{0}=\varepsilon_{2}^{k}=0$.
Thus, we get $\Theta_{i}=\sum_{k=0}^{i-1}\big(d(h_{1}^{i},h_{2}^{k})-\varepsilon_{2}^{k}\big)>\frac{i\cdot\varepsilon_{1}^{0}}{2}$.
Similarly, we get
$\Delta_{i}=\sum_{k=0}^{i-1}\big(d(h_{1}^{k},h_{2}^{i})-\varepsilon_{1}^{k}\big)>\frac{i\cdot\varepsilon_{2}^{0}}{2}$.
\end{proof}

\noindent \textbf{Remark:} Theorem~\ref{theorem:two-view-expansion} shows that if the two views meet the expansion condition, the condition in Theorem~\ref{theorem:upper-bound} holds. It implies that our result is more general. However, these previous results only focus on co-training, they can not explain why the single-view disagreement-based approaches can work.

\subsubsection{Single-View Disagreement-Based Approaches}
Secondly, we study the setting where there exists only one view and give Theorem~\ref{theorem:single-view-diversity} to show that there exist two hypotheses which have large disagreement.

\begin{theorem}\label{theorem:single-view-diversity}
For any real numbers $0<a, b<\frac{1}{2}$, there exist two hypotheses $h_{1}$ and $h_{2}$ which satisfy the following conditions: $\texttt{err}(h_{1})=a$, $\texttt{err}(h_{2})=b$ and $|a-b|\leq d(h_{1},h_{2}) \leq (a+b)$.
\end{theorem}

\begin{proof}
Without loss of generality, we assume that $a\leq b$ and that $h_{1}$ is the hypothesis whose error rate is $a$, i.e., $\texttt{err}(h_{1})=a$. Now we show how to find the hypothesis $h_{2}$ which satisfies the conditions: $\texttt{err}(h_{2})=b$ and $|a-b|\leq d(h_{1},h_{2}) \leq (a+b)$.

(i) When $d(h_{1},h_{2})=|a-b|=b-a$. Let $T_{1}=\{x\in \mathcal {X}: h_{1}(x) \neq c(x)\}$. Select a set of size $(b-a)$ from $\mathcal {X}-T_{1}$ and let $T_{2}$ denote this set, i.e., $T_{2}\subset\mathcal {X}-T_{1}$ and $P(T_{2})=b-a$. Then, let $h_{2}$ be the hypothesis corresponding to the following classification rule:
\begin{eqnarray}\label{combination}
    h_{2}(x) = \left\{ \begin{array}{ll}
    h_{1}(x)  & \textrm{if $x\in T_{1}$}\\
    -h_{1}(x) & \textrm{if $x\in T_{2}$}\\
    h_{1}(x)  & \textrm{otherwise}
    \end{array}. \right.
\end{eqnarray}

(ii) When $(b-a)<d(h_{1},h_{2})\leq (a+b)$. Select a set of size $\frac{a+b-d(h_{1},h_{2})}{2}$ from $T_{1}$ and let $T_{3}$ denote this set, i.e., $T_{3}\subset T_{1}$ and $P(T_{3})=\frac{a+b-d(h_{1},h_{2})}{2}$; select a set of size $\frac{d(h_{1},h_{2})+b-a}{2}$ from $\mathcal {X}-T_{1}$ and let $T_{4}$ denote this set, i.e., $T_{4}\subset \mathcal {X}-T_{1}$ and $P(T_{4})=\frac{d(h_{1},h_{2})+b-a}{2}$. Then, let $h_{2}$ be the hypothesis corresponding to the following classification rule:
\begin{eqnarray}\label{combination}
    h_{2}(x) = \left\{ \begin{array}{ll}
    h_{1}(x)  & \textrm{if $x\in T_{3}$}\\
    -h_{1}(x) & \textrm{if $x\in T_{1}-T_{3}$}\\
    -h_{1}(x) & \textrm{if $x\in T_{4}$}\\
    h_{1}(x)  & \textrm{otherwise}
    \end{array}. \right.
\end{eqnarray}
\end{proof}

\noindent \textbf{Remark:} Theorem~\ref{theorem:single-view-diversity} indicates that fixing the error rates of two hypotheses, the disagreement between them can vary from $|a-b|$ to $(a+b)$. For example, there exist hypotheses $h_{1}^{i}$ and $h_{2}^{k}$ which have large disagreement and satisfy the condition that $d(h_{1}^{i},h_{2}^{k})>\frac{2}{3}\varepsilon_{1}^{i}+\varepsilon_{2}^{k}$ and $d(h_{1}^{k},h_{2}^{i})>\varepsilon_{1}^{k}+\frac{2}{3}\varepsilon_{2}^{i}$, i.e.,  $\Theta_{i}=\sum_{k=0}^{i-1}\big(d(h_{1}^{i},h_{2}^{k})-\varepsilon_{2}^{k}\big)>\frac{2i}{3}\varepsilon_{1}^{i}$ and $\Delta_{i}=\sum_{k=0}^{i-1}\big(d(h_{1}^{k},h_{2}^{i})-\varepsilon_{1}^{k}\big)>\frac{2i}{3}\varepsilon_{2}^{i}$.
Before the error rates $\varepsilon_{1}^{i}$ and $\varepsilon_{2}^{i}$ decrease to $\frac{3}{4}\varepsilon_{1}^{0}$ and $\frac{3}{4}\varepsilon_{2}^{0}$, respectively, the condition $\Theta_{i}>\frac{i}{2}\varepsilon_{1}^{0}$ and $\Delta_{i}>\frac{i}{2}\varepsilon_{2}^{0}$ in Theorem~\ref{theorem:upper-bound} holds.

\subsection{Why There is No Further Improvement After a Number of Rounds}\label{sec:no-further-improvement}
In the above section, we prove that the learning performance of the classifiers can be improved by exploiting unlabeled data. Can the classifiers be always improved as the learning process goes on? It is a very important and interesting problem. In fact, in some empirical studies such as the natural language processing community \cite{Pierce2001}, it has been observed that the classifiers could not
be improved further after a number of rounds in disagreement-based approaches, and if the process is going on they would be degraded. This is very different
from what has been disclosed by previous theoretical studies, e.g., the performance of the classifiers can be boosted to arbitrarily high
\cite{Blum:Mitchell1998,Balcan:Blum:Yang2005}. Our following theoretical analysis in
this section will give an explanation to this.

\begin{theorem}\label{theorem:disagreement-covergence}
In Algorithm~\ref{alg:co-training-process}, for any $\epsilon\in (0,1)$ and $\delta\in (0,1)$, there exists some integer $N>0$, for any integer $t\geq N$, the following inequalities hold ($v=1,2$):
\begin{eqnarray}
\label{eq:error-converge}
  P\big(d(h_{v}^{t},h_{v}^{N})\leq \epsilon\big)\geq 1-\delta,\\
\label{eq:disagreement-converge}
  P\big(|d(h_{1}^{t},h_{2}^{t})-d(h_{1}^{N},h_{2}^{N})|\leq \epsilon\big)\geq 1-\delta.
\end{eqnarray}
\end{theorem}

\begin{proof} Let $\sigma_{v}^{k}$ denote the training data $\sigma_{v}$ in the
$k$-th round, let $\mathcal {D}_{v}$
denote the distribution of examples in $\sigma_{v}^{k}$ when $k\rightarrow +\infty$
and let $h_{v}'$ denote the hypothesis which is perfectly consistent with $\mathcal {D}_{v}$. Considering that $h_{v}^{k}$ is trained with $\sigma_{v}^{k}$ which contain $l+u\cdot k$ examples from $\mathcal {D}_{v}$, according to the standard PAC learning
theory, there exists some integer $N>0$, for any integer $t\geq N$,
$P\big(d(h_{v}^{t},h_{v}')\leq \epsilon/8\big)\geq 1-\delta/2$.
Considering that $d(h_{v}^{t},h_{v}^{N})\leq
d(h_{v}^{t},h_{v}')+d(h_{v}^{N},h_{v}')$, we get $P\big(d(h_{v}^{t},h_{v}^{N})\leq
\epsilon\big)\geq 1-\delta$.
Since
\begin{eqnarray}
\nonumber
  |d(h_{1}^{t},h_{2}^{t})-d(h_{1}',h_{2}')|&\!=\!& |d(h_{1}^{t},h_{2}^{t})-d(h_{1}^{t},h_{2}')
  +d(h_{1}',h_{2}^{t})-d(h_{1}',h_{2}')\\
\nonumber
  &\!\!&+d(h_{1}^{t},h_{2}')-d(h_{1}',h_{2}')+d(h_{1}',h_{2}')-d(h_{1}',h_{2}^{t})|\\
\nonumber
  &\!\leq\!& |d(h_{1}^{t},h_{2}^{t})-d(h_{1}^{t},h_{2}')|
  +|d(h_{1}',h_{2}^{t})-d(h_{1}',h_{2}')|\\
\nonumber
  &\!\!&+|d(h_{1}^{t},h_{2}')-d(h_{1}',h_{2}')|+
  |d(h_{1}',h_{2}')-d(h_{1}',h_{2}^{t})|\\
\nonumber
  &\!\leq\!& d(h_{1}^{t},h_{1}')+3\cdot d(h_{2}^{t},h_{2}'),
\end{eqnarray}
we have
\begin{eqnarray}
\nonumber
  |d(h_{1}^{t},h_{2}^{t})-d(h_{1}^{N},h_{2}^{N})|&\!\leq\!&
  |d(h_{1}^{t},h_{2}^{t})-d(h_{1}',h_{2}')|+|d(h_{1}^{N},h_{2}^{N})-
  d(h_{1}',h_{2}')|\\
\nonumber
  &\!\leq\!& 2\cdot d(h_{1}^{t},h_{1}')+ 6\cdot d(h_{2}^{t},h_{2}').
\end{eqnarray}
Thus, we get
$P\big(|d(h_{1}^{t},h_{2}^{t})-d(h_{1}^{N},h_{2}^{N})|\leq\epsilon\big)\geq1-\delta$. \end{proof}

\noindent \textbf{Remark:} Equation~\ref{eq:error-converge} indicates that the error rates of the classifiers will converge, which implies that the classifiers could not be improved further after a number of learning rounds. Actually, the training data for one classifier may contain noisy examples mistakenly labeled by the other classifier, so the surrogate distribution $\mathcal {D}_{v}$ $(v=1, 2)$ is different from the underlying distribution $\mathcal {D}$. Without strong assumptions, e.g., \textit{two views are conditionally independent} or \textit{the classifier is never ``confident but wrong''}, it is very hard to find a good approximation of the target concept $c$ with $\mathcal {D}_{v}$. Equation~\ref{eq:disagreement-converge} indicates that the disagreement between the classifiers will converge. As the classifiers label more and more unlabeled instances for each other, they will become more and more similar and the disagreement between them will decrease closely to $0$. In this case, the disagreement-based process degenerates into self-training. If we continue it, the risk of over-fitting will be great and the performance will be degraded as observed in the empirical study of \cite{Pierce2001}.

\subsection{Empirical Studies}

In this section, we provide empirical studies to verify our theoretical analyses, i.e., whether the disagreement and error rates of the classifiers will converge, and whether larger disagreement will lead to better improvement.

We use the \textit{course} data set \cite{Blum:Mitchell1998} and three UCI data sets, i.e. \textit{kr-vs-kp}, \textit{mushroom} and \textit{tic-tac-toe} \cite{Blake:Keogh:Merz1998}. The \textit{course} data set has two views (i.e., \textit{pages} view and \textit{links}
view) and contains 1,051 examples, among which there are 230 positive
examples (roughly 22\%). The UCI data sets do not have two views.
\textit{kr-vs-kp} contains 3,196 examples, among which there are 1,527 positive examples (roughly 48\%); \textit{mushroom} contains 8,124 examples, among which there are 3,916 positive examples (roughly 48\%); \textit{tic-tac-toe} contains 958 examples, among which there are 332 positive examples (roughly 35\%). On each data set, we randomly select 25\% data as the test set while using the remaining 75\% data to generate
a labeled data $L$ whose size will be mentioned in
Figures~\ref{Estimated-Round-a}, \ref{Estimated-Round-b} and \ref{Estimated-Round-c}, and the rest of the 75\% data are used to generate the unlabeled data $U$.

In the experiments, we let each classifier label its most confident unlabeled instances for the other in each round to reduce label noise. However, this way will cause a problem that the training set is not an i.i.d. sample from the underlying distribution. To reduce the influence of non-i.i.d. sample, we create a small pool $U'$ of size $75$ by drawing instances randomly from the large unlabeled data $U$ and select the confident instances from $U'$ to label, as that in \cite{Blum:Mitchell1998}. They reported that using a small pool can get better results than selecting confident instances directly from $U$, because it forces the classifiers to select
more representative instances. The number of newly labeled positive and negative examples is in proportion to the positive and negative examples in the data set. In our experiments, in order to study the convergence of disagreement and error rates, the disagreement-based process proceeds until no unlabeled instance in $U'$ is labeled as the positive class. The size of $L$ plays an important role in disagreement-based approaches, we run experiments with different $L$ on each data set. Each experiment is repeated for $100$ runs and the average performance is recorded. We run the experiments with various base classifiers including SMO, J48, MultilayerPerceptron and NaiveBayes in WEKA \cite{weka}.

\subsubsection{Convergence of Disagreement and Error rates}\label{sec:convergence-dis-error}
We run co-training on the \textit{course} data set by using NaiveBayes, and run the single-view disagreement-based approach on the UCI data sets by using two different classifiers, i.e., SMO and J48 on the \textit{kr-vs-kp} data set, SMO and NaiveBayes on the \textit{mushroom} data set, and SMO and MultilayerPerceptron on the \textit{tic-tac-toe} data set. The error rates, the disagreement between the two classifiers are depicted in Figures~\ref{Estimated-Round-a} and \ref{Estimated-Round-b}.

In Figure~\ref{Estimated-Round-a}(a), the disagreement between the classifiers increases in the first several rounds and then decreases. This is because in \textit{course}-3-9-1-3 the initial labeled training data set $L$ is too small, the two initial classifiers $h_{1}^{0}$ and $h_{2}^{0}$ are simple and only learn small part of the task. After co-training is executed, the amount of training examples for each classifier increases. Then each retrained classifier learns more about the task from its own perspective and the disagreement between the classifiers increases. As the two classifiers label more and more unlabeled instances for each other, they become more and more similar and the disagreement between them decreases.

\begin{figure*}[!h]
\centering
\begin{minipage}[c]{0.32\linewidth}
\centering
\includegraphics[width = \linewidth]{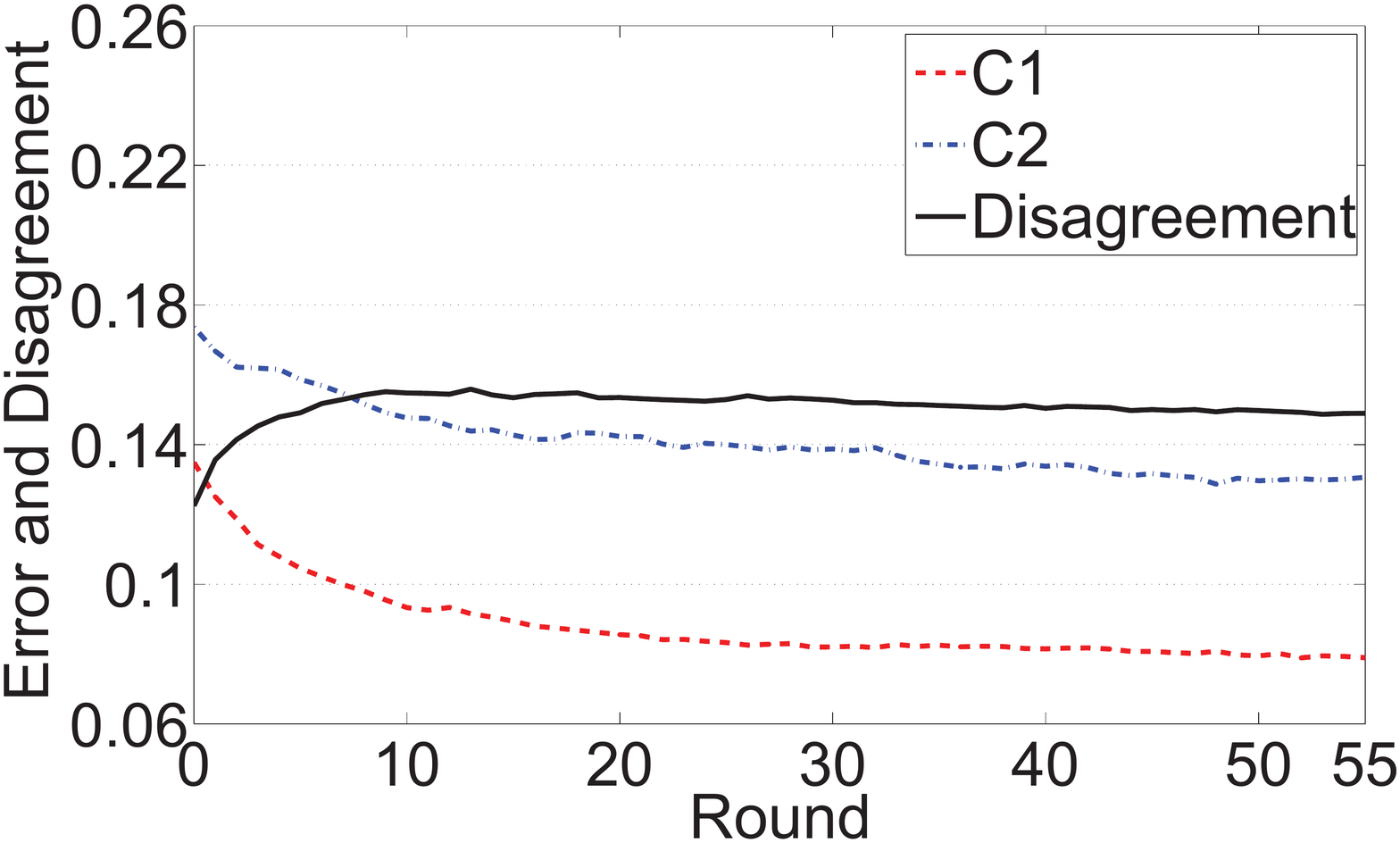}
\centering \mbox{\footnotesize (a)\textit{course}-3-9-1-3}
\end{minipage}
\centering
\begin{minipage}[c]{0.32\linewidth}
\centering
\includegraphics[width = \linewidth]{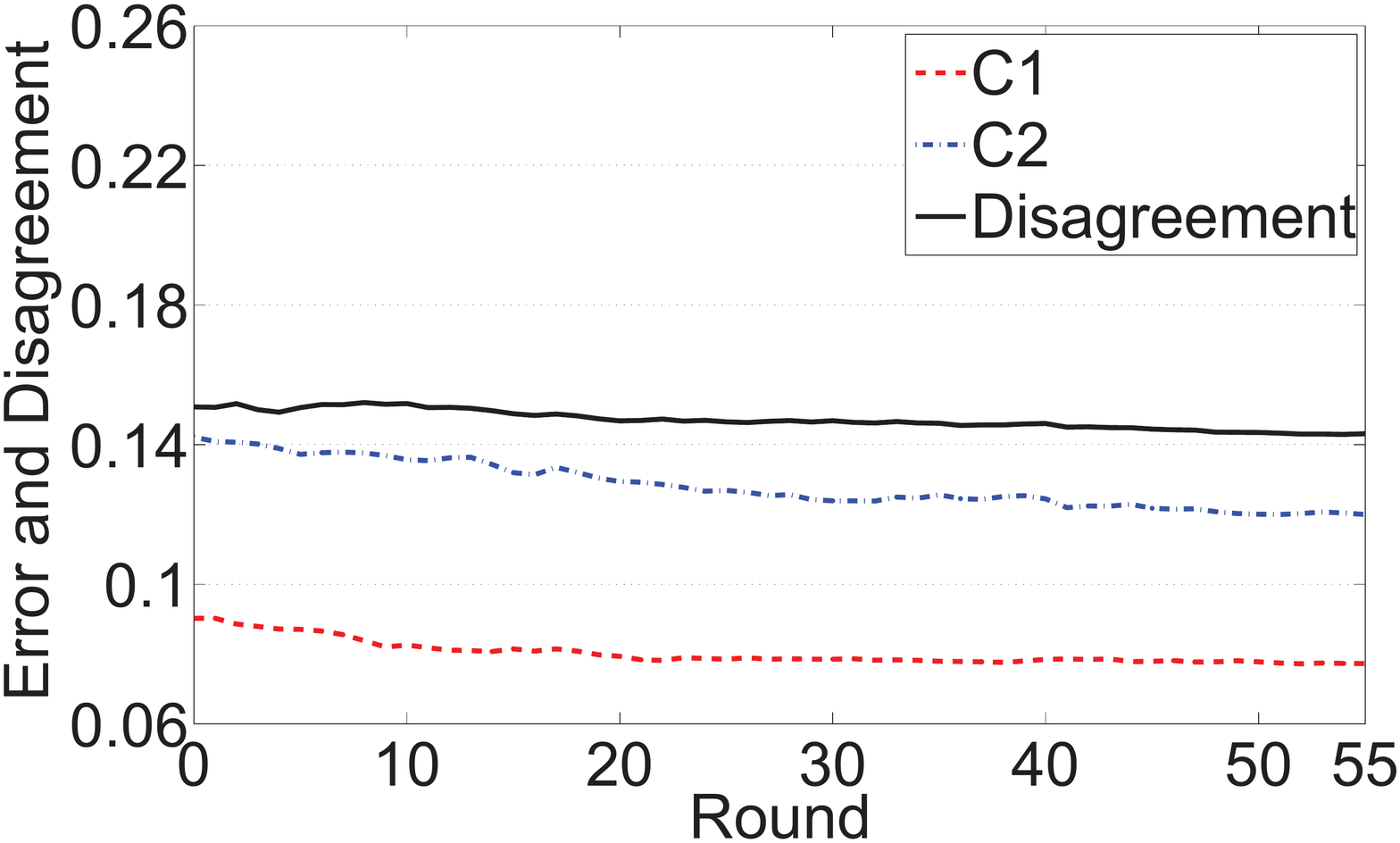}
\centering \mbox{\footnotesize (b)\textit{course}-11-33-1-3}
\end{minipage}
\centering
\begin{minipage}[c]{0.32\linewidth}
\centering
\includegraphics[width = \linewidth]{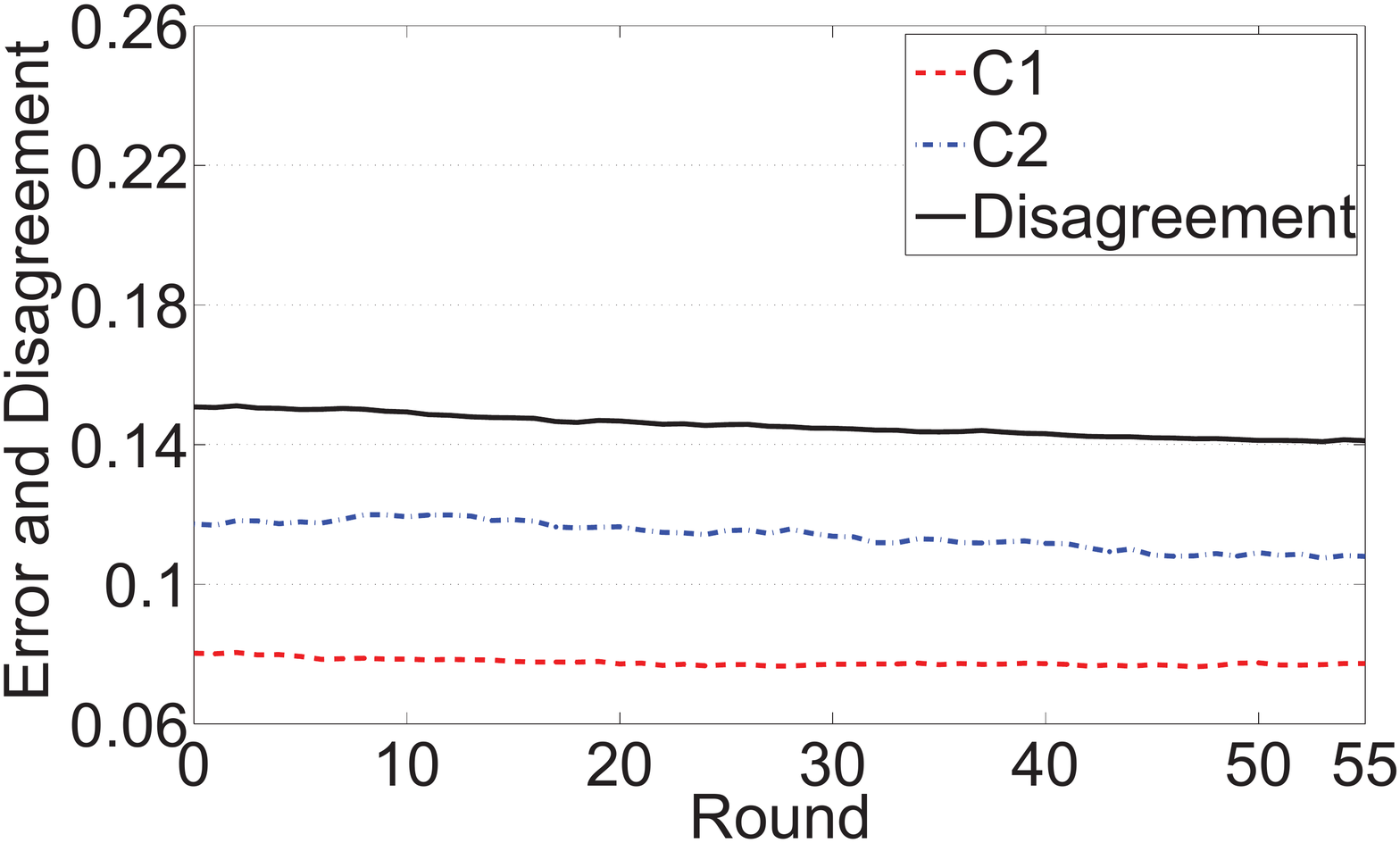}
\centering \mbox{\footnotesize (c)\textit{course}-27-81-1-3}
\end{minipage}\\[+3pt]
\caption{Experimental results of co-training. NaiveBayes is used to train the classifiers. \textit{C1} and \textit{C2} denote the two classifiers trained in the two views, respectively. \textit{Disagreement} denotes the disagreement of the two classifiers. \textit{data}-a-b-c-d means that on the data set \textit{data}, the
initial labeled example set contains $a$ positive examples and $b$ negative examples, and in each round each classifier labels $c$ positive and $d$ negative examples for the other classifier.}\label{Estimated-Round-a}
\end{figure*}

\begin{figure*}[!h]
\centering
\begin{minipage}[c]{0.32\linewidth}
\centering
\includegraphics[width = \linewidth]{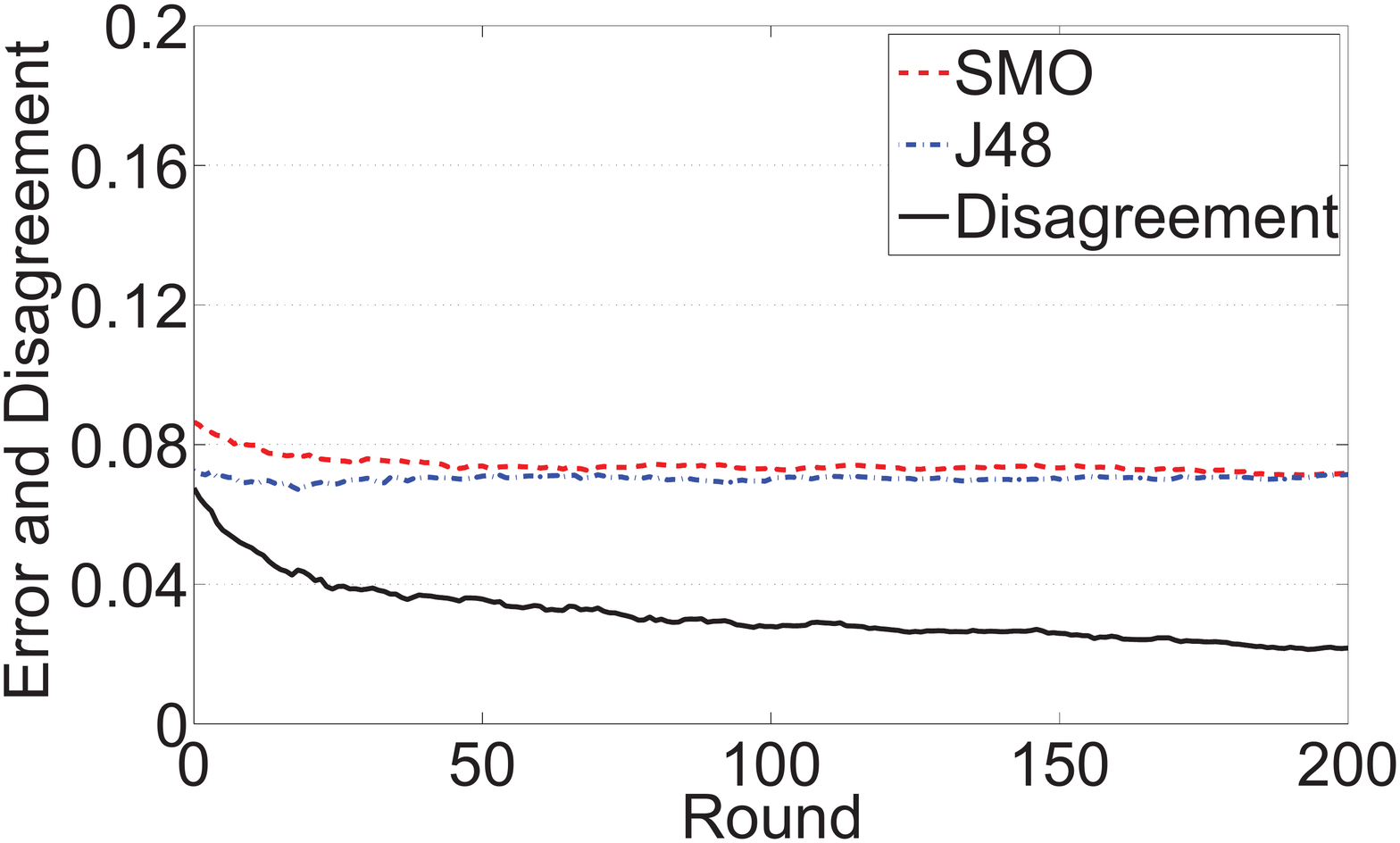}
\centering \mbox{\footnotesize (a)\textit{kr-vs-kp}-60-60-1-1}
\end{minipage}
\centering
\begin{minipage}[c]{0.32\linewidth}
\centering
\includegraphics[width = \linewidth]{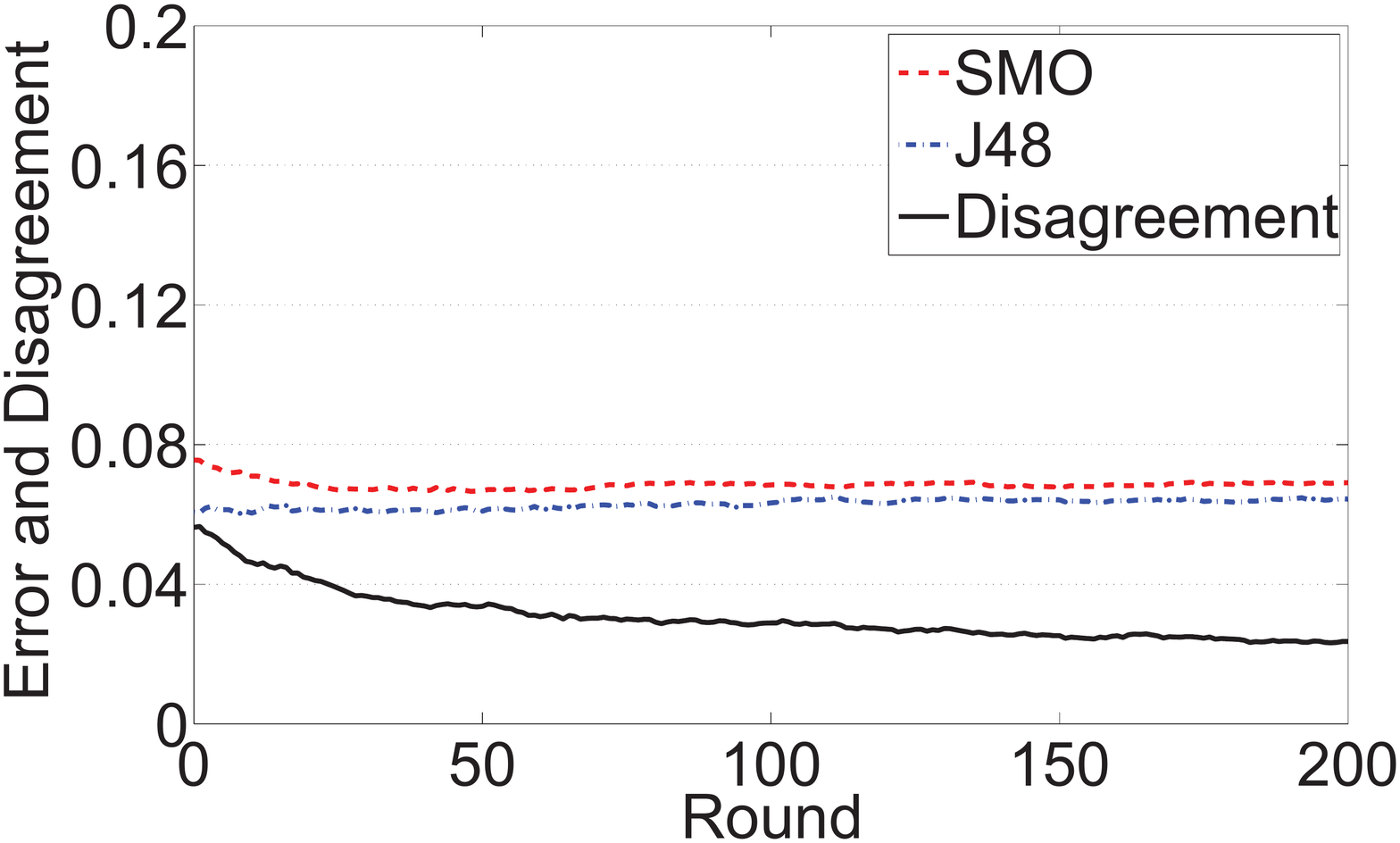}
\centering \mbox{\footnotesize (b)\textit{kr-vs-kp}-75-75-1-1}
\end{minipage}
\centering
\begin{minipage}[c]{0.32\linewidth}
\centering
\includegraphics[width = \linewidth]{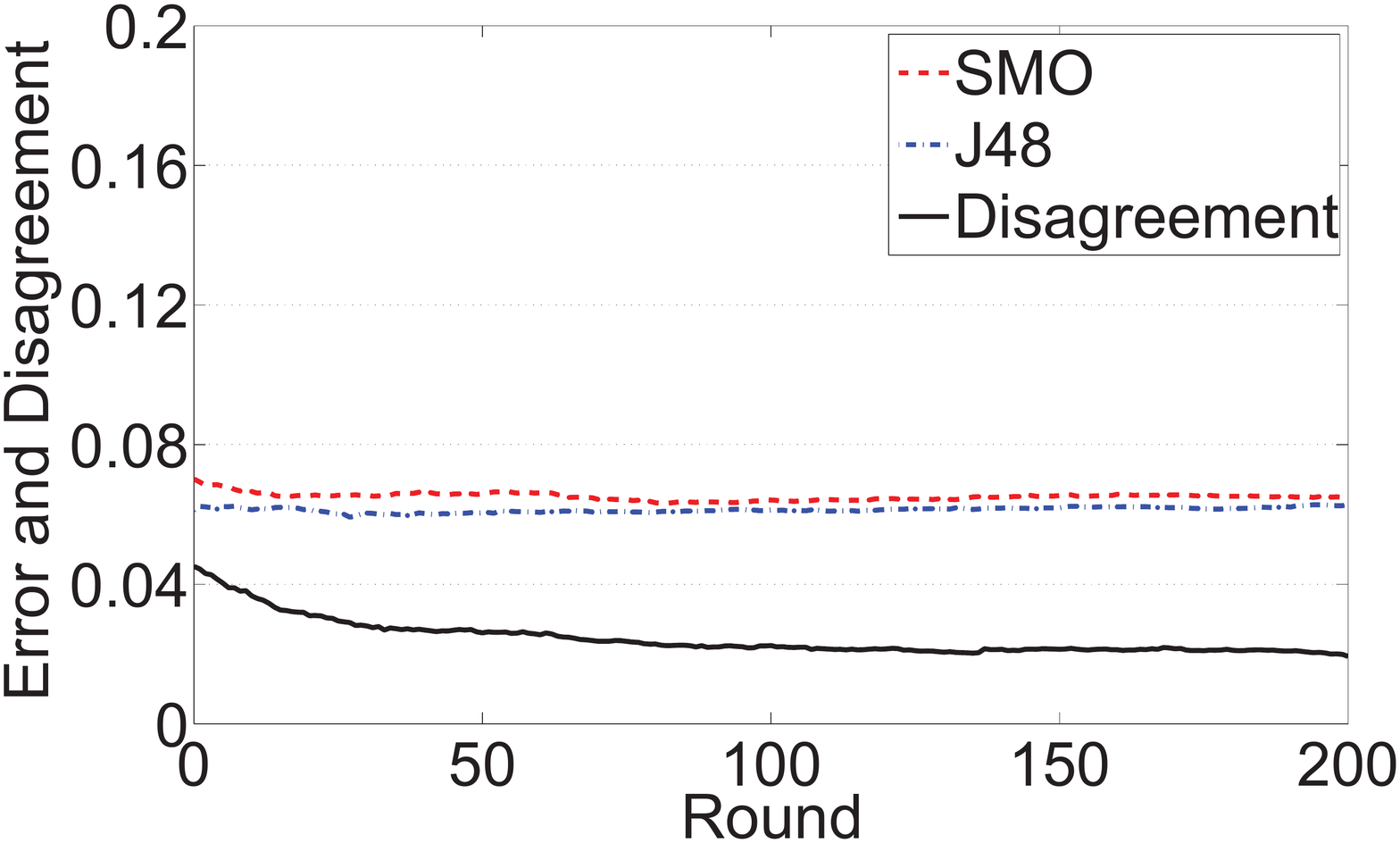}
\centering \mbox{\footnotesize (c)\textit{kr-vs-kp}-90-90-1-1}
\end{minipage}\\[+8pt]
\centering
\begin{minipage}[c]{0.32\linewidth}
\centering
\includegraphics[width = \linewidth]{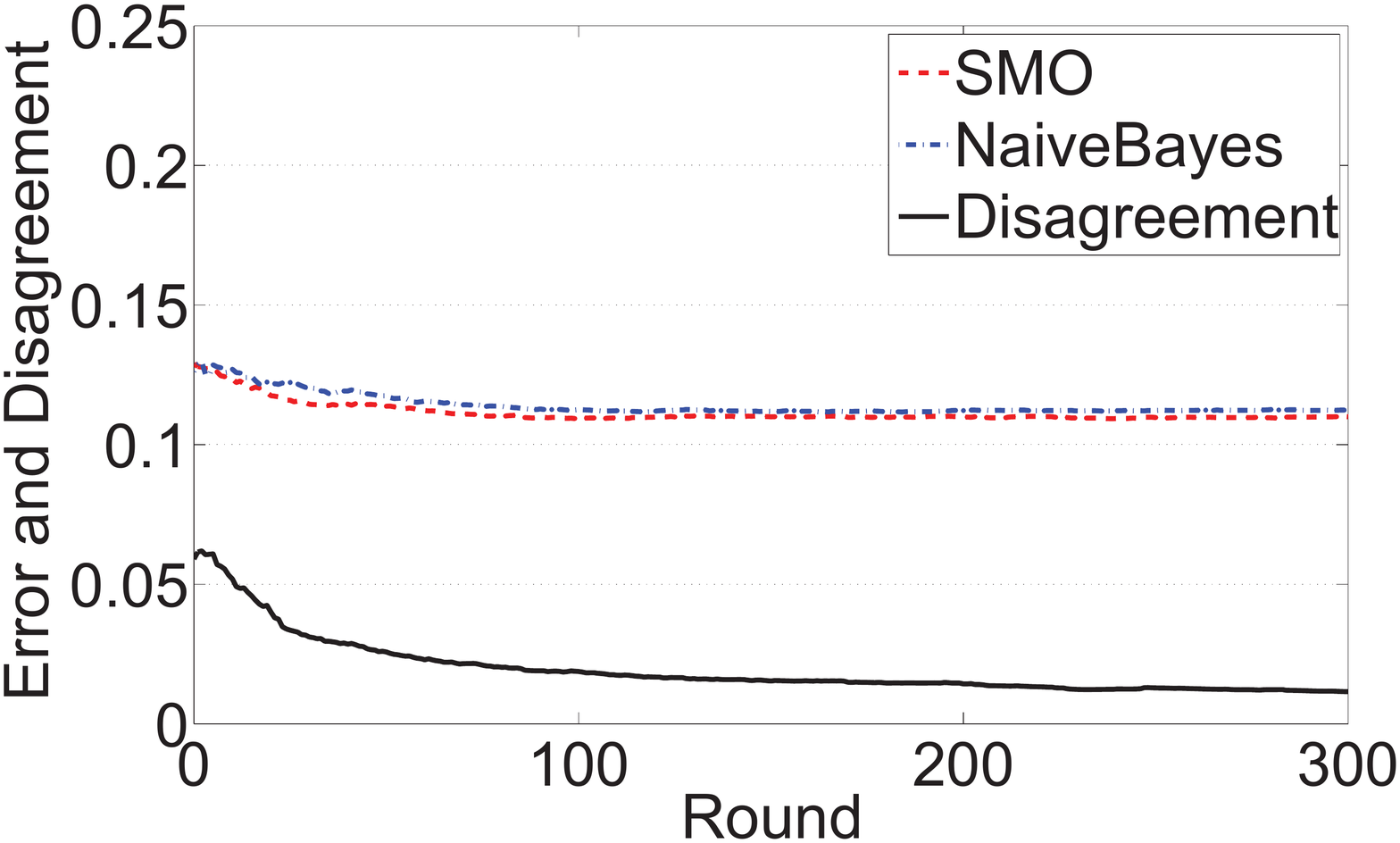}
\centering \mbox{\footnotesize (d)\textit{mushroom}-5-5-1-1}
\end{minipage}
\centering
\begin{minipage}[c]{0.32\linewidth}
\centering
\includegraphics[width = \linewidth]{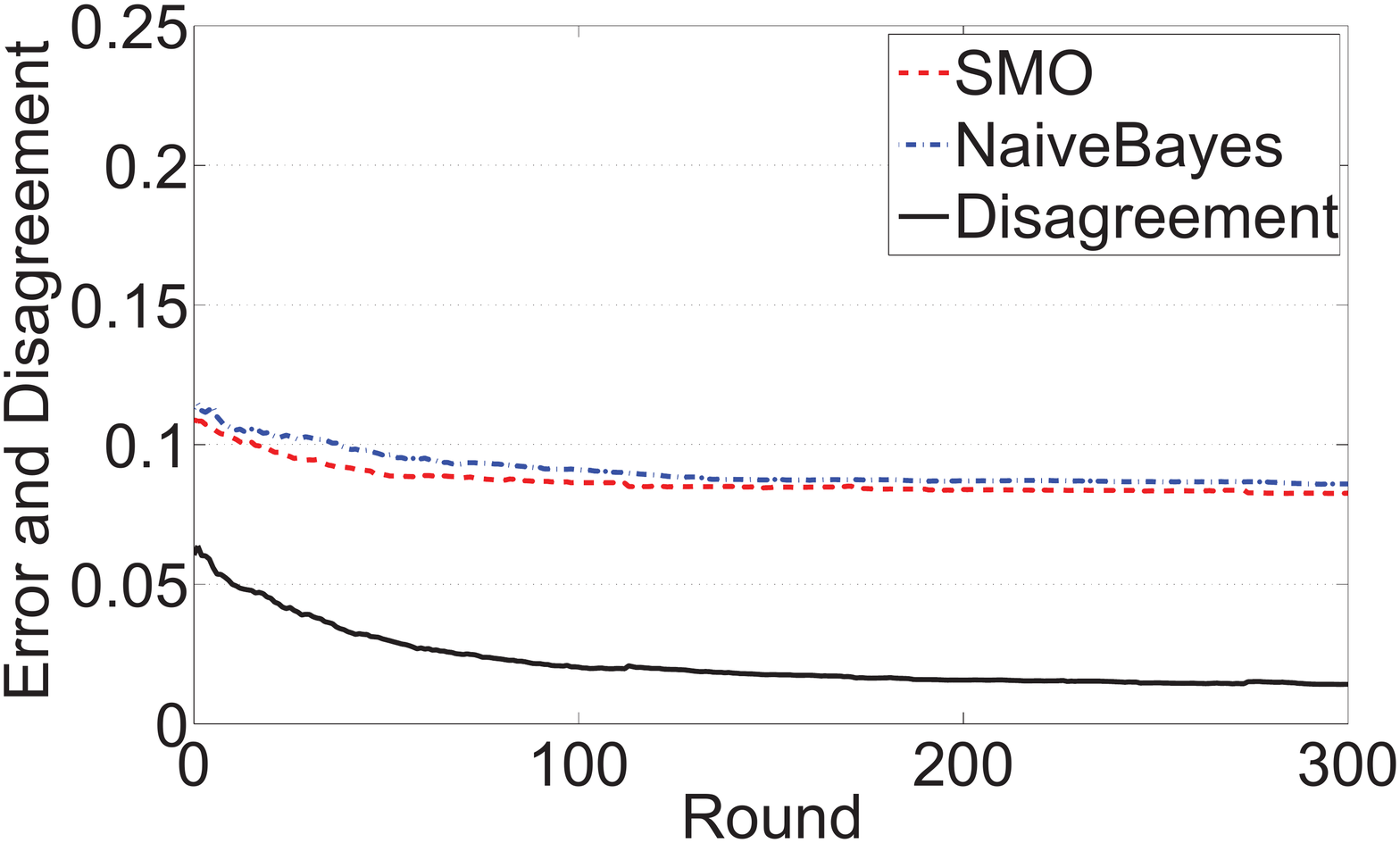}
\centering \mbox{\footnotesize (e)\textit{mushroom}-7-7-1-1}
\end{minipage}
\centering
\begin{minipage}[c]{0.32\linewidth}
\centering
\includegraphics[width = \linewidth]{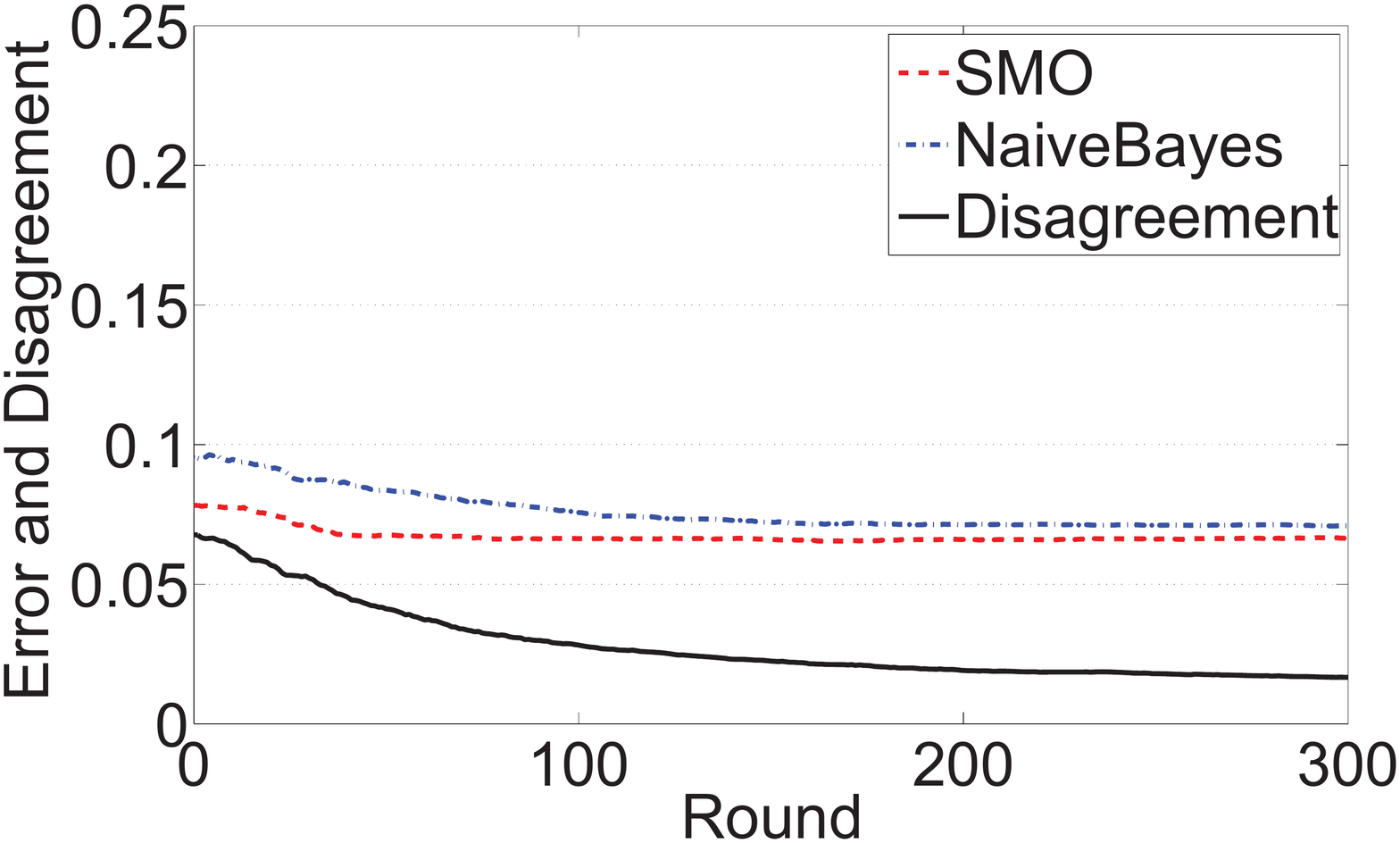}
\centering \mbox{\footnotesize (f)\textit{mushroom}-13-13-1-1}
\end{minipage}\\[+8pt]
\centering
\begin{minipage}[c]{0.32\linewidth}
\centering
\includegraphics[width = \linewidth]{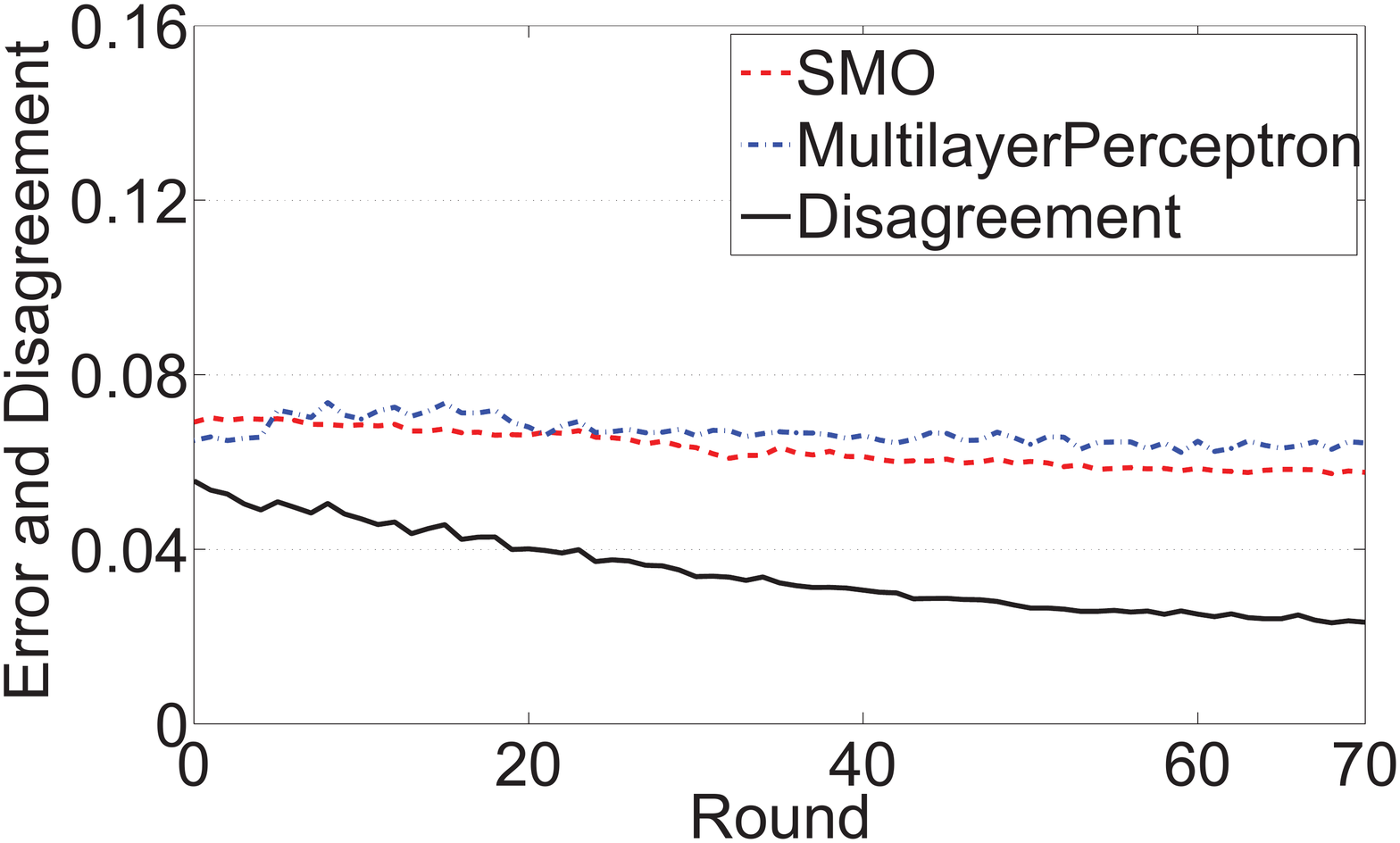}
\centering \mbox{\footnotesize (g)\textit{tic-tac-toe}-50-50-1-1}
\end{minipage}
\centering
\begin{minipage}[c]{0.32\linewidth}
\centering
\includegraphics[width = \linewidth]{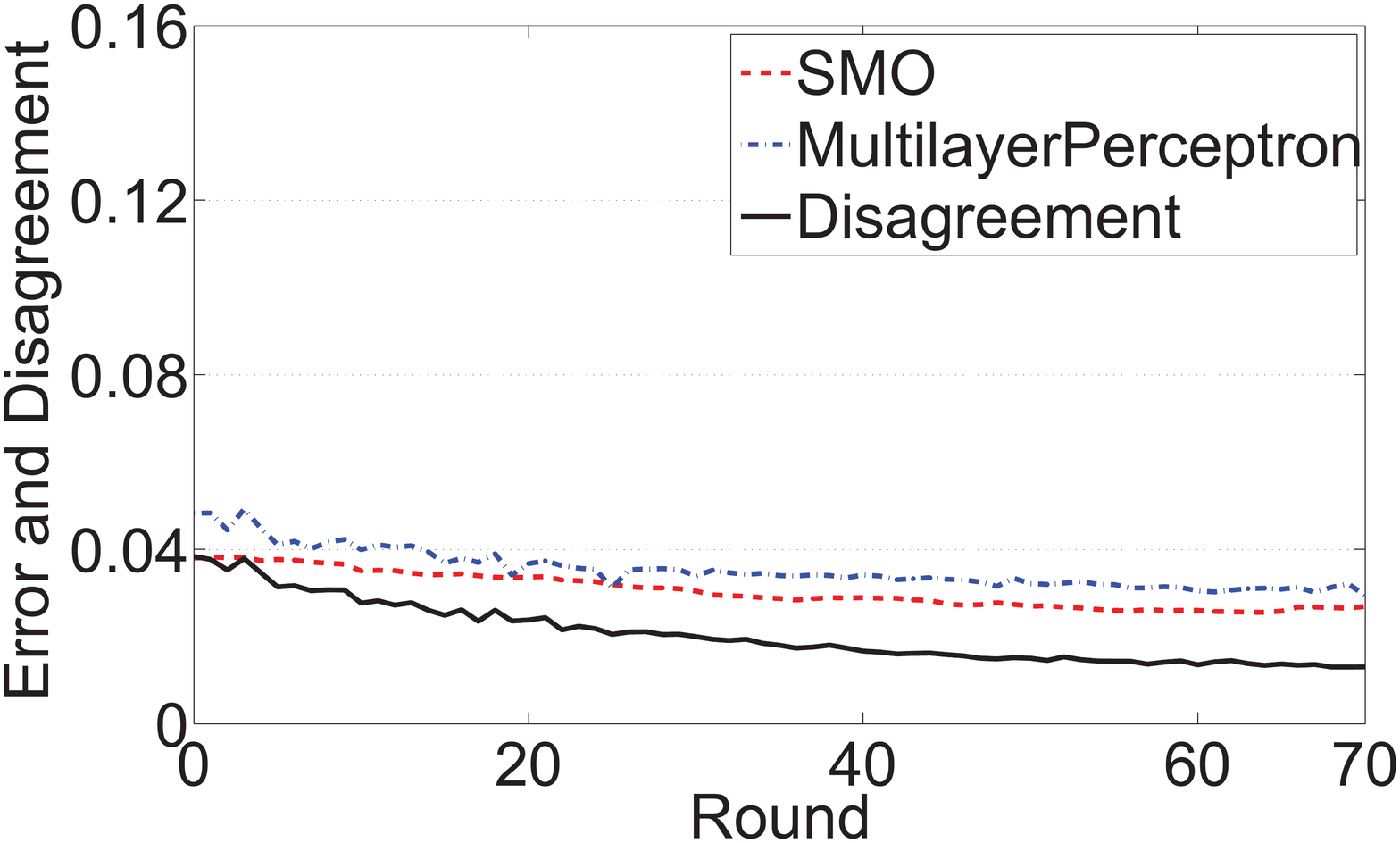}
\centering \mbox{\footnotesize (h)\textit{tic-tac-toe}-55-55-1-1}
\end{minipage}
\centering
\begin{minipage}[c]{0.32\linewidth}
\centering
\includegraphics[width = \linewidth]{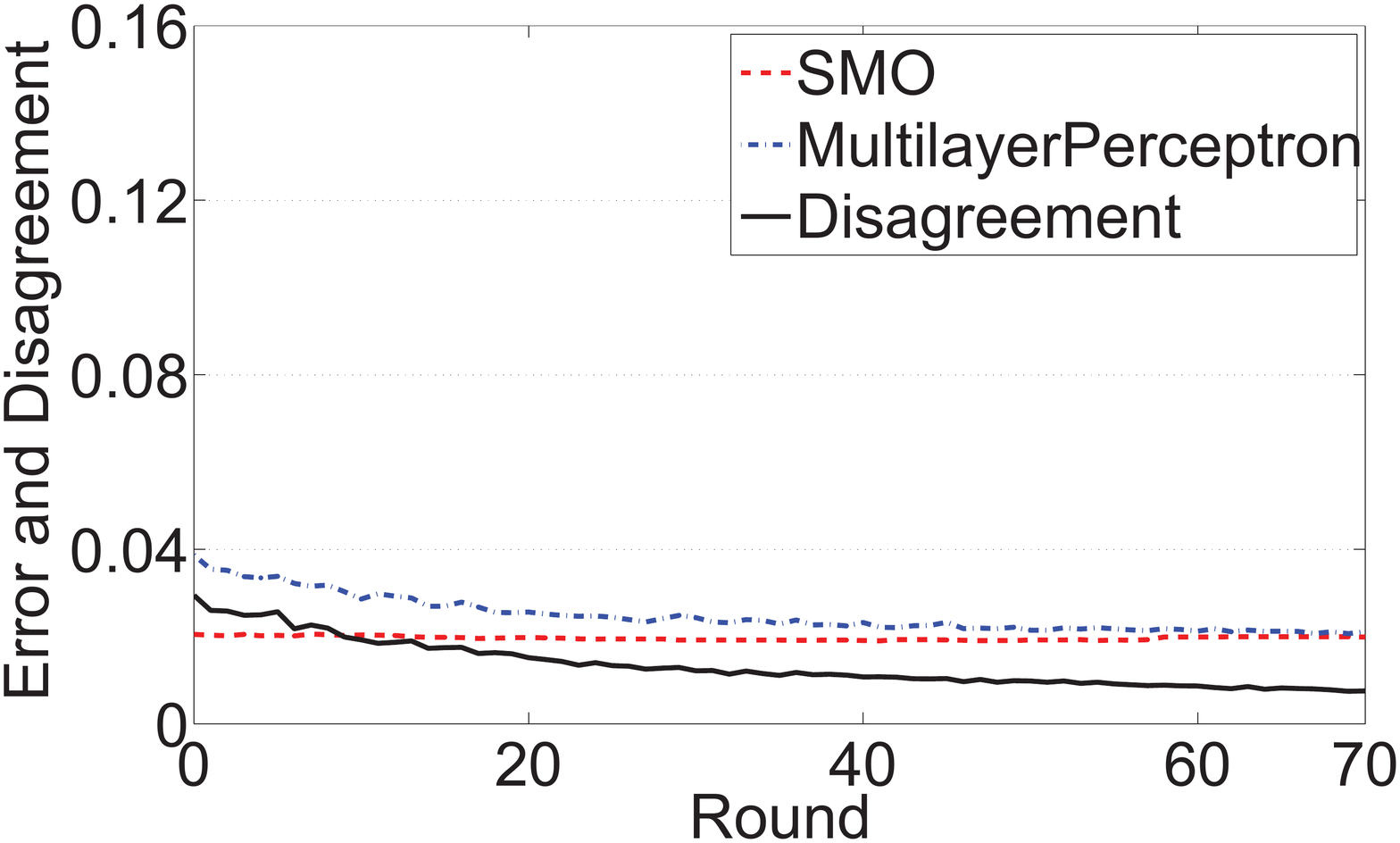}
\centering \mbox{\footnotesize (i)\textit{tic-tac-toe}-60-60-1-1}
\end{minipage}\\
\caption{Experimental results of single-view disagreement-based approach. \textit{SMO} and
\textit{J48} (or \textit{SMO} and \textit{NaiveBayes}, \textit{SMO} and
\textit{MultilayerPerceptron}) denote the two classifiers, respectively.
\textit{Disagreement} denotes the disagreement of the two classifiers.
\textit{data}-a-b-c-d means that on the data set \textit{data}, the initial labeled
example set contains $a$ positive examples and $b$ negative examples, and in each round
each classifier labels $c$ positive and $d$ negative examples for the other
classifier.}\label{Estimated-Round-b}
\end{figure*}

In general, Figures~\ref{Estimated-Round-a} and \ref{Estimated-Round-b} show that except the first several rounds, the disagreement between the classifiers decreases or converges as the process goes on. If the disagreement does not converge, the error rates of the classifiers seem to decrease as the disagreement decreases, e.g., Figures~\ref{Estimated-Round-b}(g) to (i). If the disagreement converges, the error rates of the classifiers also seem to converge, e.g., Figure~\ref{Estimated-Round-b}(d) to (f). This validates that our Theorem~\ref{theorem:disagreement-covergence} which shows that the disagreement and error rates of the classifiers in disagreement-based approaches will converge is meaningful.

\subsubsection{Larger Disagreement Leading to Better Improvement}\label{sec:disagreement-improvement}

\begin{figure*}[!h]
\centering
\begin{minipage}[c]{0.32\linewidth}
\centering
\includegraphics[width = \linewidth]{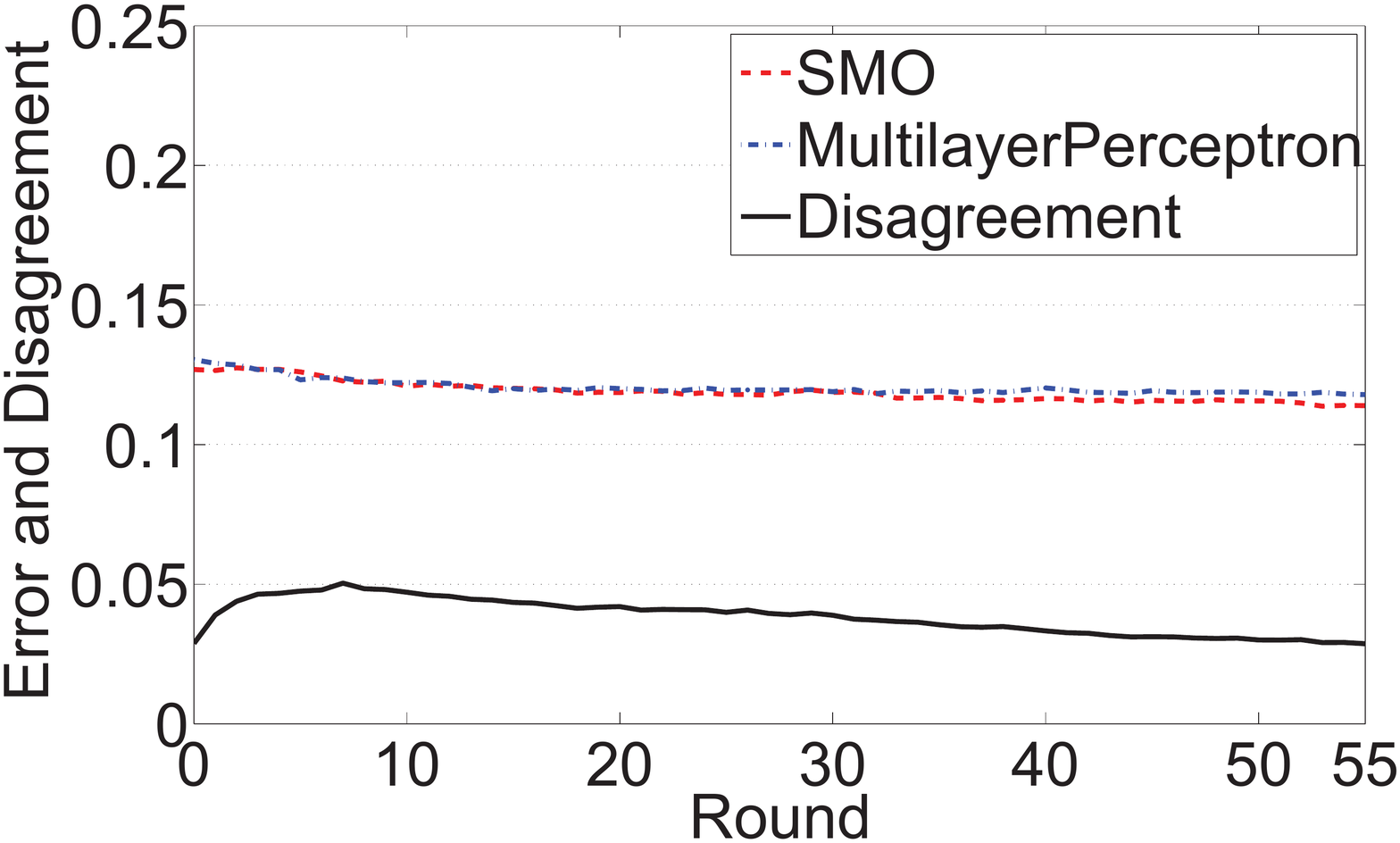}
\centering \mbox{\footnotesize (a)\textit{pages}view-5-15-1-3}
\end{minipage}
\centering
\begin{minipage}[c]{0.32\linewidth}
\centering
\includegraphics[width = \linewidth]{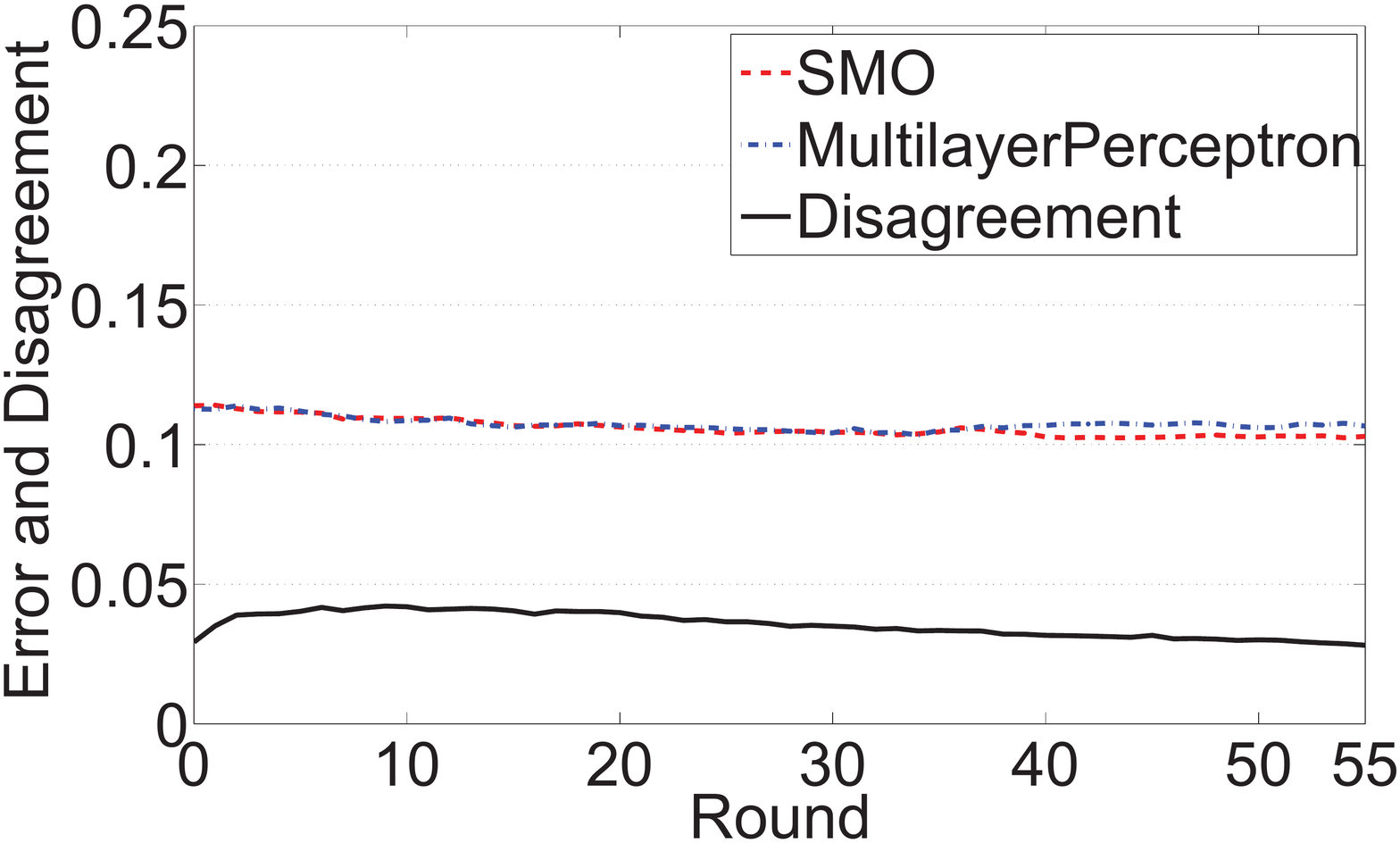}
\centering \mbox{\footnotesize (b)\textit{pages}view-7-21-1-3}
\end{minipage}
\centering
\begin{minipage}[c]{0.32\linewidth}
\centering
\includegraphics[width = \linewidth]{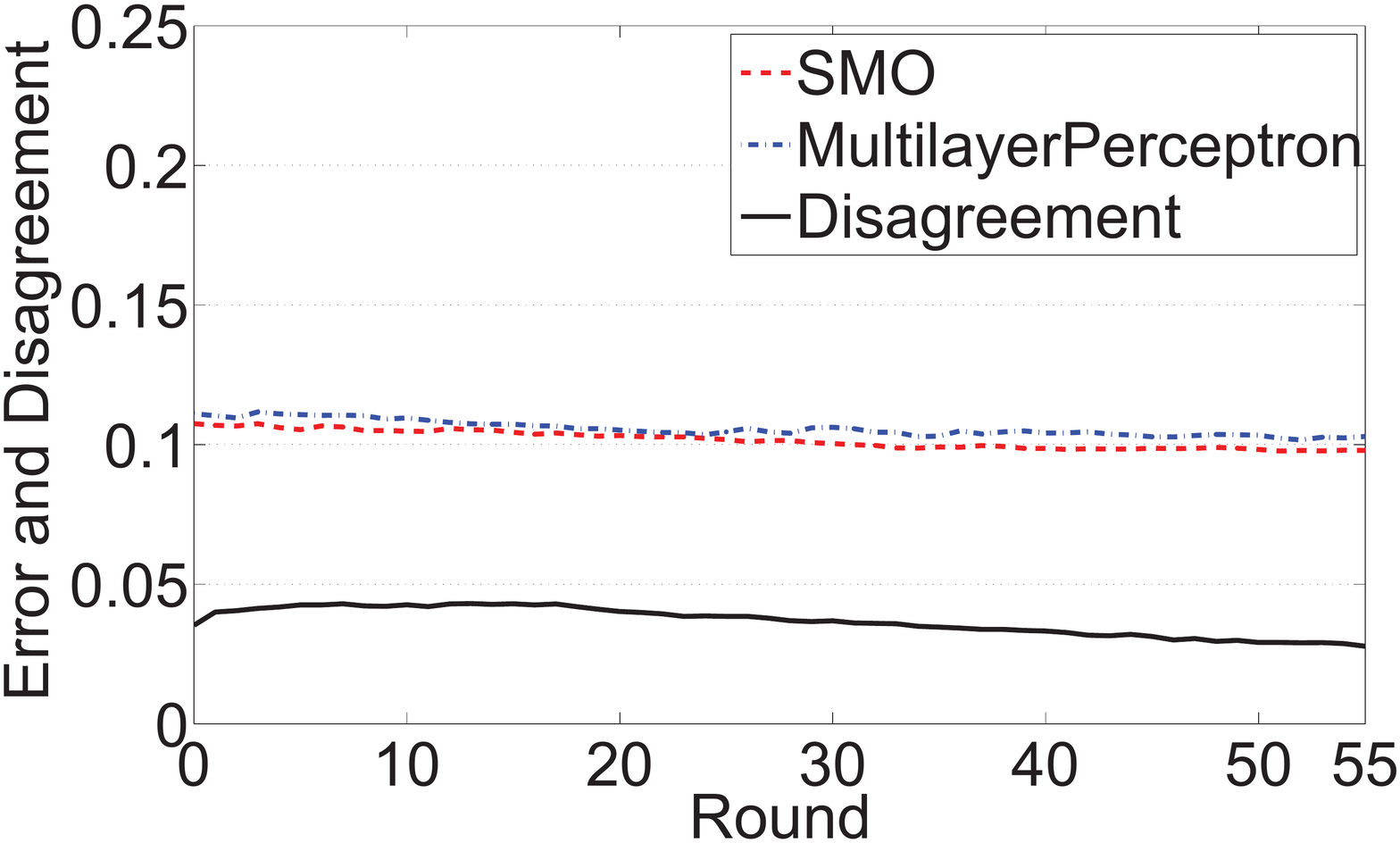}
\centering \mbox{\footnotesize (c)\textit{pages}view-12-36-1-3}
\end{minipage}\\[+10pt]
\centering
\begin{minipage}[c]{0.32\linewidth}
\centering
\includegraphics[width = \linewidth]{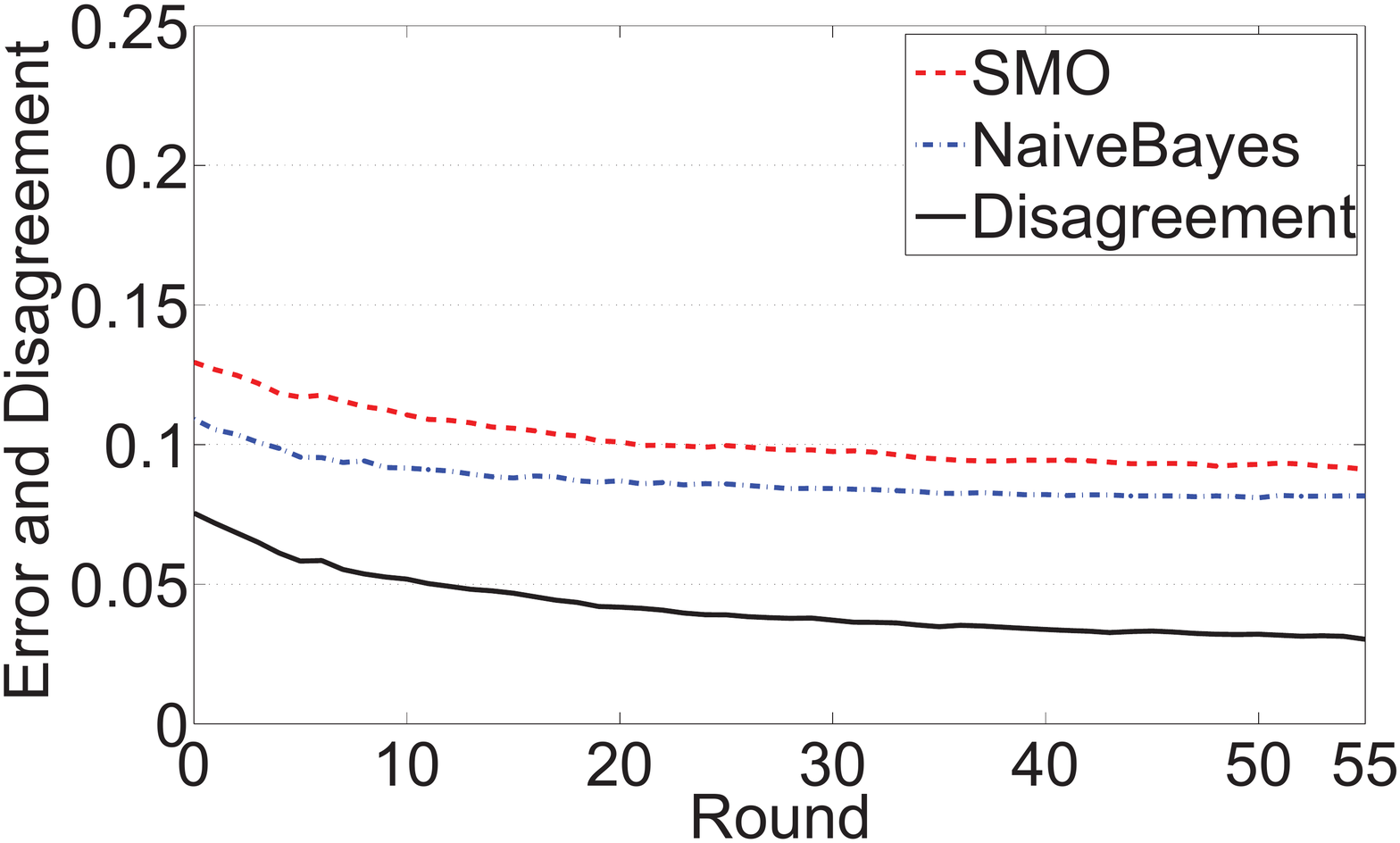}
\centering \mbox{\footnotesize (d)\textit{pages}view-5-15-1-3}
\end{minipage}
\centering
\begin{minipage}[c]{0.32\linewidth}
\centering
\includegraphics[width = \linewidth]{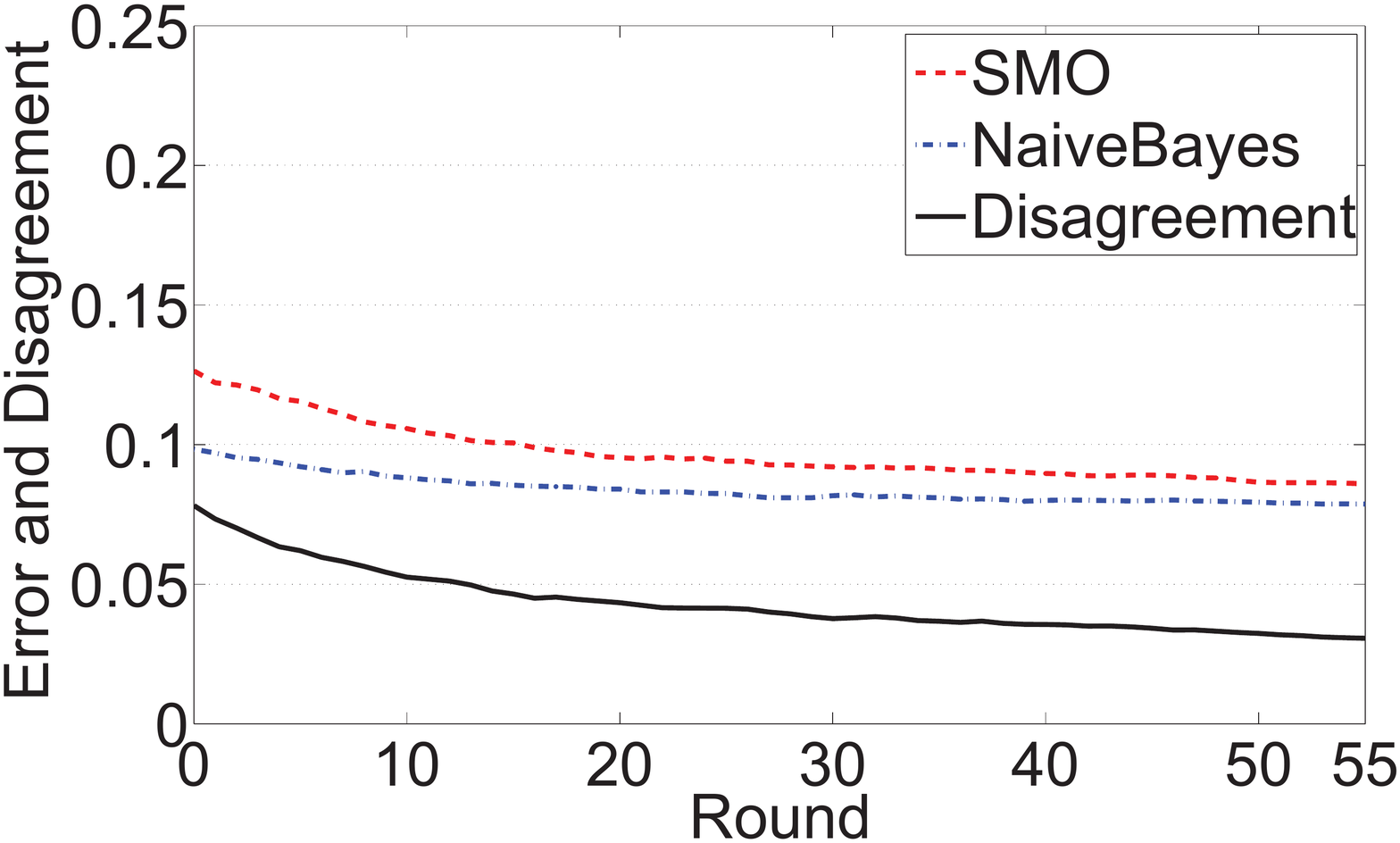}
\centering \mbox{\footnotesize (e)\textit{pages}view-7-21-1-3}
\end{minipage}
\centering
\begin{minipage}[c]{0.32\linewidth}
\centering
\includegraphics[width = \linewidth]{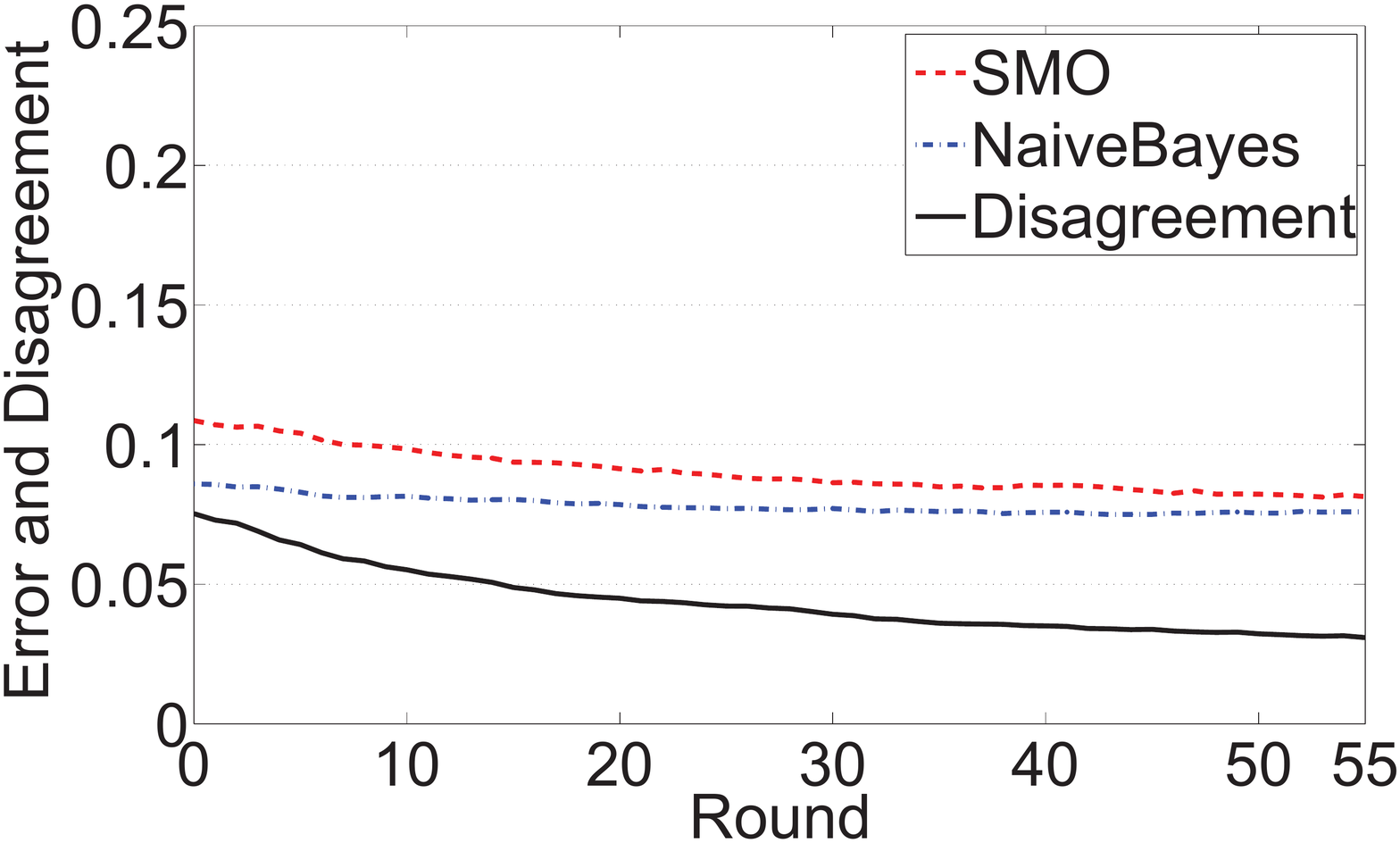}
\centering \mbox{\footnotesize (f)\textit{pages}view-12-36-1-3}
\end{minipage}\\[+3pt]
\caption{Comparing the performance by using two different groups of classifiers on the \textit{pages} view of \textit{course} data set. \textit{SMO} and
\textit{MultilayerPerceptron} (or \textit{SMO} and \textit{NaiveBayes}) denote the two
classifiers, respectively. \textit{Disagreement}
denotes the disagreement of the two classifiers. \textit{data}-a-b-c-d means that on
the data set \textit{data}, the initial labeled example set contains $a$ positive
examples and $b$ negative examples, and in each round each classifier labels $c$
positive and $d$ negative examples for the other classifier.}\label{Estimated-Round-c}
\end{figure*}

In order to study the disagreement further, more experiments are conducted. We run the disagreement-based approach with two different groups of classifiers on the \textit{pages} view of \textit{course} data set. The first group is SMO and MultilayerPerceptron, and the second group is SMO and NaiveBayes. With this experiment, it is clearer whether the classifiers with larger disagreement could lead to better improvement. The results are depicted in Figure~\ref{Estimated-Round-c}. Note that Figures~\ref{Estimated-Round-c}(a) to (c) use the same group of classifiers under different sizes of $L$, while Figures~\ref{Estimated-Round-c}(d) to (e) use another group of classifiers under different sizes of $L$.

By comparing Figures~\ref{Estimated-Round-c}(a) with (d), (b) with (e) and (c) with (f), it can be found that the disagreement between the two classifiers trained by the second group is larger than that trained by the first group. Note that, SMO
appears in both groups, and its improvement is larger in
Figures~\ref{Estimated-Round-c}(d) to (f) than that in
Figures~\ref{Estimated-Round-c}(a) to (c), respectively. This validates that our result in Section~\ref{sec:upper-bound} is meaningful.

\section{Sufficient and Necessary Condition}\label{sec:sufficient-necessary}
All previous theoretical analyses on co-training focus on the sufficient condition, i.e., under what condition co-training could work, so a fundamental issue may arise: what is the sufficient and necessary condition for co-training to succeed? In this section we will study what the sufficient and necessary condition is. For the convenience of description, we first suppose that the data have two views and then discuss how to generalize it to the case where the data have only one view.

Let $\mathcal {X}=\mathcal {X}_{1}\times \mathcal {X}_{2}$ denote the instance space, where $\mathcal {X}_{1}$ and $\mathcal {X}_{2}$
correspond to the two views, respectively. $\mathcal {Y}=\{-1, +1\}$ is
the label space, $L=\{(\langle x^{1}_{1},x^{1}_{2}\rangle, y^{1}), \cdots, (\langle x^{l}_{1},x^{l}_{2}\rangle, y^{l}) \}\subset {\mathcal {X} \times \mathcal {Y}}$ are the labeled data, $U=\{\langle x^{l+1}_{1},x^{l+1}_{2}\rangle,\\ \cdots, \langle x^{l+|U|}_{1},x^{l+|U|}_{2}\rangle\} \subset
\mathcal{X}$ are the unlabeled data. The labeled data $L$ independently and identically come from some unknown distribution $\mathcal {D}$, whose marginal distribution on $\mathcal {X}$ is $\mathcal {D}_{\mathcal {X}}$, and the unlabeled data $U$ independently and identically come from $\mathcal {D}_{\mathcal {X}}$. $c=(c_{1}, c_{2})$ is the target concept, where $c_{1}$ and $c_{2}$ are the target concepts in the two views, respectively, i.e., $c_{1}(x_{1})=c_{2}(x_{2})=y$ for any example $(\langle x_{1},x_{2}\rangle, y)$.

\subsection{Graph View of Co-Training}
When Blum and Mitchell \cite{Blum:Mitchell1998} proposed co-training, they gave a bipartite graph over two views for intuition. The left-hand side of the bipartite graph has one node for each instance in view $\mathcal {X}_{1}$ and the right-hand side has one node for each instance in view $\mathcal {X}_{2}$.
Since each instance $x=\langle x_{1},x_{2}\rangle$ has two views, there is an undirected edge between the left-hand $x_{1}$ and the right-hand $x_{2}$. Considering that each view $\mathcal {X}_{v}$ ($v=1,2$) corresponds to a graph where each node is one instance and there is an undirected edge between two nodes if they have the same class label, it naturally inspires us to study co-training with the graphs over two views.

Generally, assigning a label to an unlabeled instance $x_{v}^{t}$ ($v=1,2$) based on a
labeled example $x_{v}^{s}$ can be viewed as estimating the conditional probability
$P\big(y(x_{v}^{t})=y(x_{v}^{s})|x_{v}^{t},x_{v}^{s}\big)$. For controlling the confidence, we can set a threshold $\eta_{v}>0$
(generally $\eta_{v}=1/2$). If
$P\big(y(x_{v}^{t})=y(x_{v}^{s})|x_{v}^{t},x_{v}^{s}\big)< \eta_{v}$, we set
$P\big(y(x_{v}^{t})=y(x_{v}^{s})|x_{v}^{t},x_{v}^{s}\big)=0$. Note that
$P\big(y(x_{v}^{t})=y(x_{v}^{s})|x_{v}^{t},x_{v}^{s}\big)=0$ does not mean that
$x_{v}^{s}$ and $x_{v}^{t}$ must have different labels, since the data may not provide any helpful information for estimating the similarity between $x_{v}^{s}$ and $x_{v}^{t}$. We can assign a
label to $x_{v}^{t}$ according to
$P\big(y(x_{v}^{t})=y(x_{v}^{s})|x_{v}^{t},x_{v}^{s}\big)$ and the label of
$x_{v}^{s}$. For two labeled examples $x_{v}^{w}$ and $x_{v}^{q}$, if they have the
same label, we set $P\big(y(x_{v}^{w})=y(x_{v}^{q})|x_{v}^{w},x_{v}^{q}\big)=1$ and
otherwise $P\big(y(x_{v}^{w})=y(x_{v}^{q})|x_{v}^{w},x_{v}^{q}\big)=0$. Let each entry
$P_{v}^{ij}$ of the matrix $P_{v}$ correspond to
$P\big(y(x_{v}^{i})=y(x_{v}^{j})|x_{v}^{i},x_{v}^{j}\big)$ ($1\leq i,j \leq l+|U|$) and
$f_{v} =\left[\begin{array}{c}
f_{v}^{L}\\
f_{v}^{U}
\end{array}\right] = \left[\begin{array}{c}
Y^{L}\\
0
\end{array}\right]$. Without loss of generality, $P_{v}$ can be normalized to a probabilistic transition matrix according to Equation~\ref{eq:normalize}.
\begin{equation}\label{eq:normalize}
    P_{v}^{ij} \leftarrow P_{v}^{ij} / \sum\nolimits_{t=1}^{l+|U|}P_{v}^{it}
\end{equation}
Then the labels can be propagated from labeled examples to unlabeled instances
according to the process \cite{ZhuThesis05}: 1) Propagate $f_{v}=P_{v}f_{v}$; 2) Clamp
the labeled data $f_{v}^{L}=Y^{L}$; 3) Repeat from step 1 until $f_{v}$ converges. The
labels of unlabeled instances in $U$ can be assigned according to $sign(f_{v}^{U})$.
For some unlabeled instance $x_{v}^{t}$ if $f_{v}^{t}=0$, it means that label
propagation on graph $P_{v}$ has no idea on $x_{v}^{t}$. So in each view the classifier can be viewed as label propagation from labeled examples to unlabeled instances on
graph $P_{v}$. The error rate
$\texttt{err}(f_{v}^{U})$, the accuracy $\texttt{acc}(f_{v}^{U})$ and the uncertainty $\bot(f_{v}^{U})$
of this graph-based approach can be counted on $U$ as $\texttt{acc}(f_{v}^{U})=P_{x_{v}^{t}\in
U}[f_{v}^{U}(x_{v}^{t})\cdot c_{v}(x_{v}^{t})>0]$, $\texttt{err}(f_{v}^{U})=P_{x_{v}^{t}\in
U}[f_{v}^{U}(x_{v}^{t})\cdot c_{v}(x_{v}^{t})<0]$ and $\bot(f_{v}^{U})=P_{x_{v}^{t}\in
U}[f_{v}^{U}(x_{v}^{t})=0]$. In one view, the labels can be propagated from initial
labeled examples to some unlabeled instances in $U$ and these newly labeled examples
can be added into the other view. Then the other view can propagate the labels of
initial labeled examples and these newly labeled examples to the remaining unlabeled
instances in $U$. This process can be repeated until the stopping condition is met.
Thus, co-training can be re-described as the combinative label
propagation process over two views in Algorithm~\ref{alg:graph-based-co-training}, where $S_{1}^{k}\oplus S_{2}^{k} =(S_{1}^{k}-S_{2}^{k})\cup (S_{2}^{k}-S_{1}^{k})$.

\begin{algorithm}[t]
\begin{algorithmic}
   \STATE {\bfseries Input:} Labeled data $L$, unlabeled data $U$ and probabilistic transition matrices $P_{1}$ and $P_{2}$.
   \STATE {\bfseries Output:} $f^{U}_{v}$ corresponding to $S_{v}^{k}$ ($v=1,2$).
   \STATE {\bfseries Initialize:} Perform label propagation from labeled data $L$ to unlabeled data $U$ on graph $P_{v}$ and get the labeled examples set $S^{0}_{v}$;
   \FOR{$k=0, 1, 2, \cdots$}
   \IF{$S_{1}^{k}\oplus S_{2}^{k}\neq\emptyset$}
   \STATE Perform label propagation from labeled examples set $S_{3-v}^{k}\cap (U-S_{v}^{k})$ to unlabeled data $U-S_{v}^{k}$ on graph $P_{v}$ and get the labeled examples set $T_{v}^{k}$;\\
   $S_{v}^{k+1}=S_{v}^{k} \cup T_{v}^{k}$;
   \ELSE
   \STATE \textbf{break};
   \ENDIF
   \ENDFOR
\end{algorithmic}
\caption{Graph-based description of co-training}
\label{alg:graph-based-co-training}
\end{algorithm}

Label propagation needs a graph which is represented by the matrix $P$. In this section, we focus on co-training with two graphs $P_{1}$ and $P_{2}$ constructed from the two views. How to construct a graph is an important issue studied in graph-based approaches \cite{JebaraWC09,MaierLH08} and is beyond the scope of this article.

\subsection{Co-Training with Perfect Graphs}
First, we assume that $P_{v}$ ($v=1,2$) is \textit{perfect graph},
i.e., if $P\big(y(x_{v}^{t})=y(x_{v}^{s})|x_{v}^{t},x_{v}^{s}\big)>0$, $x_{v}^{s}$ and
$x_{v}^{t}$ must have the same label. It means that the classifier is either ``confident
of labeling'' or ``having no idea''. Before showing the sufficient and necessary
condition for co-training with perfect graphs to succeed, we need Lemma~\ref{lemma:graph1} to
indicate the relationship between label propagation and connectivity.

\begin{lemma}\label{lemma:graph1}
Suppose that $P$ is perfect graph. Unlabeled instance $x^{t_{0}}$ can be labeled by label propagation on graph $P$ if and only if it can be connected with some labeled
example $x^{t_{r}}$ in graph $P$ through a path $R$ in the form of
$V_{R}=\{t_{0},t_{1}, \cdots, t_{r} \}$, where $P^{t_{\rho}t_{\rho+1}}>0$
$(\rho=0,\cdots,r-1)$.
\end{lemma}

\begin{proof}
It is well known \cite{ZhuThesis05} that the label propagation process has the
following closed form solution for each connected component in graph $P$.
\begin{equation}\label{eq:output1}
    f^{U_{\theta}}=(I-P^{U_{\theta}U_{\theta}})^{-1}P^{U_{\theta}L_{\theta}}Y^{L_{\theta}}.
\end{equation}
Here $U_{\theta}\cup L_{\theta}$ is a connected component $\pi_{\theta}$ in graph $P$,
where $U_{\theta}\subseteq U$ and $L_{\theta} \subseteq L$.

If an unlabeled instance $x^{t}$ cannot be connected with any labeled example, with
respect to Equation~\ref{eq:output1}, we know that $f^{t}=0$. If $x^{t_{0}}$ can be connected
with some labeled example $x^{t_{r}}$ through a path $R$ in the form of
$V_{R}=\{t_{0},t_{1}, \cdots, t_{r} \}$, considering that $P$ is a perfect graph we get
$|f^{t_{0}}|\geq\prod\nolimits_{\rho=0}^{r-1}P^{t_{\rho}t_{\rho+1}}|y^{t_{r}}|$. Thus,
$x^{t_{0}}$ can be labeled with label $sign(f^{t_{0}})$ by label propagation.
\end{proof}

Lemma~\ref{lemma:graph1} indicates that when every unlabeled instance can be connected with some labeled example through a path in perfect graph $P$, label propagation on graph $P$ is successful. Now we give Theorem~\ref{theorem:graph1}.

\begin{theorem}\label{theorem:graph1}
Suppose $P_{v}$ $(v=1,2)$ is perfect graph. $f_{v}^{U}({x_{v}^{t}})\cdot
c_{v}(x_{v}^{t})>0$ for all unlabeled instance $x^{t}\in U$ ($t=l+1,\cdots,l+|U|$) if and
only if $S_{1}^{k}\oplus S_{2}^{k}$ is not $\emptyset$ in Algorithm~\ref{alg:graph-based-co-training} until
$S_{v}^{k}=L\cup U$.
\end{theorem}

\begin{proof}
Here we give a proof by contradiction. Suppose that for any unlabeled instance
$x^{t}\in U$ ($t=l+1,\cdots,l+|U|$), $f_{v}^{U}({x_{v}^{t}})\cdot c_{v}(x_{v}^{t})>0$.
From Lemma~\ref{lemma:graph1} and the process in Algorithm~\ref{alg:graph-based-co-training} we know that for any unlabeled
instance $x^{t_{0}}\in U$, $x^{t_{0}}$ can be connected with some labeled example
$x^{t_{r}}\in L$ through a path $R$ in the form of $V_{R}=\{t_{0},t_{1},\cdots,
t_{r}\}$, where $P_{1}^{t_{\rho}t_{\rho+1}}>0$ or $P_{2}^{t_{\rho}t_{\rho+1}}>0$
$(\rho=0,\cdots,r-1)$. If $S_{1}^{k}\oplus S_{2}^{k}=\emptyset$ while $S_{v}^{k}\neq
L\cup U$, there must exist some unlabeled instances in $U-S_{v}^{k}$. Considering
that $S_{v}^{k}$ are obtained by label propagation on graph $P_{v}$, so from
Lemma~\ref{lemma:graph1} we know that for any unlabeled instance $x^{h}\in U -S_{v}^{k}$,
there is no path between $x^{h}$ and any labeled example $x^{d}\in S_{v}^{k}$ in graph
$P_{v}$, i.e., $P_{v}^{hd}=0$. It is in contradiction with that any unlabeled instance
in $U$ can be connected with some labeled example in $L$ through a path $R$. Therefore,
if $f_{v}^{U}({x_{v}^{t}})\cdot c_{v}(x_{v}^{t})>0$ for all unlabeled instance $x^{t}$,
$S_{1}^{k}\oplus S_{2}^{k}$ is not $\emptyset$ until $S_{v}^{k}=L\cup U$.

Suppose the graph $P_{v}$ contains $\lambda_{v}$ connected components. If one example
in some connected component is labeled, from Lemma~\ref{lemma:graph1} we know that all
unlabeled instances in this connected component can be labeled by label propagation. If
$S_{1}^{k}\oplus S_{2}^{k}$ is not $\emptyset$ until $S_{v}^{k}=L\cup U$, in the $k$-th
iteration of Algorithm~\ref{alg:graph-based-co-training}, the unlabeled instances in at least one connected component of either $P_{1}$ or $P_{2}$ will be labeled by label propagation. Thus, after at most $\lambda_{1}+\lambda_{2}$ iterations all unlabeled instances in $U$ can be
assigned with labels by the process in Algorithm~\ref{alg:graph-based-co-training}. Considering that $P_{v}$ in each view is perfect graph, we get that for any unlabeled instance $x^{t}\in U$,
$f_{v}^{U}({x_{v}^{t}})\cdot c_{v}(x_{v}^{t})>0$.
\end{proof}

\noindent \textbf{Remark:} Theorem~\ref{theorem:graph1} provides the sufficient and necessary condition for co-training
with perfect graphs to succeed. With this theorem, for tasks with two views, if two perfect graphs can be constructed from the two views, we can decide whether co-training will be successful.

\subsection{Co-Training with Non-perfect Graphs}
In many real-world applications, it is generally hard to construct a perfect graph. We will discuss the case when the perfect graph assumption is waived in this section.

In label propagation on \textit{non-perfect graph}, an unlabeled instance may be connected with labeled examples belonging to different classes. As discussed in the proof of Lemma~\ref{lemma:graph1}, the label propagation for each connected component $\pi_{\theta}$ in graph $P$ has the closed form of
$f^{U_{\theta}}=(I-P^{U_{\theta}U_{\theta}})^{-1}P^{U_{\theta}L_{\theta}}
Y^{L_{\theta}}$.
Let $A=(I-P^{U_{\theta}U_{\theta}})^{-1}$, we can get Equation~\ref{eq:output2} from Equation~\ref{eq:output1}.
\begin{eqnarray}
\label{eq:output2}
    f^{t}=\sum\nolimits_{s\in L_{\theta}} \sum\nolimits_{j \in U_{\theta}} A^{tj}P^{js}Y^{s}
    ~~~~(t\in U_{\theta})
\end{eqnarray}
From Equation~\ref{eq:output2} we know that in each connected component the contribution of the
labeled example $x^{s}$ $(s \in L_{\theta})$ to the unlabeled instance $x^{t}$ $(t\in U_{\theta})$ is $\sum_{j \in U_{\theta}}A^{tj}P^{js}$. Now we define the \textit{positive contribution} and \textit{negative contribution} to an unlabeled instance.

\begin{definition}\label{def:contribution}
Let $L_{\theta}$ denote the labeled examples and $U_{\theta}$ denote the unlabeled instances belonging to the connected component $\pi_{\theta}$ in graph $P$. For an
unlabeled instance $x^{t}$ $(t\in U_{\theta})$, the positive contribution to $x^{t}$ is
\begin{eqnarray}\label{eq:positive}
 \sum\nolimits_{Y^{s}=y^{t}} \sum\nolimits_{j \in U_{\theta}} A^{tj}P^{js}|Y^{s}|
\end{eqnarray}
and the negative contribution to $x^{t}$ is
\begin{eqnarray}\label{eq:negative}
 \sum\nolimits_{Y^{s}\neq y^{t}} \sum\nolimits_{j \in U_{\theta}} A^{tj}P^{js}|Y^{s}|.
\end{eqnarray}
\end{definition}
If the positive contribution is larger than the negative contribution, the unlabeled instance $x^{t}$ will be labeled correctly by label propagation \footnote{We neglect the probability mass on the instances for which the non-zero positive contribution is
equal to the non-zero negative contribution in this article.}. Now we give
Theorem~\ref{theorem:graph2} for co-training with non-perfect graphs.

\begin{theorem}\label{theorem:graph2}
Suppose $P_{v}$ $(v=1,2)$ is non-perfect graph. $f_{v}^{U}({x_{v}^{t}})\cdot
c_{v}(x_{v}^{t})>0$ for all unlabeled instance $x^{t}\in U$ ($t=l+1,\cdots,l+|U|$) if and
only if both (1) and (2) hold in Algorithm~\ref{alg:graph-based-co-training}: (1) $S_{1}^{k}\oplus S_{2}^{k}$ is not
$\emptyset$ until $S_{v}^{k}=L\cup U$; (2) For any unlabeled instance in the connected
component $\pi^{\theta_{k}}_{v}$, where $\pi^{\theta_{k}}_{v}\subseteq (U-S_{v}^{k})$
and $\pi^{\theta_{k}}_{v} \cap S^{k}_{3-v} \neq \emptyset$, its positive contribution
is larger than its negative contribution.
\end{theorem}

\begin{proof}
Here we give a proof by contradiction. Suppose for any unlabeled instance $x^{t}\in U$,
$f_{v}^{U}({x_{v}^{t}})\cdot c_{v}(x_{v}^{t})>0$. If $S_{1}^{k}\oplus S_{2}^{k}$ is
equal to $\emptyset$ while $S_{v}^{k}\neq L\cup U$, for any unlabeled instance
$x=\langle x_{1}, x_{2}\rangle$ in $U-S_{v}^{k}$, $f_{v}^{U}(x_{v})=0$. It is in contradiction with
$f_{v}^{U}(x_{v})\cdot c_{v}(x_{v}) >0$. If for some unlabeled instance $x$ in the
connected component $\pi^{\theta_{k}}_{v}$, where $\pi^{\theta_{k}}_{v}\subseteq
(U-S_{v}^{k})$ and $\pi^{\theta_{k}}_{v} \cap S^{k}_{3-v} \neq \emptyset$, its positive
contribution is no larger than negative contribution, $f_{v}^{U}(x_{v})\cdot
c_{v}(x_{v})\leq 0$. It is also in contradiction with $f_{v}^{U}(x_{v})\cdot
c_{v}(x_{v})>0$.

If the conditions (1) and (2) hold, with Definition~\ref{def:contribution} it is easy to get
that for any unlabeled instance $x^{t}\in U$, $f_{v}^{U}({x_{v}^{t}})\cdot
c_{v}(x_{v}^{t})>0$.
\end{proof}

\noindent \textbf{Remark:} Theorem~\ref{theorem:graph2} provides the sufficient and necessary condition for co-training with non-perfect graphs to succeed. Note that in both Theorem~\ref{theorem:graph1} and
Theorem~\ref{theorem:graph2}, it is the necessary condition that $S_{1}^{k}\oplus S_{2}^{k}$ is not $\emptyset$ until $S_{v}^{k}=L\cup U$ $(v=1,2)$. In the following part we will further study what this necessary condition means and how to verify it before co-training.

First, we introduce the combinative graph $P_{c}$ in Equation~\ref{eq:eq6} which aggregates
graphs $P_{1}$ and $P_{2}$.
\begin{eqnarray}\label{eq:eq6}
  P_{c}^{ij}=\max[P_{1}^{ij}, P_{2}^{ij}]
\end{eqnarray}
Then we give Theorem~\ref{theorem:graph3} to indicate the necessary condition, i.e., each unlabeled instance can be connected with some labeled example in graph $P_c$.

\begin{theorem}\label{theorem:graph3}
$S_{1}^{k}\oplus S_{2}^{k}$ is not $\emptyset$ in Algorithm~\ref{alg:graph-based-co-training} until $S_{v}^{k} =L\cup U$
$(v=1,2)$ if and only if each unlabeled instance $x^{t_{0}}\in U$ can be connected with some labeled example $x^{t_{r}}\in L$ in graph $P_{c}$ through a path $R_{c}$ in the
form of $V_{R_{c}}=\{t_{0},t_{1}, \cdots, t_{r} \}$, where
$P_{c}^{t_{\rho}t_{\rho+1}}>0$ $(\rho=0,\cdots,r-1)$.
\end{theorem}

\begin{proof}
If we neglect the probability mass on the instances for which the non-zero positive contribution is equal to the non-zero negative contribution in this article, similarly to the proof of Lemma~\ref{lemma:graph1} we get that: unlabeled instance can be labeled by label propagation on graph $P$ if and only if it can be connected with some labeled example
in graph $P$ through a path.

If $S_{1}^{k}\oplus S_{2}^{k}$ is not $\emptyset$ until $S_{v}^{k} =L\cup U$, any
unlabeled instance $x^{t}\in U$ can be labeled by the process in Algorithm~\ref{alg:graph-based-co-training}. So $x^{t}$
must belong to one of $S_{1}^{0}$, $S_{2}^{0}$, $T_{1}^{k}$ or $T_{2}^{k}$ for some $k\geq 0$. Considering Equation~\ref{eq:eq6}, the above discussions and the fact that $S_{1}^{0}$, $S_{2}^{0}$, $T_{1}^{k}$ and $T_{2}^{k}$ have been obtained in previous iterations by label propagation and will be used as labeled examples in next iteration, we can get that $x^{t_{0}}$ can be connected with some labeled example $x^{t_{r}}\in L$ in graph $P_{c}$ through a path $R_{c}$.

If each unlabeled instance $x^{t_{0}}\in U$ can be connected with some labeled example $x^{t_{r}}\in L$ through a path $R_{c}$, with respect to Equation~\ref{eq:eq6}, we can get that either $P_{1}^{t_{\rho}t_{\rho+1}}$ or $P_{2}^{t_{\rho}t_{\rho+1}}$ is larger than $0$ for $\rho=0,\cdots,r-1$. Because $x^{t_{r}}$ is a labeled example, with the above discussions and the process in Algorithm~\ref{alg:graph-based-co-training} we know that $x^{t_{\rho}}$ ($\rho=0,\cdots,r-1$) can be labeled by label propagation on either $P_{1}$ or $P_{2}$. Therefore, finally $S_{v}^{k}=L \cup U$.
\end{proof}

\subsection{Co-Training with $\epsilon$-Good Graphs}

It is overly optimistic to expect to learn the target concept using
co-training with non-perfect graphs. While learning the approximately correct concept
using co-training with approximately perfect graphs is more reasonable. In perfect
graph, all edges between the instances are reliable; while in non-perfect graph, it is
hard to know which and how many edges are reliable. Restricting the reliability and
allowing an $\epsilon$-fraction exception is more feasible in real-world applications. In
this section, we focus on the approximately perfect graph and provide sufficient
condition for co-training the approximately correct concept.

Let $\pi^{1}_{v},\cdots, \pi^{\lambda_{v}}_{v}$ $(v=1,2)$ denote the $\lambda_{v}$
connected components in graph $P_{v}$, the definitions of \textit{purity} and
\textit{$\epsilon$-good graph} are given as follows.
\begin{definition}\label{def:definition1}
Let $pur(\pi^{\theta}_{v})$ denote the purity of the connected component
$\pi^{\theta}_{v}$ in graph $P_{v}$, then
\begin{eqnarray}\label{purity}
\nonumber
    pur(\pi^{\theta}_{v})=\max\Big[{|\{x_{v}:x_{v}\in \pi^{\theta}_{v}\wedge c_{v}(x_{v})=1
    \}|}/{|\pi^{\theta}_{v}|},\\
\nonumber
    {|\{x_{v}:x_{v}\in \pi^{\theta}_{v}\wedge c_{v}(x_{v})=-1 \}|}/{|\pi^{\theta}_{v}|}\Big]
\end{eqnarray}
If $pur(\pi^{\theta}_{v})\geq 1-\epsilon$ for all $1 \leq \theta \leq \lambda_{v}$, we
say that $P_{v}$ is an $\epsilon$-good graph.
\end{definition}
The purity of the connected component reflects the reliability of the graph. With the purity, we can define the
label of $\pi^{\theta}_{v}$ as $c_{v}(\pi^{\theta}_{v})$.
\begin{eqnarray}
\nonumber
  c_{v}(\pi^{\theta}_{v}) = \left\{ \begin{array}{ll} 1 & \textrm{if
$|\{x_{v}:x_{v}\in
\pi^{\theta}_{v}\wedge c_{v}(x_{v})=1\}|> |\{x_{v}:x_{v}\in \pi^{\theta}_{v}\wedge c_{v}(x_{v})=-1\}|$} \\
-1 & \textrm{otherwise}
\end{array} \right.
\end{eqnarray}
With $\epsilon$-good graph, predicting the labels of all $\pi^{\theta}_{v}$ correctly
is sufficient to get a classifier whose error rate is no larger than $\epsilon$. From
Definition~\ref{def:contribution} we know that the \textit{contribution} is related to the
number of labeled examples in the connected component. In a connected component, if the
labeled examples with label $y$ $(y\in\{-1,1\})$ is much more than that
with label $-y$, the unlabeled instances belonging to this connected component may be
labeled with $y$. Based on this, we assume graph $P_{v}$ satisfies the following
condition: in the connected component $\pi^{\theta_{k}}_{v}$ of graph $P_{v}$ where
$\pi^{\theta_{k}}_{v}\subseteq (U-S_{v}^{k})$ and $\pi^{\theta_{k}}_{v} \cap
S^{k}_{3-v}\neq \emptyset$, let $f^{k}_{v}$ correspond to
$S_{v}^{k}$, if $|\{x^{t}: x^{t}\in \pi^{\theta_{k}}_{v} \cap S^{k}_{3-v}\wedge
f^{k}_{3-v}({x^{t}})\cdot y>0\}|/|\pi^{\theta_{k}}_{v}|
> |\{x^{t}: x^{t}\in \pi^{\theta_{k}}_{v} \cap S^{k}_{3-v}\wedge
f^{k}_{3-v}({x^{t}})\cdot y<0\}|/|\pi^{\theta_{k}}_{v}|+\gamma$, the unlabeled
instances belonging to $\pi^{\theta_{k}}_{v}$ can be labeled with $y$ by label
propagation on graph $P_{v}$. Here $\gamma \in [0,1)$ can be thought of as a form of
\textit{margin} which controls the confidence in label propagation. With this
assumption, we get Theorem~\ref{theorem:graph4} which provides a margin-like sufficient
condition for co-training the approximately correct concept with $\epsilon$-good
graphs.

\begin{theorem}\label{theorem:graph4}
Suppose $P_{v}$ $(v=1,2)$ is $\epsilon$-good graph. $acc(f_{v}^{U})\geq 1-\epsilon$ if
both (1) and (2) hold in Algorithm~\ref{alg:graph-based-co-training}: (1) $S_{1}^{k}\oplus S_{2}^{k}$ is not $\emptyset$
until $S_{v}^{k}=L\cup U$; (2) In the connected component $\pi^{\theta_{k}}_{v}$, where
$\pi^{\theta_{k}}_{v}\subseteq (U-S_{v}^{k})$ and $\pi^{\theta_{k}}_{v} \cap
S^{k}_{3-v}\neq \emptyset$, $|\{x^{t}: x^{t}\in \pi^{\theta_{k}}_{v} \cap S^{k}_{3-v}
\wedge f^{k}_{3-v}({x^{t}})\cdot
c_{v}(\pi^{\theta_{k}}_{v})>0\}|/|\pi^{\theta_{k}}_{v}|
> |\{x^{t}: x^{t}\in \pi^{\theta_{k}}_{v} \cap S^{k}_{3-v} \wedge
f^{k}_{3-v}({x^{t}})\cdot
c_{v}(\pi^{\theta_{k}}_{v})<0\}|/|\pi^{\theta_{k}}_{v}|+\gamma$.
\end{theorem}

\subsection{Relationship to Previous Results}
As mentioned in Section~\ref{sec:Introduction}, there are several theoretical analyses indicating that co-training can succeed if some conditions about the two views hold, i.e., \textit{conditional independence}, \textit{weak dependence} and \textit{$\alpha$-expansion}.
Co-training is a representative paradigm of disagreement-based approaches, and we provide a \textit{large disagreement} analysis for disagreement-based approaches in Section~\ref{sec:general-analysis} which is also applicable to co-training. Now we will discuss the relationship between the graph-based analysis and previous results on co-training.

\subsubsection{Conditional Independence}
Blum and Mitchell \cite{Blum:Mitchell1998} proved that when the two sufficient views are conditionally
independent to each other, co-training can be successful. The
\textit{conditional independence} means that for the connected components
$\pi^{\theta_{i}}_{1}$ of $P_{1}$ and $\pi^{\theta_{j}}_{2}$ of $P_{2}$,
$P(\pi^{\theta_{i}}_{1}\cap \pi^{\theta_{j}}_{2})
=P(\pi^{\theta_{i}}_{1})P(\pi^{\theta_{j}}_{2})$. Since $S_{v}^{k}$ ($v=1,2$) is the union of some connected components of $P_{v}$, we have $P(S_{1}^{k}\cap S_{2}^{k})
=P(S_{1}^{k})P(S_{2}^{k})$. It means that $P(S_{1}^{k}\oplus
S_{2}^{k})=P(S_{1}^{k})(1-P(S_{2}^{k}))+P(S_{2}^{k})(1-P(S_{1}^{k}))$, which implies that the condition (1) in Theorem~\ref{theorem:graph4} holds. In addition, Equations~\ref{eq:independent1} and \ref{eq:independent2} can be obtained for $\epsilon$-good graphs.
\begin{eqnarray}
\label{eq:independent1}
    P\big(\pi^{\theta_{k}}_{v} \cap S^{k}_{3-v} \wedge
        f^{k}_{3-v}({x^{t}})\cdot c_{v}(\pi^{\theta_{k}}_{v})>0\big)&\!\!\geq\!\!&
        P(\pi^{\theta_{k}}_{v})P(S^{k}_{3-v})(1-\epsilon)\\
\label{eq:independent2}
    P\big(\pi^{\theta_{k}}_{v} \cap S^{k}_{3-v} \wedge
f^{k}_{3-v}({x^{t}})\cdot c_{v}(\pi^{\theta_{k}}_{v})<0\big)
    &\!\!<\!\!&P(\pi^{\theta_{k}}_{v})P(S^{k}_{3-v})\epsilon
\end{eqnarray}
Thus, we get that the condition (2) in Theorem~\ref{theorem:graph4} holds with $\gamma=P(S^{k}_{3-v})(1-2\epsilon)$.

\subsubsection{Weak Dependence}
Abney \cite{Abney2002} found that \textit{weak dependence} can lead to successful
co-training. The \textit{weak dependence} means that for the connected components
$\pi^{\theta_{i}}_{1}$ of $P_{1}$ and $\pi^{\theta_{j}}_{2}$ of $P_{2}$,
$P(\pi^{\theta_{i}}_{1}\cap \pi^{\theta_{j}}_{2}) \leq \tau
P(\pi^{\theta_{i}}_{1})P(\pi^{\theta_{j}}_{2})$ for some $\tau>0$. It implies that the
number of examples in $S_{1}^{k}\oplus S_{2}^{k}$ is not very small. So the condition (1) in Theorem~\ref{theorem:graph4} holds. For $\epsilon$-good graphs, without loss of
generality, assume that $P(\pi^{\theta_{k}}_{v} \cap S^{k}_{3-v})
=\tau_{1}P(\pi^{\theta_{k}}_{v})P(S^{k}_{3-v})$ and that
\begin{eqnarray}
\label{eq:weakindependent1}
    P\big(\pi^{\theta_{k}}_{v} \cap S^{k}_{3-v} \wedge
    f^{k}_{3-v}({x^{t}})\cdot c_{v}(\pi^{\theta_{k}}_{v})<0\big) \leq\tau_{2}P(\pi^{\theta_{k}}_{v})P(S^{k}_{3-v})\epsilon
\end{eqnarray}
for some $\tau_{1}>0$ and $\tau_{2}>0$, we can have
\begin{eqnarray}
\nonumber
  P\big(\pi^{\theta_{k}}_{v} \cap S^{k}_{3-v} \wedge
f^{k}_{3-v}({x^{t}})\cdot c_{v}(\pi^{\theta_{k}}_{v})>0\big)
   &\!\!\geq\!\!& P\big(\pi^{\theta_{k}}_{v} \cap S^{k}_{3-v}\big)-\tau_{2}
P(\pi^{\theta_{k}}_{v})P(S^{k}_{3-v})\epsilon\\
\label{eq:weakindependent2}
   &\!\!=\!\!&P(\pi^{\theta_{k}}_{v})P(S^{k}_{3-v})(\tau_{1}-\tau_{2}\epsilon).
\end{eqnarray}
Thus, we get that the condition (2) in Theorem~\ref{theorem:graph4} holds with $\gamma=P(S^{k}_{3-v})(\tau_{1}-2\tau_{2}\epsilon)$.

\subsubsection{$\alpha$-Expansion}
Balcan et al. \cite{Balcan:Blum:Yang2005} proposed $\alpha$-expansion and proved that it can guarantee the success of co-training. They assumed that the classifier in each view is never ``confident but wrong", which corresponds to the case with perfect graphs in Theorem~\ref{theorem:graph1}. The \textit{$\alpha$-expansion} means that $S_{1}^{k}$ and
$S_{2}^{k}$ satisfy the condition that $P(S_{1}^{k}\oplus S_{2}^{k})\geq \alpha \min[P(S_{1}^{k}\cap S_{2}^{k}), P(\overline{S_{1}^{k}}\cap \overline{S_{2}^{k}})]$.
When $\alpha$-expansion holds, it is easy to know that the condition in
Theorem~\ref{theorem:graph1} holds. Note that $S_{1}^{k}\oplus S_{2}^{k}\neq \emptyset$ is
weaker than \textit{$\alpha$-expansion}, since $P(S_{1}^{k}\oplus S_{2}^{k})$ does not
need to have a lower bound with respect to some positive $\alpha$.

\subsubsection{Large Disagreement}
Our result in Section~\ref{sec:upper-bound} shows that when the classifiers have large disagreement, the performance can be improved. Since the classifiers may have both error and uncertainty with non-perfect graphs, it is complicated to define the disagreement. Therefore, we only discuss co-training with perfect graphs here. For perfect graphs, the classifiers are ``confident of labeling'', so the error rate is $0$. It implies that the condition in Theorem~\ref{theorem:graph1} holds (see Section~\ref{sec:conclusions} for more discussions about the results in Section~\ref{sec:upper-bound} and Section~\ref{sec:sufficient-necessary}).

\subsubsection{Other Implication and Discussions}
From the above discussions it can be found if any previous condition holds, our condition in the graph-based analysis also holds; this means that our result is more general and tighter. Furthermore, this graph-based analysis also has other interesting implication. There were some works which combine the weight matrices or Laplacians for
each graph and then classify unlabeled instances according to the combination
\cite{Sindhwani:Niyogi:Belkin2005,Argyriou2006,ZhangPD06,DengyongZhou2007}, the underlying principle is not clear. To some extent, Theorem~\ref{theorem:graph3} can provide some theoretical supports to these methods, i.e., these methods are developed to satisfy the necessary condition for co-training with graphs to succeed as much as possible. Note that, the graph-based analysis on co-training in this section only cares the two graphs rather than where these graphs come from, and therefore it is also applicable to single-view disagreement-based approaches when there is only one view but two graphs can be obtained in different distance matrices.

\section{Analysis on Co-Training with Insufficient Views}\label{sec:insufficient-co-training}
All previous theoretical analyses on co-training are based on the assumption that each view can provide sufficient information to learn the target concept. However, in many real-world applications, due to feature corruption or various feature noise, neither view can provide sufficient information to learn the target concept. There exist some examples $\big(\langle x_{1}, x_{2}\rangle, y\big)$, on which the posterior probability $P(y=+1|x_{v})$ or $P(y=-1|x_{v})$ $(v=1,2)$ is not equal to $1$ due to the insufficient information provided by $x_{v}$ for predicting the label. In this section, we will present the theoretical analysis on co-training with insufficient views which is much more challenging but practical, especially when the two views provide diverse information.

\subsection{View Insufficiency}\label{sec:insufficiency}
Let $\mathcal {X}=\mathcal {X}_{1}\times \mathcal {X}_{2}$ denote the instance space, where $\mathcal {X}_{1}$ and $\mathcal {X}_{2}$
are the two views. $\mathcal {Y}=\{-1, +1\}$ are
the label space, $L=\{(\langle x^{1}_{1}, x^{1}_{2}\rangle, y^{1}), \cdots, (\langle x^{l}_{1}, x^{l}_{2}\rangle, y^{l}) \} \subset {\mathcal {X} \times \mathcal {Y}}$ are the labeled data, $U=\{\langle x^{l+1}_{1}, x^{l+1}_{2}\rangle, \cdots, \langle x^{l+|U|}_{1}, x^{l+|U|}_{2}\rangle\} \subset \mathcal{X}$ is the unlabeled data. The labeled data $L$ independently and identically come from some unknown distribution $\mathcal {D}$, whose marginal distribution on $\mathcal {X}$ is $\mathcal {D}_{\mathcal {X}}$, and the unlabeled data $U$ independently and identically come from $\mathcal {D}_{\mathcal {X}}$. $c=(c_{1}, c_{2})$ is the target concept, where $c_{1}$ and $c_{2}$ are the target concept in the two views, respectively, i.e., $c_{1}(x_{1})=c_{2}(x_{2})=y$ for any $(\langle x_{1},x_{2}\rangle, y)$. Since neither view can provide sufficient information to learn the target concept, we may never achieve the target concept with the insufficient views. For an example $\big(\langle x_{1}, x_{2}\rangle, y\big)$, let
$\varphi_{v}(x_{v})=P(y=+1|x_{v})$. If $\varphi_{v}(x_{v})$ is $0$ or $1$, it implies that the features of $x_{v}$ provide sufficient information to correctly determine its label $y$; while if $\varphi_{v}(x_{v})=\frac{1}{2}$, it implies that the
features of $x_{v}$ provide no helpful information to correctly predict its label $y$. It is easy to understand that $\big|2\varphi_{v}(x_{v})-1\big|$ is a measurement of the information provided by $x_{v}$ for predicting its label $y$. Now we give the definition of view insufficiency.

\begin{definition}[View Insufficiency]\label{def:view-insufficiency}
Let $\mathcal {D}$ denote the unknown distribution
over $\mathcal {X}\times \mathcal {Y}$. For $(x,y)\in \mathcal {X}\times
\mathcal {Y}$, $\varphi(x)=P(y=+1|x)$. The insufficiency $\Upsilon(\mathcal
{X},\mathcal {Y},\mathcal {D})$ of view $\mathcal
{X}$ for the learning task with respect to the distribution $\mathcal {D}$ is defined as
\begin{eqnarray}\label{eq:view-insufficiency}
\nonumber
  \Upsilon(\mathcal {X}, \mathcal {Y}, \mathcal {D})=1-\int_{\mathcal {D}}|2\varphi(x)-1|P(x)dx.
\end{eqnarray}
\end{definition}
$\Upsilon(\mathcal {X}, \mathcal {Y}, \mathcal {D})\in[0,1]$ measures the insufficiency
of view $\mathcal {X}$ for correctly learning $\mathcal {Y}$ over the distribution
$\mathcal {D}$. When $|2\varphi(x)-1|=1$ for all examples, the view insufficiency $\Upsilon(\mathcal {X},\mathcal {Y},\mathcal {D})=0$, i.e., view $\mathcal {X}$ provides
sufficient information to correctly classify all examples; while
$\varphi(x)=\frac{1}{2}$ for all examples, the view insufficiency $\Upsilon(\mathcal
{X},\mathcal {Y},\mathcal {D})=1$, i.e., view $\mathcal {X}$ provides no information to
correctly classify any example. With Definition~\ref{def:view-insufficiency}, we let $\Upsilon_{v}=\Upsilon(\mathcal {X}_{v},\mathcal {Y},\mathcal {D})$ denote the insufficiency of view $\mathcal {X}_{v}$.

Let $\mathcal {F}_{v}$: $\mathcal{X}_{v} \to [-1, +1]$ denote the hypothesis space for learning the task with view $\mathcal {X}_{v}$ $(v=1,2)$ and $d_{v}$ denote the finite
VC-dimension of $\mathcal {F}_{v}$. The classification rule induced by a hypothesis $f_{v}\in \mathcal {F}_{v}$ on an instance $x=\langle x_{1},x_{2} \rangle$
is $sign(f_{v}(x_{v}))$. The error rate of a hypothesis $f_{v}$ with the
distribution $\mathcal {D}$ is $\texttt{err}(f_{v})=P_{(\langle x_{1},x_{2}\rangle,y)\in \mathcal {D}}(sign(f_{v}(x_{v}))\neq y)$ and let $\texttt{err}(F_{v})=\max_{f_{v}\in F_{v}}\texttt{err}(f_{v})$ for $F_{v}\subseteq \mathcal {F}_{v}$. Let $f^{*}_{v}(x_{v})=2\varphi_{v}(x_{v})-1$, $sign(f^{*}_{v}(x_{v}))=+1$ if $\varphi_{v}(x_{v})>\frac{1}{2}$ and $sign(f^{*}_{v}(x_{v}))=-1$ otherwise. Suppose $f^{*}_{v}$ belongs to $\mathcal {F}_{v}$, and it is well-known that \cite{Devroye1996} $f^{*}_{1}$ and $f^{*}_{2}$ are the optimal Bayes classifiers in the two views, respectively.
Generally, the two views may provide different information, i.e., there exist some instances $x=\langle x_{1}, x_{2}\rangle$ on which $P(y=+1|x_{1})$ is very different from $P(y=+1|x_{2})$. Thus, $f^{*}_{1}$ is not perfectly compatible with $f^{*}_{2}$ and
$d(f^{*}_{1}, f^{*}_{2})$ denotes the difference between $f^{*}_{1}$ and
$f^{*}_{2}$.
\begin{eqnarray}
\nonumber
  d(f^{*}_{1},f^{*}_{2})=P_{\langle x_{1}, x_{2}\rangle\in \mathcal {X} }\big(sign(f^{*}_{1}(x_{1}))\neq sign(f^{*}_{2}(x_{2}))\big)
\end{eqnarray}
Let $\eta_{v}=\texttt{err}(f^{*}_{v})$ denote the error rate of the optimal classifier $f^{*}_{v}$, we have the following Proposition~\ref{pro:insufficiency-error}.

\begin{proposition}\label{pro:insufficiency-error}
$\Upsilon_{v}\!=\!2\eta_{v}$. $(v=1,2)$
\end{proposition}

\begin{proof}
Given an example $\left(\langle x_{1},x_{2} \rangle, y\right)$,
\begin{eqnarray}
\nonumber
  &&P\big(sign\big(f^{*}_{v}(x_{v})\big)\neq y|x_{v}\big)\\
\nonumber
  &&=1-P\big(sign\big(f^{*}_{v}(x_{v})\big)=1, y=1|x_{v}\big)-P\big(sign\big(f^{*}_{v}(x_{v})\big)=-1, y=-1|x_{v}\big)\\
\nonumber
  &&=1-\mathbb{I}\big\{sign\big(f^{*}_{v}(x_{v})\big)=1\big\}P\big(y=1|x_{v}\big)
  -\mathbb{I}\big\{sign\big(f^{*}_{v}(x_{v})\big)=-1\big\}P\big(y=-1|x_{v}\big)\\
\nonumber
  &&=1-\mathbb{I}\big\{\varphi_{v}(x_{v})>1/2\big\}\varphi_{v}(x_{v})
  -\mathbb{I}\big\{\varphi_{v}(x_{v})\leq1/2\big\}\big(1-\varphi_{v}(x_{v})\big)
\end{eqnarray}
So we get
\begin{eqnarray}
\nonumber
  \eta_{v}&\!=\!&\mathbb{E}\big(1-\mathbb{I}\big\{\varphi_{v}(x_{v})>1/2\big\}\varphi_{v}(x_{v})
  -\mathbb{I}\big\{\varphi_{v}(x_{v})\leq1/2\big\}\big(1-\varphi_{v}(x_{v})\big)\big)\\
\nonumber
  &\!=\!&\mathbb{E}\big(1/2-\big|\varphi_{v}(x_{v})-1/2\big|\big)=
  \frac{1}{2}\Upsilon_{v}.
\end{eqnarray}
\end{proof}

\noindent \textbf{Remark:} Proposition~\ref{pro:insufficiency-error} states that when the view is insufficient, the optimal classifier trained on this view will mistakenly classify some instances. The larger the insufficiency is, the worse the performance of the optimal classifier will be.

\subsection{Learning Approximation of Optimal Classifier with Complementary Views}

\begin{algorithm}[!h]
\begin{algorithmic}
   \STATE {\bfseries Input:} Labeled data $L$, unlabeled data $U$, and two hypothesis spaces $\mathcal {F}_{1}$ and $\mathcal {F}_{2}$.
   \STATE {\bfseries Output:} $\mathcal {F}^{C}_{1}$ and $\mathcal {F}^{C}_{2}$.
   \STATE {\bfseries Initialize:} Set $\varrho_{0}=L$;
   \FOR{$i=0, 1, 2, \ldots$}
        \STATE Get $\mathcal {F}^{i}_{v} \subseteq {\mathcal{F}}_{v}$ ($v=1,2$) by minimizing the empirical risk on $\varrho_{i}$ with respect to view $\mathcal {X}_{v}$ and set $T_{i}=\emptyset$;
        \FOR{$x=\langle x_{1},x_{2} \rangle\in U$}
            \FOR{$f_{1}\in \mathcal {F}^{i}_{1}$, $f_{2}\in \mathcal {F}^{i}_{2}$}
                \IF{$|f_{1}(x_{1})|\geq \gamma_{1}$}
                    \STATE $T_{i}=T_{i}\cup (x,sign(f_{1}(x_{1})))$ and delete $x$ from $U$;\\
                    \textbf{break};
                \ENDIF
                \IF{$|f_{2}(x_{2})|\geq \gamma_{2}$}
                    \STATE $T_{i}=T_{i}\cup (x,sign(f_{2}(x_{2})))$ and delete $x$ from $U$;\\
                    \textbf{break};
                \ENDIF
            \ENDFOR
        \ENDFOR
   \IF{$T_{i}=\emptyset$}
        \STATE \textbf{return} $\mathcal {F}^{C}_{1}=\mathcal {F}^{i}_{1}$ and $\mathcal {F}^{C}_{2}=\mathcal {F}^{i}_{2}$;
   \ENDIF
   \STATE $\varrho_{i+1}=\varrho_{i}\cup T_{i}$;
   \ENDFOR
\end{algorithmic}
\caption{Margin-based co-training with insufficient views}\label{alg:insufficient-co-training1}
\end{algorithm}

Usually, co-training allows one classifier to label unlabeled instances for the other. For insufficient views, in each view there are some instances which can not provide sufficient information for predicting the label. So
we should check how much information each instance can provide. We use the confidence of the prediction on an instance to measure the information.
When the confidence is no less than some preset threshold, we use the predicted label as its pseudo-label and add it into the training set.
Since each view is insufficient and only provides partial information, we use the newly labeled unlabeled instances which are labeled by both classifiers as the retraining data, and the process is described in Algorithm \ref{alg:insufficient-co-training1}. However, as mentioned in Section~\ref{sec:general-analysis} if there is no prior knowledge about the relationship between hypothesis space and unlabeled data, it is hard to guarantee that selecting confident instances to label is helpful. In margin-based algorithms, margin can be used to measure the confidence. Intuitively, it is likely that similar hypotheses tend to have similar margin outputs, i.e., two hypotheses with small error difference should have small margin difference. With this intuition, we give the following Definition \ref{def:marginLipschitz}.

\begin{definition}[Margin Lipschitz]\label{def:marginLipschitz}
Let $\mathcal {F}_{v}$ $(v=1,2)$ denote the hypothesis
space, for $x=\langle x_{1},x_{2} \rangle$ and $f_{v}\in \mathcal{F}_{v}$, there exists
some constant $C^{\mathfrak{L}}_{v}$ to satisfy $|f_{v}(x_{v})-f_{v}^{*}(x_{v})|\leq C^{\mathfrak{L}}_{v}(\texttt{err}(f_{v})-\texttt{err}(f_{v}^{*}))$.
\end{definition}

Definition~\ref{def:marginLipschitz} states that the label predicted by weak classifiers with large margin is likely to be the same as the label predicted by the optimal classifier. Thus, the confident instances would help find the optimal classifier. Here we give two examples that satisfy the Margin Lipschitz definition.

\begin{example}\label{example1}
Image that the instances in ${\mathcal {X}}_{v}$ are distributed uniformly over the unit ball in $\mathbb{R}^{d_{v}}$ and that the underground labels are determined by a linear hyperplane $w^{*}_{v}$ going through the origin, i.e., $y=sign(w^{*}_{v}\cdot x_{v})$ for any $(x_{v},y)$. It can be verified that Definition~\ref{def:marginLipschitz} holds with the constant $C^{\mathfrak{L}}_{v}\geq\pi$.
\end{example}

\begin{example}\label{example2}
Image that the instances in ${\mathcal {X}}_{v}$ are distributed uniformly over the unit ball in $\mathbb{R}^{d_{v}}$ and that the underground labels are determined by a random variable $\beta$ ($\beta=+1$ with probability $1-\eta$ and $\beta=-1$ with probability $\eta$) and a linear hyperplane $w^{*}_{v}$ going through the origin, i.e., $y=\beta\cdot sign(w^{*}_{v}\cdot x_{v})$ for any $(x_{v},y)$. Here the variable $\beta$ is exploited to simulate the instances with insufficient features due to feature noise. It can be verified that Definition~\ref{def:marginLipschitz} also holds with the constant $C^{\mathfrak{L}}_{v}\geq\pi$.
\end{example}

To quantify the amount of the confident instances, we give the following Definition \ref{def:quantity}.

\begin{definition}\label{def:quantity}
Let $x=\langle x_{1}, x_{2}\rangle$,
\begin{eqnarray}
\nonumber
   \mu_{1}(\gamma_{1}, \mathcal {F}_{1})&\!\!\!\!=\!\!\!\!&P\big\{x \in U: \exists f_{1}\in \mathcal {F}_{1} ~~\mathrm{s.t.}~~|f_{1}(x_{1})|\geq\gamma_{1}\big\},\\
\nonumber
   \mu_{2}(\gamma_{2}, \mathcal {F}_{2})&\!\!\!\!=\!\!\!\!&P\big\{x \in U: \exists f_{2}\in \mathcal {F}_{2}~~\mathrm{s.t.}~~|f_{2}(x_{2})|\geq\gamma_{2} \big\},\\
\nonumber
   \mu(\gamma_{1}, \gamma_{2}, \mathcal {F}_{1}, \mathcal {F}_{2})&\!\!\!\!=\!\!\!\!&P\big\{x \in U: \exists f_{1}\in \mathcal {F}_{1}, f_{2}\in \mathcal {F}_{2} ~~\mathrm{s.t.}~~|f_{1}(x_{1})|\geq\gamma_{1}~\mathrm{or}~ |f_{2}(x_{2})|\geq\gamma_{2} \big\},
\end{eqnarray}
and let $\nu(\gamma_{1}, \gamma_{2}, \mathcal {F}_{1}, \mathcal {F}_{2})$ denote the probability mass on the instances which are labeled with large margin just by one view, then
\begin{eqnarray}
\nonumber
    \mu(\gamma_{1}, \gamma_{2}, \mathcal {F}_{1}, \mathcal {F}_{2})=
    \frac{\nu(\gamma_{1}, \gamma_{2}, \mathcal {F}_{1}, \mathcal {F}_{2})+\mu_{1}(\gamma_{1}, \mathcal {F}_{1})+\mu_{2}(\gamma_{2}, \mathcal {F}_{2})}{2}.
\end{eqnarray}
\end{definition}
$\nu(\gamma_{1}, \gamma_{2}, \mathcal {F}_{1}, \mathcal {F}_{2})$ measures the disagreement between two views with respect to margins $\gamma_{1}$ and $\gamma_{2}$, we call it \textit{margin-based disagreement}. When $\nu(\gamma_{1}, \gamma_{2}, \mathcal {F}_{1}, \mathcal {F}_{2})$ is large, the two views could help each other strongly by providing diverse confident information; while when $\nu(\gamma_{1}, \gamma_{2}, \mathcal {F}_{1}, \mathcal {F}_{2})$ is small, the two views only help each other little since they provide almost the same information. $\mu(\gamma_{1}, \gamma_{2}, \mathcal {F}_{1}, \mathcal {F}_{2})$ denotes the probability mass on the instances which are labeled with large margin by one of the two views. If the two views have large margin-based disagreement $\nu(\gamma_{1}, \gamma_{2}, \mathcal {F}_{1}, \mathcal {F}_{2})$, $\mu(\gamma_{1}, \gamma_{2}, \mathcal {F}_{1}, \mathcal {F}_{2})$ is large. For the extreme case that $\mu(\gamma_{1}, \gamma_{2}, \mathcal {F}_{1}, \mathcal {F}_{2})=1$, i.e., each unlabeled instance can be labeled with large margin by one of the two views, we say that the two views are \textit{complementary views}, since they provide complementary information. For this extreme case, we have the following Theorem \ref{theorem:insufficient-theorem1}.

\begin{theorem}\label{theorem:insufficient-theorem1}
Suppose the hypothesis space $\mathcal {F}_{v}$ $(v=1,2)$ satisfies Definition~\ref{def:marginLipschitz}, let
$\mathcal {F}^{L}_{v}\subseteq \mathcal {F}_{v}$ denote the hypotheses minimizing the
empirical risk on the initial labeled data $L$, $R_{v}=\max_{f_{v}\in\mathcal
{F}^{L}_{v}}\texttt{err}(f_{v})$ and $\gamma_{v}=C^{\mathfrak{L}}_{v}(R_{v}-\eta_{v})$. For $\epsilon \in (0,\frac{1}{2})$ and $\delta \in (0,1)$, if
$|U|=O(\frac{d_{v}\ln\frac{1}{\delta}}{\epsilon^{2}})$ and $\mu(\gamma_{1},
\gamma_{2}, \mathcal {F}^{L}_{1}, \mathcal {F}^{L}_{2})=1$, with probability $1-\delta$
the outputs $\mathcal {F}^{C}_{1}$ and $\mathcal {F}^{C}_{2}$ in Algorithm \ref{alg:insufficient-co-training1} satisfy
$\texttt{err}(\mathcal {F}^{C}_{v})=\max_{f_{v}\in\mathcal {F}^{C}_{v}}\texttt{err}(f_{v})\leq \eta_{1}/2+\eta_{2}/2+d(f_{1}^{*},f_{2}^{*})/2+\epsilon$.
\end{theorem}

\begin{proof}
Since $\mu(\gamma_{1}, \gamma_{2}, \mathcal {F}^{L}_{1}, \mathcal {F}^{L}_{2})=1$,
after $1$ round all unlabeled instances in $U$ are assigned with pseudo-labels and
added into the data set $\varrho_{1}$. Then classifier set $\mathcal {F}^{1}_{v}$ is
got by minimizing the empirical risk on $\varrho_{1}$ with view $\mathcal {X}_{v}$. For
$x=\langle x_{1}, x_{2}\rangle$, $\widehat{y}$ denotes its pseudo-label. If
$|f_{v}(x_{v})|\geq\gamma_{v}=C^{\mathfrak{L}}_{v}(R_{v}-\eta_{v})$, with Definition~\ref{def:marginLipschitz}
we know that $f_{v}$ and $f_{v}^{*}$ make the same prediction on $x_{v}$. So for any
example $\big(\langle x_{1}, x_{2}\rangle, \widehat{y}\big)\in \varrho_{1}$, either
$\widehat{y}=sign\big(f_{1}^{*}(x_{1})\big)$ or $\widehat{y}=sign\big(f_{2}^{*}(x_{2})\big)$
holds. Here we consider the worst case that
\begin{eqnarray}
\nonumber
  \widehat{y} = \left\{ \begin{array}{ll} y & \textrm{if
  $sign\big(f_{1}^{*}(x_{1})\big)=sign\big(f_{2}^{*}(x_{2})\big)=y$} \\
  -y & \textrm{otherwise}
\end{array}. \right.
\end{eqnarray}
Let $f^{com}_{v}$ denote the hypothesis that $sign\big(f^{com}_{v}(x_{v})\big)=y$ if
$sign\big(f_{1}^{*}(x_{1})\big)=sign\big(f_{2}^{*}(x_{2})\big)=y$, and
$sign\big(f^{com}_{v}(x_{v})\big)=-y$ otherwise. It is easy to find that $f^{com}_{v}$
is consistent with the examples in $\varrho_{1}$ for the worst case and
$\texttt{err}(f^{com}_{v})=\eta_{1}/2+\eta_{2}/2+d(f_{1}^{*},f_{2}^{*})/2$. $\texttt{err}(f^{com}_{v})$ is larger
than $\texttt{err}(f_{v}^{*})$, so learning a classifier with error rate no larger than
$\texttt{err}(f^{com}_{v})+\epsilon$ is no harder than learning a classifier with error rate no
larger than $\texttt{err}(f_{v}^{*})+\epsilon$. Now we regard $f^{com}_{v}$ as the optimal classifier
in $\mathcal {F}_{v}$ and neglect the probability mass on the hypothesis whose error
rate is less than $\texttt{err}(f^{com}_{v})$. Since the classifiers in $\mathcal {F}^{C}_{v}$
minimize the empirical risk on $\varrho_{1}$ which is an i.i.d sample with
size of $|L|+|U|$ and $|U|=O\big(\frac{d_{v}\ln\frac{1}{\delta}}{\epsilon^{2}}\big)$,
we get $\max_{f_{v}\in\mathcal {F}^{C}_{v}}\texttt{err}(f_{v})\leq \texttt{err}(f^{com}_{v})+\epsilon$ with
probability $1-\delta$.
\end{proof}

\noindent \textbf{Remark:} Theorem~\ref{theorem:insufficient-theorem1} states that if the two views are complementary views, i.e., every unlabeled instance in $U$ can be labeled with large margin by one of the two views, co-training could output the nearly good hypothesis set $\mathcal {F}^{C}_{1}$ and $\mathcal {F}^{C}_{2}$. Sometimes the pseudo-label which is the same as the prediction of the optimal classifier in view $\mathcal {X}_{v}$ is not helpful in achieving the optimal
classifier in view $\mathcal {X}_{3-v}$, since there exists the difference
$d(f_{1}^{*},f_{2}^{*})$ between the two optimal classifiers in the two views. Thus, the hypothesis in $\mathcal {F}^{C}_{v}$ is not very close to the optimal classifier $f_{v}^{*}$.

To achieve the good approximation of the optimal classifier, some prior knowledge about the optimal classifier needs to be known, which is shown as follows.

\begin{definition}[Information Assumption]\label{def:Information}
For $\big(\langle x_{1},x_{2} \rangle,
y\big)\in \mathcal {X}\times\mathcal {Y}$, if view $\mathcal {X}_{v}$ provides much
information about it, i.e., $|P(y=+1|x_{v})-\frac{1}{2}|\geq \gamma'_{v}/2$,
then the optimal classifier $f_{v}^{*}$ in view $\mathcal {X}_{v}$ classifies it correctly,
i.e., $sign(f_{v}^{*}(x_{v}))=y$.
\end{definition}

Definition~\ref{def:Information} states that for an example if one view can provide much information about
it, it will be correctly classified by the optimal classifier in this view. Thus, we give the following Theorem \ref{theorem:insufficient-theorem2}.

\begin{theorem}\label{theorem:insufficient-theorem2}
Suppose the hypothesis space $\mathcal {F}_{v}$ $(v=1,2)$ satisfies Definition~\ref{def:marginLipschitz}, let
$\mathcal {F}^{L}_{v}\subseteq \mathcal {F}_{v}$ denote the hypotheses minimizing the empirical risk on the initial labeled data $L$, $R_{v}=\max_{f_{v}\in\mathcal
{F}^{L}_{v}}\texttt{err}(f_{v})$ and
$\gamma_{v}=C^{\mathfrak{L}}_{v}(R_{v}-\eta_{v})+\gamma'_{v}$. For $\epsilon \in (0,\frac{1}{2})$ and $\delta \in (0,1)$, if Definition~\ref{def:Information}
holds, $|U|=O(\frac{d_{v}\ln\frac{1}{\delta}}{\epsilon^{2}})$ and
$\mu(\gamma_{1}, \gamma_{2}, \mathcal {F}^{L}_{1}, \mathcal {F}^{L}_{2})=1$, with probability $1-\delta$ the outputs $\mathcal {F}^{C}_{1}$ and $\mathcal {F}^{C}_{2}$ in Algorithm~\ref{alg:insufficient-co-training1} satisfy $\texttt{err}(\mathcal {F}^{C}_{v})=\max_{f_{v}\in\mathcal {F}^{C}_{v}}\texttt{err}(f_{v})\leq \eta_{v}+\epsilon$.
\end{theorem}

\begin{proof}
For $x=\langle x_{1}, x_{2}\rangle$, $\widehat{y}$ denotes its pseudo-label. If $|f_{v}(x_{v})|\geq\gamma_{v}=C^{\mathfrak{L}}_{v}(R_{v}-\eta_{v})+\gamma'_{v}$,
with Definition~\ref{def:marginLipschitz} we know that $f_{v}$ and $f_{v}^{*}$ make the same prediction on $x_{v}$
and $|f_{v}^{*}(x_{v})|\geq\gamma'_{v}$. So we get
$\big|P(y=+1|x_{v})-\frac{1}{2}\big|\geq \gamma'_{v}/2$. Then with Definition~\ref{def:Information} we know that $\widehat{y}=sign\big(f_{v}(x_{v})\big)=sign\big(f_{v}^{*}(x_{v})\big)=y$. So we get that the pseudo-label of any example in $\varrho_{1}$ is the same as its underground label. Since $\mu(\gamma_{1}, \gamma_{2}, \mathcal {F}^{L}_{1}, \mathcal
{F}^{L}_{2})=1$, we know that all unlabeled instances in $U$ are assigned with underground labels and added into $\varrho_{1}$. So $\varrho_{1}$ is an i.i.d sample with size of $|L|+|U|$. Considering that
$|U|=O\big(\frac{d_{v}\ln\frac{1}{\delta}}{\epsilon^{2}}\big)$, we get
$\max_{f_{v}\in\mathcal {F}^{C}_{v}}\texttt{err}(f_{v})\leq \eta_{v}+\epsilon$ with probability
$1-\delta$.
\end{proof}

\noindent \textbf{Remark:} Theorem~\ref{theorem:insufficient-theorem2} states that if the two views are complementary views with respect to larger margins $\gamma_{1}$ and $\gamma_{2}$, co-training could output the $\epsilon$-approximation of the optimal classifier.

\subsection{Learning Approximation of Optimal Classifier with Non-Complementary Views}

However, in real-world applications not all insufficient views are complementary, i.e., $\mu(\gamma_{1}, \gamma_{2}, \mathcal {F}^{L}_{1},
\mathcal {F}^{L}_{2})$ is smaller than $1$. With Definition~\ref{def:marginLipschitz} we know that the threshold $\gamma_{v}$ ($v=1,2$) which guarantees the quality of the confident instances is related to the error rates of weak hypotheses. An intuitive way to get more confident instances to augment the training data is updating the weak hypotheses with newly labeled confident instances and adaptively decreasing the margin threshold, which is shown in Algorithm~\ref{alg:insufficient-co-training2}. When $\mu(\gamma_{1}, \gamma_{2}, \mathcal {F}^{L}_{1}, \mathcal {F}^{L}_{2})$ is smaller than $1$, it will make co-training suffer from the sampling bias, since the training set in each view might not be an i.i.d sample from the marginal distribution $\mathcal {D}_{\mathcal {X}}$. Now we give the following definition to approximately bound the difference between two training samples.

\begin{algorithm}[!h]
\begin{algorithmic}
   \STATE {\bfseries Input:} Labeled data $L$, unlabeled data $U$, two hypothesis spaces $\mathcal {F}_{1}$ and $\mathcal {F}_{2}$, $m_{0}=|L|$, $n=|L|+|U|$, $\gamma^{0}_{v}=C^{\mathfrak{L}}_{v}(R_{v}-\eta_{v})+\gamma'_{v}$, and $\varrho_{0}=L$.
   \STATE {\bfseries Output:} $\mathcal {F}^{C}_{1}$ and $\mathcal {F}^{C}_{2}$.
   \FOR{$i=0, 1, 2, \ldots$}
        \STATE Get $\mathcal {F}^{i}_{v} \subseteq {\mathcal{F}}_{v}$ by minimizing the empirical risk on $\varrho_{i}$ with respect to view $\mathcal {X}_{v}$ and set $T_{i}=\emptyset$;
        \FOR{$x=\langle x_{1},x_{2} \rangle\in U$}
            \FOR{$f_{1}\in \mathcal {F}^{i}_{1}$, $f_{2}\in \mathcal {F}^{i}_{2}$}
                \IF{$|f_{1}(x_{1})|\geq \gamma^{i}_{1}$}
                    \STATE $T_{i}=T_{i}\cup (x, sign(f_{1}(x_{1})))$ and delete $x$ from $U$;\\
                    \textbf{break};
                \ENDIF
                \IF{$|f_{2}(x_{2})|\geq \gamma^{i}_{2}$}
                    \STATE $T_{i}=T_{i}\cup (x, sign(f_{2}(x_{2})))$ and delete $x$ from $U$;\\
                    \textbf{break};
                \ENDIF
            \ENDFOR
        \ENDFOR
    \IF{$i=0$ and $|T_{0}|>\sqrt[3]{n^{2}m_{0}}-m_{0}$}
        \STATE $\gamma^{1}_{v}=\gamma^{0}_{v}-C^{\mathfrak{L}}_{v}(R_{v}-\eta_{v})
    (1-\frac{n\sqrt{m_{0}}}{(m_{0}+|T_{0}|)^{3/2}})$, $\varrho_{1}=\varrho_{0}\cup T_{0}$, $m_{1}=m_{0}+|T_{0}|$;
    \ENDIF
    \IF{$|T_{0}|\leq\sqrt[3]{n^{2}m_{0}}-m_{0}$ or $T_{i}=\emptyset$}
        \STATE \textbf{return} $\mathcal {F}^{C}_{1}=\mathcal {F}^{i}_{1}$ and $\mathcal {F}^{C}_{2}=\mathcal {F}^{i}_{2}$;
    \ENDIF
    \IF{$i\geq1$}
        \STATE $\gamma^{i+1}_{v}=\gamma^{0}_{v}-C^{\mathfrak{L}}_{v}(R_{v}-\eta_{v})
    (1-\frac{n\sqrt{m_{0}}}{(m_{i}+|T_{i}|)^{3/2}})$, $\varrho_{i+1}=\varrho_{i}\cup T_{i}$, $m_{i+1}=m_{i}+|T_{i}|$;
    \ENDIF
   \ENDFOR
\end{algorithmic}
\caption{Adaptive margin-based co-training with insufficient views}\label{alg:insufficient-co-training2}
\end{algorithm}

\begin{definition}[Approximate KL Divergence]\label{def:KLdivergence}
Let $\Omega$ be a large example set i.i.d sampled
from the unknown distribution $\mathcal {D}$ and $\Lambda\subseteq \Omega$ be a set of examples, define the following $D_{AKL}(\Lambda\parallel \Omega)$ as an approximate KL divergence from the distribution generating $\Lambda$ to the distribution $\mathcal {D}$.
\begin{eqnarray}
\nonumber
  D_{AKL}(\Lambda\parallel \Omega)&\!=\!&\sum_{x^{j}\in \Omega}
  P\big(\mathbb{I}\{x^{j}\in\Lambda\}\big)\ln\frac{P(\mathbb{I}\{x^{j}\in\Lambda\})}
  {P(\mathbb{I}\{x^{j}\in \Omega\})}\\
\nonumber
  &\!=\!&\sum_{x^{j}\in\Lambda}\frac{1}{|\Lambda|}\ln\frac{1/|\Lambda|}{1/|\Omega|}+0
  =\ln\frac{|\Omega|}{|\Lambda|}.
\end{eqnarray}
\end{definition}

Let us interpret Definition~\ref{def:KLdivergence} intuitively. $\Omega$ is a large example set i.i.d sampled from the unknown distribution $\mathcal {D}$, so we use the uniform distribution
over $\Omega$ as an approximation of $\mathcal {D}$. In this way we use the uniform distribution over $\Lambda$ as an approximation of the distribution generating $\Lambda$ and define $D_{AKL}(\Lambda\parallel \Omega)$ as an approximate KL divergence from the distribution generating $\Lambda$ to the distribution $\mathcal {D}$. We give the following assumption to bound the influence of sampling bias.

\begin{definition}[Sampling Bias Assumption]\label{def:samplebias}
Let $\Omega$ be a large example set i.i.d sampled
from the unknown distribution $\mathcal {D}$ and $\Lambda\subseteq \Omega$ be a set of examples. Let $f_{\Lambda}$ denote the hypothesis minimizing the empirical risk on
$\Lambda$, $R^{*}$ be the error rate of the optimal classifier and $R'$ be the upper bound on the error rate of the hypothesis minimizing the empirical risk on an i.i.d. sample with size of $|\Lambda|$ from the distribution $\mathcal {D}$, then $\texttt{err}(f_{\Lambda})-R^{*}\leq (R'-R^{*})\cdot\exp\big(D_{AKL}(\Lambda\parallel \Omega)\big)$.
\end{definition}

Definition~\ref{def:samplebias} states that the error difference between the classifier trained with possibly biased sample and the optimal classifier can be bounded by that between the classifier trained with unbiased sample and the optimal classifier times
an exponential function of the approximate KL divergence.
Here we give an example to show when the sampling bias assumption holds in co-training.

\begin{example}\label{example3}
Suppose that the two views are conditionally independent to each other, then Definition~\ref{def:samplebias} holds. In Algorithm~\ref{alg:insufficient-co-training2}, the training set $\varrho_{i}$ is a subset of $L\cup U$ and consists of the instances whose margins are no less than the threshold. That the two views are conditionally independent means the hypothesis in the first view is independent of the hypothesis in the second view to make predictions, i.e., $f_{1}(x_{1})$ is independent of $f_{2}(x_{2})$ for any $x=\langle x_{1},x_{2} \rangle$. So it can be regarded that the instances in $\varrho_{i}$ are randomly drawn from $L\cup U$, i.e., $\varrho_{i}$ is an i.i.d sample. Thus, $\texttt{err}(f_{\varrho_{i}})-R^{*}\leq R'-R^{*} \leq (R'-R^{*})\cdot\exp(n/|\varrho_{i}|)$. This example can be relaxed to the case that the two views are weakly dependent if the prior knowledge about the margin outputs over the two views is known.
\end{example}

Let $\Omega$ be an i.i.d sample size of $m$, it is well-known that \cite{Anthony1999} there exists an universal constant $C$ such that for $\delta \in (0,1)$ we have
$\texttt{err}(f_{v})-\texttt{err}(f_{v}^{*}) \leq \sqrt{\frac{C}{m}(d_{v}+\ln(\frac{1}{\delta}))}$
with probability $1-\delta$ for any $f_{v}$ minimizing the empirical risk on $\Omega$. Generally, there may exist more than one hypothesis which have the same empirical risk. Let ${H}^{\Omega}_{v}$ denote the set of
hypotheses which have the same minimum empirical risk on $\Omega$, it is reasonable to assume that
$\max_{f_{v}\in\mathcal{F}^{\Omega}_{v}}\texttt{err}(f_{v})-\texttt{err}(f^{*}_{v})
=\sqrt{\frac{C}{m}\big(d_{v}+\ln(\frac{1}{\delta})\big)}$,
which means the PAC-bound is tight and the maximum error rate of the hypotheses which minimize the empirical risk on $\Omega$ is proportional to $\frac{1}{\sqrt{m}}$. We are now ready to give the theorem on co-training with insufficient views.

\begin{theorem}\label{theorem:insufficient-theorem3}
Suppose the hypothesis space $\mathcal {F}_{v}$ $(v=1,2)$ satisfies Definition~\ref{def:marginLipschitz}, let
$\varrho_{i}$ denote the training set in the $i$-th round of Algorithm \ref{alg:insufficient-co-training2}, $\mathcal {F}^{i}_{v} \subseteq \mathcal {F}_{v}$ denote the set of hypotheses minimizing the empirical risk on $\varrho_{i}$, $n=|L|+|U|$ and
$R_{v}=\max_{f^{0}_{v}\in\mathcal {F}^{0}_{v}}\texttt{err}(f^{0}_{v})$. For $\epsilon \in (0,\frac{1}{2})$ and $\delta \in (0,1)$, if Definitions~\ref{def:Information} and \ref{def:KLdivergence} hold, $|U| = O(\frac{d_{v}\ln\frac{1}{\delta}}{\epsilon^{2}})$,
$\mu(\gamma^{0}_{1}, \gamma^{0}_{2}, \mathcal {F}^{0}_{1}, \mathcal
{F}^{0}_{2}) > \frac{\sqrt[3]{n^{2}|L|}-|L|}{n-|L|}$ and $|T_{i}|>0$ for $i\geq1$ until $|\varrho_{i}|\!=\!n$, with probability $1-\delta$ the outputs $\mathcal {F}^{C}_{1}$ and $\mathcal {F}^{C}_{2}$ in Algorithm \ref{alg:insufficient-co-training2} satisfy $\texttt{err}(\mathcal {F}^{C}_{v})=\max_{f_{v}\in\mathcal {F}^{C}_{v}}\texttt{err}(f_{v})\leq \eta_{v}+\epsilon$.
\end{theorem}

\begin{proof}
Since $L$ is an i.i.d sample and $m_{0}=|L|$, for the hypothesis set $\mathcal {F}^{r_i}_{v}$ minimizing the empirical risk on an i.i.d sample with size of $m_{i}=|\varrho_{i}|$, with the assumption that the maximum error rate of the hypotheses minimizing the empirical risk on the i.i.d sample $\Omega$ is proportional to $\frac{1}{\sqrt{|\Omega|}}$, we have
\begin{eqnarray}
\nonumber
  \max_{f^{r_i}_{v}\in\mathcal{F}^{r_i}_{v}}\texttt{err}(f^{r_i}_{v})-
  \texttt{err}(f_{v}^{*})
  =\frac{\sqrt{m_{0}}}{\sqrt{m_{i}}}\big(\max_{f^{0}_{v}\in\mathcal
  {F}^{0}_{v}}\texttt{err}(f^{0}_{v})-\texttt{err}(f_{v}^{*})\big) =\frac{\sqrt{m_{0}}}{\sqrt{m_{i}}}(R_{v}-\eta_{v}).
\end{eqnarray}
If $\gamma^{0}_{v}=C^{\mathfrak{L}}_{v}(R_{v}-\eta_{v})+\gamma'_{v}$, with the proof in Theorem~\ref{theorem:insufficient-theorem2} we know that the pseudo-label of any example in $\varrho_{1}$ is the same as the underground label. Since $L\cup U$ is a large i.i.d sample from the marginal distribution $\mathcal {D}_{\mathcal {X}}$, so with Definition~\ref{def:samplebias} we get
\begin{eqnarray}
\nonumber
  \max_{f^{1}_{v}\in\mathcal {F}^{1}_{v}}\texttt{err}(f^{1}_{v})-\texttt{err}(f_{v}^{*})\leq
  \frac{\sqrt{m_{0}}}{\sqrt{m_{1}}}(R_{v}-\eta_{v})\cdot \exp\big(\ln\frac{n}{m_{1}}\big).
\end{eqnarray}
For $f^{1}_{v}\in\mathcal {F}^{1}_{v}$, if
\begin{eqnarray}
\nonumber
  |f^{1}_{v}(x_{v})|\geq\gamma^{1}_{v}=\gamma^{0}_{v}-C^{\mathfrak{L}}_{v}(R_{v}
  -\eta_{v})\big(1-\frac{n\sqrt{m_{0}}}{m_{1}\sqrt{m_{1}}}\big),
\end{eqnarray}
with Definition~\ref{def:marginLipschitz} we get $|f_{v}^{*}(x_{v})|\geq \gamma'_{v}$ and
$sign\big(f^{1}_{v}(x_{v})\big)=sign\big(f_{v}^{*}(x_{v})\big)$. With Definition~\ref{def:Information} we know
$sign\big(f_{v}^{*}(x_{v})\big)=y$. Thus, the pseudo-label of any example in $\varrho_{2}$ is the same as the underground label. Similarly, for $f^{i}_{v}\in\mathcal {F}^{i}_{v}$, if
\begin{eqnarray}
\nonumber
 |f^{i}_{v}(x_{v})|\geq\gamma^{i}_{v}=\gamma^{0}_{v}-C^{\mathfrak{L}}_{v}(R_{v}
 -\eta_{v})\big(1-\frac{n\sqrt{m_{0}}}{m_{i}\sqrt{m_{i}}}\big),
\end{eqnarray}
we get $sign\big(f^{i}_{v}(x_{v})\big)=y$. If $T_{i}\neq\emptyset$ until
$|\varrho_{i}|=|L|+|U|$, all instances in $U$ are labeled with underground labels. So $\varrho_{i}$ is an i.i.d sample with size of $|L|+|U|$. Since
$|U|=O\big(\frac{d_{v}\ln\frac{1}{\delta}}{\epsilon^{2}}\big)$, we get
$\max_{f_{v}\in\mathcal {F}^{C}_{v}}\texttt{err}(f_{v})\leq \eta_{v}+\epsilon$ with probability
$1-\delta$. If we want $\gamma^{1}_{v}<\gamma^{0}_{v}$, $1-\frac{n\sqrt{m_{0}}}{m_{1}\sqrt{m_{1}}}$
must be larger than $0$, i.e., $m_{1}>\sqrt[3]{n^{2}m_{0}}$. It implies that
$m_{0}+\mu(\gamma^{0}_{1}, \gamma^{0}_{2}, \mathcal {F}^{0}_{1}, \mathcal
{F}^{0}_{2})|U|>\sqrt[3]{n^{2}m_{0}}$, so we get $\mu(\gamma^{0}_{1}, \gamma^{0}_{2}, \mathcal {H}^{0}_{1},\mathcal {H}^{0}_{2})>\frac{\sqrt[3]{n^{2}|L|}-|L|}{n-|L|}$.
\end{proof}

\noindent \textbf{Remark:} Theorem~\ref{theorem:insufficient-theorem3} states that if the two views provide diverse information, i.e., $\mu(\gamma^{0}_{1},
\gamma^{0}_{2}, \mathcal {F}^{0}_{1}, \mathcal
{F}^{0}_{2})>\frac{\sqrt[3]{n^{2}|L|}-|L|}{n-|L|}$ ($\mathcal {F}^{0}_{1}=\mathcal {F}^{L}_{1}$,
$\mathcal {F}^{0}_{2}=\mathcal {F}^{L}_{2}$), co-training could improve the performance of weak hypotheses by exploiting unlabeled data. This result tells that the diverse information between the two views plays an important role in co-training with insufficient views.

\subsection{Assumption Relaxation and Discussions}
Our result is based on a little bit strong \textit{Margin Lipschitz} assumption, which is caused by the fact that the learning task with insufficient views for semi-supervised learning is very difficult. In this section, we try to give a heuristic analysis for the case where the \textit{Margin Lipschitz} assumption is relaxed. Instead, we give the following \textit{Probabilistic Margin} assumption: for $\frac{1}{2} \leq\gamma_{v}\leq 1$ ($v=1,2$),
\begin{eqnarray}
\nonumber
  P_{\langle x_{1}, x_{2}\rangle\in \mathcal {X}}
  \big\{x_{v}:|h_{v}(x_{v})|\geq\gamma_{v} \wedge sign(h_{v}(x_{v}))\neq y\big\}
  \leq\phi(\gamma_{v}).
\end{eqnarray}
Here $\phi: [\frac{1}{2},1]\rightarrow [0,1]$ is a monotonically decreasing function, e.g., $\phi(\gamma)=\beta\ln(\frac{1}{\gamma})$ for some parameter $\beta$. \textit{Probabilistic Margin} assumption allows for small label noise in the examples labeled with large margin. Considering the worst case of the influence of label noise, i.e., the examples with noisy labels are completely inconsistent with the optimal classifier, it can be found that when the two views provide diverse information, co-training could output the hypotheses whose error rate are close to $\eta_{v}+\beta\ln(\frac{1}{\gamma_{v}})$, which is smaller than the error rate of the classifier trained only on the small initial labeled data set $L$. This shows that co-training could improve learning performance by exploiting unlabeled data even with insufficient views.

Now we discuss what influence the view insufficiency will bring to the learning process. Since we could not know the distribution and the posterior probability $\varphi_{v}(x_{v})$ ($v=1,2$) of the example space in advance, it is difficult to analyze the general case. We focus on the famous Tsybakov condition \cite{Tsybakov04} case that for some finite $C_{v}^{0} >
0$, $k > 0$ and $0 < t \leq 1/2$,
\begin{eqnarray*}
\nonumber
  P_{\langle x_{1},x_{2}\rangle\in \mathcal {X}}
  \big(|\varphi_{v}(x_{v})\!-\!1/2|\leq t\big)\leq C_{v}^{0}t^{k},
\end{eqnarray*}
where small $k$ implies large view insufficiency $\Upsilon$, and give a heuristic analysis to
illuminate the relationship between view insufficiency and diversity. Considering the worst case of Tsybakov condition for the fixed parameter $k$, i.e., $P\big(\big|\varphi_{v}(x_{v})-1/2\big|\leq t\big)= C_{v}^{0}t^{k}$,
we get
$P\big(|2\varphi_{v}(x_{v})-1|>\gamma\big)= 1-C_{v}^{0}(\frac{\gamma}{2})^{k}$
for $0<\gamma\leq1$. $\big|2\varphi_{v}(x_{v})-1\big|$ is the output margin of the optimal classifier $f_{v}^{*}$, with the intuition that similar hypotheses tend to have similar margin outputs, the magnitude of the instances with margin larger than $\gamma$ in view $\mathcal {X}_{v}$ is probably $\alpha\big(1-C_{v}^{0}(\frac{\gamma}{2})^{k}\big)$ for some parameter $\alpha$.
$\mu_{v} \approx \alpha (1-C_{v}^{0}(\frac{\gamma_{v}}{2})^{k})$ quantifies the amount of instances labeled with large margin by view $\mathcal {X}_{v}$. Considering that $\mu=\frac{\nu+\mu_{1}+\mu_{2}}{2}$, when $\nu$ is fixed, if the view insufficiency increases, the confident information $\mu$ provided by the two views decreases; when $\mu_{1}$ and $\mu_{2}$ are fixed, if the two views have large margin-based disagreement $\nu$, the confident information $\mu$ provided by the two views increases, which shows that the margin-based disagreement $\nu$ is important to co-training. For understanding the magnitude of $\mu$ better, we give the following example. There are adequate unlabeled instances in real-world semi-supervised applications, suppose we have $n=|L|+|U|=1000$ and $L=12$, similarly to the empirical study on co-training in \cite{Blum:Mitchell1998}, $\mu=\frac{\sqrt[3]{n^{2}|L|}-|L|}{n-|L|}$ at the first step in Theorem~\ref{theorem:insufficient-theorem3} should be 22\%. With respect to $\mu=\frac{\nu+\mu_{1}+\mu_{2}}{2}$, if the two views provide diverse information ($\nu$ is large), the weak hypothesis in each view predicting about 18\% (even less) of the unlabeled instances with large margin might be enough to guarantee that $\mu\geq22\%$, which is common in real-world applications.

In our result, the margin threshold
$\gamma_{v}=C^{\mathfrak{L}}_{v}(R_{v}-\eta_{v})+\gamma'_{v}$ depends on several parameters. Generally, the optimal classifier would make mistakes only when the instances are close to the boundary, i.e., $P(y=+1|x)$ is close to $1/2$. So $\gamma'_{v}$ is close to $0$. $(R_{v}-\eta_{v})$ depends on the number of initial labeled data $L$ and is proportional to $1/\sqrt{|L|}$. So when $|L|\approx4(C^{\mathfrak{L}}_{v})^{2}C\big(d_{v}+\ln(\frac{1}{\delta})\big)$,
$C^{\mathfrak{L}}_{v}(R_{v}-\eta_{v})$ is close to $1/2$. Thus, $\gamma_{v}$ is close to $1/2$.

\subsubsection{Connection to Co-Regularization}\label{sec:connection-co-Regularization}
In semi-supervised learning, co-regularization allows
for views with partial insufficiency, but it assumes that the two views provide almost the same information. Unfortunately, in real-world applications each view may be corrupted by different kinds of noise, it is unreasonable to assume that the two views provide almost the same information. When the two views are corrupted by different noise or provide diverse information, the two optimal classifiers are no longer compatible with each other and the performance of co-regularization will be influenced since it strongly encourages the agreement between two views.

Sridharan and Kakade \cite{Sridharan2008} used the conditional mutual information $\emph{I} (A:B|C)$ to measure how much knowing A reduces the uncertainty of B conditionally on already knowing C, they assumed that $\emph{I} (\mathcal {Y}:\mathcal {X}_{v}|\mathcal {X}_{3-v})\leq\epsilon_{info}$ ($v=1,2$) holds for some small $\epsilon_{info}>0$, and provided an information theoretic framework for co-regularization which minimizes the following co-regularized loss for the pair $(f_{1},f_{2})$ ($f_{v} \in \mathcal {F}_{v}$).
\begin{eqnarray}
\nonumber
    Loss_{co}(f_{1},f_{2})=\frac{1}{2}\big(\widehat{R}_{L}(f_{1})+\widehat{R}_{L}(f_{2})\big)
    +\alpha_{1}\|f_{1}\|+\alpha_{2}\|f_{2}\|+\alpha_{3}\widehat{D}_{U}(f_{1},f_{2})
\end{eqnarray}
$\widehat{R}_{L}$ is the empirical risk with respect to the labeled data $L$ and $\widehat{D}_{U}$ is the empirical disagreement with respect to the unlabeled data $U$. Note that $\emph{I} (\mathcal {Y}\!:\!\mathcal {X}_{v}|\mathcal {X}_{3-v})\leq\epsilon_{info}$ means that if we already knew view $\mathcal {X}_{v}$
then there is little more information that we could get from view $\mathcal {X}_{3-v}$ about $\mathcal {Y}$, i.e., the two views provide almost the same information. Sridharan and Kakade \cite{Sridharan2008} showed that the excess error between the output hypothesis of co-regularization and the optimal classifier is punished by the term $\sqrt{\epsilon_{info}}$. This implies that it is hard for co-regularization to find the $\epsilon$-approximation of the optimal classifier when the two views are insufficient and provide diverse information. Now we give the following Proposition~\ref{pro:coregularization} to show that co-regularization may never output the approximations of the optimal classifier.

\begin{proposition}\label{pro:coregularization}
Suppose $\|f_{v}\|=1$ for $f_{v}\in\mathcal {F}_{v}$ ($v=1,2$). Let $\mathcal {F}_{v}^{L}\subset \mathcal {F}_{v}$ denote the hypotheses minimizing the empirical risk on the labeled data $L$ and $(g_{1}, g_{2})=\arg\min_{f_{v}\in\mathcal {F}_{v}^{L}}d(f_{1},f_{2})$. If $|U|$ is sufficiently large, $Loss_{co}(g_{1}, g_{2})$ is no larger than $Loss_{co}(f_{1}^{*}, f_{2}^{*})$.
\end{proposition}

\begin{proof}
Considering that $\widehat{R}_{L}(g_{v})=\widehat{R}_{L}(f^{*}_{v})$ and that $\widehat{D}_{U}\big(g_{1},g_{2}\big)\leq
\widehat{D}_{U}\big(f_{1}^{*},f_{2}^{*}\big)$ holds for sufficiently large $|U|$, it is easy to get Proposition \ref{pro:coregularization} proved.
\end{proof}

\noindent \textbf{Remark:} Let us give an intuitive explanation to Proposition~\ref{pro:coregularization}. It states that co-regularization prefers to output a pair of hypotheses which minimizes the disagreement on the unlabeled data rather than the optimal classifier. Its performance will be influenced by the incompatibility between the two views, especially when the unlabeled data are very large while the labeled data are small. It might contribute to understanding the difference between co-regularization and co-training. For two views which provide almost the same information, the optimal classifiers in the two views are compatible with each other, and co-regularization could find the optimal classifiers by minimizing the error rate on labeled data and the disagreement on unlabeled data over two views; for two views which provide diverse information, we show that co-regularization may fail, while co-training which iteratively utilizes the confident information in one view to help the other is a good learning strategy.

\section{Why Combination of Classifiers is Good}\label{sec:combination}
Usually, the two classifiers in disagreement-based approaches are combined to make predictions in practice, e.g., the two classifiers in co-training \cite{Blum:Mitchell1998} are combined by multiplying the posterior probabilities and the empirical results showed that the combination is better than the individual classifiers. However, there is no theoretical study to explain why and when the combination can be better than the individual classifiers.

Let $h_{com}$ denote the combination of the individual classifiers $h_{1}$ and $h_{2}$,
then the combination $h_{com}$ by multiplying the posterior probabilities in \cite{Blum:Mitchell1998}
is formulated as
\begin{eqnarray}\label{combination2}
    h_{com}(x) = \left\{ \begin{array}{ll}
    +1 & \textrm{if $P(h_{1}=+1|x)P(h_{2}=+1|x)$}\\
       & \textrm{~~~$>P(h_{1}=-1|x)P(h_{2}=-1|x)$}\\
    -1 & \textrm{if $P(h_{1}=+1|x)P(h_{2}=+1|x)$}\\
       & \textrm{~~~$<P(h_{1}=-1|x)P(h_{2}=-1|x)$}\\
    0 & \textrm{otherwise}
    \end{array}. \right.
\end{eqnarray}
Actually, this combination strategy can be generalized to any margin-based classifiers. Let $\mathcal {F}$: $\mathcal {X}\rightarrow [-1,+1]$ denote the hypothesis space, the classification rule on $x\in\mathcal {X}$ induced by a hypothesis $f\in \mathcal {F}$ is $sign\big(f(x)\big)$ and $|f(x)|$ is the margin of $f$ on $x$. Let $P(y=+1|x)=\frac{1+f(x)}{2}$, the classification rule
$sign\big(f(x)\big)$ is equal to maximizing the posterior probability
$P(y=+1|x)$, i.e., $f(x)>0\Leftrightarrow P(y=+1|x)>\frac{1}{2}$. So
the combination strategy in Equation~\ref{combination2} for $f_{1}$ and $f_{2}$ is formulated as
\begin{eqnarray}\label{combination3}
    f_{com}(x) = \left\{ \begin{array}{ll}
    +1 & \textrm{if $f_{1}(x)+f_{2}(x)>0$}\\
    -1 & \textrm{if $f_{1}(x)+f_{2}(x)<0$}\\
    0 & \textrm{otherwise}
    \end{array}, \right.
\end{eqnarray}
i.e., $f_{com}(x)=sign(f_{1}(x)+f_{2}(x))$. It implies that $f_{com}$ follows the decision of the hypothesis which has larger margin and the error rate of $f_{com}$ is:
\begin{equation}
\nonumber
    \texttt{err}(f_{com})=P_{(x,y)\in\mathcal {D}}\big(f_{com}(x)\neq y\big)= P\big((f_{1}(x)+f_{2}(x))\cdot y\leq 0\big).
\end{equation}
Considering the construction of $f_{com}$, it is easy to find that when $f_{1}(x)$ and $f_{2}(x)$ make
the same prediction on $x$, $f_{com}$ will follow the both's decision. Let
$DIS(f_{1},f_{2})=\{x\in \mathcal {X}: sign(f_{1}(x))\neq sign(f_{2}(x))\}$, i.e., the disagreed instance set by $f_{1}$ and $f_{2}$, the error rate of $f_{com}$ depends on its performance on $DIS(f_{1},f_{2})$ and the following Proposition~\ref{pro:lower-bound} holds for $f_{com}$.
\begin{proposition}\label{pro:lower-bound}
The error rate of $f_{com}$ satisfies the following lower bound:
\begin{eqnarray}
\nonumber
    \texttt{err}(f_{com})&\!\geq\!& P_{(x,y)\in\mathcal {D}}
    (sign(f_{1}(x))\neq y \wedge sign(f_{2}(x))\neq y)\\
\nonumber
    &\!=\!&\frac{\texttt{err}(f_{1})+\texttt{err}(f_{2})-d(f_{1},f_{2})}{2}.
\end{eqnarray}
\end{proposition}

\begin{proof}
The worst case of $f_{com}$ is making incorrect predictions for all instances  in $DIS(f_{1},f_{2})$, it is easy to get Proposition~\ref{pro:lower-bound} proved.
\end{proof}

\noindent \textbf{Remark:} Proposition~\ref{pro:lower-bound} states that larger disagreement will lead to better lower bound for the combination of the individual classifiers.

Now we study the performance of $f_{com}$ on $DIS(f_{1},f_{2})$. For any
$x\in DIS(f_{1},f_{2})$, $f_{1}$ and $f_{2}$ make different predictions on $x$.
Let $f_{G}(x)$ denote the output of the hypothesis which makes the correct prediction and let
$f_{R}(x)$ denote the output of the hypothesis which makes the incorrect prediction, i.e.,
\begin{eqnarray}
\nonumber
    f_{G}(x) = \left\{ \begin{array}{ll}
        f_{1}(x) & \textrm{if $f_{1}(x)\cdot y>0$}\\
        f_{2}(x) & \textrm{otherwise}
    \end{array} \right. \textrm{~and~~~}
    f_{R}(x) = \left\{ \begin{array}{ll}
        f_{1}(x) & \textrm{if $f_{1}(x)\cdot y<0$}\\
        f_{2}(x) & \textrm{otherwise}.
    \end{array} \right.
\end{eqnarray}
When $|f_{G}(x)|$ is larger than $|f_{R}(x)|$, the combination $f_{com}$ gives the correct prediction on $x$. Intuitively, $f_{com}$ works in the following way: if the margin (the
confidence) is reliable, i.e., large margin implies high label quality,
the incorrect prediction happens on the instance which has small margin, i.e., $|f_{R}(x)|$ is small. If $f_{1}$ and $f_{2}$ are not very related to each other, the probability that both hypotheses have small margin on $x$ is small, i.e., $|f_{G}(x)|$ is large with great probability. Thus, $|f_{G}(x)|>|f_{R}(x)|$ holds with great probability and $f_{com}$ gives the correct prediction on $x$. Define the following confidence gain $C_{G}(f_{1},f_{2})$ and confidence risk $C_{R}(f_{1},f_{2})$:
\begin{eqnarray}
\label{gainofconfidence}
  C_{G}(f_{1},f_{2})&\!=\!&\int_{x\in DIS(f_{1},f_{2})}|f_{G}(x)|p(x)dx,\\
\label{riskofconfidence}
  C_{R}(f_{1},f_{2})&\!=\!&\int_{x\in DIS(f_{1},f_{2})}|f_{R}(x)|p(x)dx.
\end{eqnarray}
$C_{G}(f_{1},f_{2})$ is the integral correct margin of $f_{1}$ and $f_{2}$ over $DIS(f_{1},f_{2})$, while
$C_{R}(f_{1},f_{2})$ is the integral incorrect margin of $f_{1}$ and $f_{2}$ over $DIS(f_{1},f_{2})$. If $C_{G}(f_{1},f_{2})$ is much larger than $C_{R}(f_{1},f_{2})$,
$|f_{G}(x)|>|f_{R}(x)|$ may hold with great probability for $x\in DIS(f_{1},f_{2})$.

The margin-based classifiers try to classify instances correctly with large margins, however, there may exist some instances on which the margins are small, e.g., the instances close to the boundary. Intuitively, these instances with small margins are not too many, maybe bounded by some function. Suppose that the distribution of $|f_{G}(x)|$ over $DIS(f_{1},f_{2})$
satisfies the condition: for $C_{T}>0$, $k\geq0$ and all $0<t<1$ such that
\begin{eqnarray}\label{eq:distribution_fg}
    P_{x\in DIS(f_{1},f_{2})}(|f_{G}(x)|<t)\leq C_{T}\cdot t^{k}.
\end{eqnarray}
It indicates that the amount of correctly classified instances with small margins is bounded by a polynomial function, and larger $k$ will lead to less small margins. This condition is inspired by the famous Tsybakov condition \cite{Tsybakov04} for characterizing the distribution of underlying small margins. For the instances in the disagreed region $DIS(f_{1},f_{2})$, the individual classifiers make different predictions on them, it is reasonable to assume that the individual classifiers make predictions on these disagreed instances independently. Now we provide an upper bound on the error rate of the combination.

\begin{theorem}\label{theorem:combination}
Suppose the individual classifiers make predictions on the instances in the disagreed region $DIS(f_{1},f_{2})$ independently and the Tsybakov condition in Equation~\ref{eq:distribution_fg} holds, the following bound on the error rate of $f_{com}$ holds.
\begin{eqnarray*}
    P_{(x,y)\in \mathcal {D}}(f_{com}(x)\neq y) &\!\leq\!&\frac{\texttt{err}(f_{1})+\texttt{err}(f_{2})-d(f_{1},f_{2})}{2}\\
    &\!\!&+C_{T}\cdot\int_{x\in DIS(f_{1},f_{2})}|f_{R}(x)|^{k}p(x)dx.
\end{eqnarray*}
\end{theorem}

\begin{proof}
For $x\in DIS(f_{1},f_{2})$, without loss of generality, we assume that $f_{1}$ gives the correct prediction while $f_{2}$ gives the incorrect prediction, with the assumption that $f_{1}$ and $f_{2}$ make predictions on it independently and the condition in Equation~\ref{eq:distribution_fg} we get
\begin{eqnarray*}
    P(\mathbb{I}(f_{com}(x)\neq y))&\!=\!&P(\mathbb{I}(|f_{1}(x)|<|f_{2}(x)|))\\
    &\!=\!&P(\mathbb{I}(|f_{G}(x)|<|f_{R}(x)|))\\
    &\!\leq\!& C_{T}\cdot|f_{R}(x)|^{k}.
\end{eqnarray*}
Then, we get
\begin{eqnarray*}
    P_{x\in DIS(f_{1},f_{2})}(f_{com}(x)\neq y)&\!=\!&\int_{x\in DIS(f_{1},f_{2})} P(\mathbb{I}(f_{com}(x)\neq y))p(x)dx\\
    &\!\leq\!& C_{T}\cdot\int_{x\in DIS(f_{1},f_{2})}|f_{R}(x)|^{k}p(x)dx.
\end{eqnarray*}
It is easy to find that
\begin{eqnarray*}
    &&P_{(x,y)\in\mathcal {D}}(f_{com}(x)\neq y)\\
    &&=P_{(x,y)\in\mathcal {D}}(sign(f_{1}(x))\neq y \wedge
    sign(f_{2}(x))\neq y)+P_{x\in DIS(f_{1},f_{2})}(f_{com}(x)\neq y)\\
    &&\leq\frac{\texttt{err}(f_{1})+\texttt{err}(f_{2})-d(f_{1},f_{2})}{2}+
    C_{T}\cdot\int_{x\in DIS(f_{1},f_{2})}|f_{R}(x)|^{k}p(x)dx.
\end{eqnarray*}
\end{proof}

\noindent \textbf{Remark:} Let us give a comprehensive explanation to Theorem~\ref{theorem:combination}: for $k=1$, $\int_{x\in DIS(f_{1},f_{2})}|f_{R}(x)|^{k}p(x)dx=C_{R}(f_{1},f_{2})$, it indicates that if the individual classifiers have large disagreement $d(f_{1},f_{2})$ and small confidence risk $C_{R}(f_{1},f_{2})$, the combination will have low error rate. It implies that for good combination, incorrect predictions should have small margins. The following Corollary~\ref{cor:corollary1} shows when the combination is better than the individual classifiers.

\begin{corollary}\label{cor:corollary1}
Suppose the individual classifiers make predictions on the instances in the disagreed region $DIS(f_{1},f_{2})$ independently, the condition in Equation~\ref{eq:distribution_fg} holds and $\texttt{err}(f_{1})\leq\texttt{err}(f_{2})$, if
\begin{eqnarray*}
    C_{T}\cdot\int_{x\in DIS(f_{1},f_{2})}|f_{R}(x)|^{k}p(x)dx<
    \frac{\texttt{err}(f_{1})-\texttt{err}(f_{2})+d(f_{1},f_{2})}{2},
\end{eqnarray*}
$\texttt{err}(f_{com})$ is smaller than $\texttt{err}(f_{1})$.
\end{corollary}

\noindent \textbf{Remark:} Corollary~\ref{cor:corollary1} states that if the confidence risk $C_{R}(f_{1},f_{2})$ is small (depending on the disagreement between the individual classifiers), i.e., incorrect predictions have small margins, the combination is better than the individual classifiers.

\section{Discussion and Conclusion}\label{sec:conclusions}

The disagreement-based semi-supervised learning \cite{Zhou:Li2009,ZhouZH2008} was named to assemble the approaches which generate multiple weak classifiers and let them label unlabeled instances to augment the training data, including co-training and single-view disagreement-based algorithms. In these approaches, unlabeled data serve as a kind of ``platform'' for information exchange and the disagreement among multiple weak classifiers is exploited during the learning process. If there is no disagreement, the learning process will degenerate into self-training. In this article, we aim at presenting a theoretical foundation of disagreement-based approaches.

One basic issue of the theoretical foundation is why and when the disagreement-based approaches could improve learning performance by exploiting unlabeled data. In Section~\ref{sec:upper-bound}, we provide the theoretical analysis on disagreement-based approaches and give bounds on the error rates of the classifiers in the learning process. Furthermore, we prove that the disagreement will decrease after the disagreement-based process is initiated. Based on these results, it can be found that the disagreement-based approaches could improve learning performance given that the two initial classifiers trained with the initial labeled data have large disagreement. The empirical results in Section~\ref{sec:disagreement-improvement} verify that larger disagreement will lead to better performance improvement.

For disagreement-based approaches, it is often observed that the performance of the classifiers cannot be improved further after a number of rounds in empirical studies. We prove that the disagreement and error rates of the classifiers will converge after a number of rounds in Section~\ref{sec:no-further-improvement}, which provides the theoretical explanation to the observation. The empirical results in Section~\ref{sec:convergence-dis-error} validate that the disagreement
between the classifiers will decrease or converge as the learning process goes on. When the disagreement converges, the error rates of the classifiers also seem to converge, e.g., Figure~\ref{Estimated-Round-b}(d) to (f); while when the disagreement does not converge, the error rates of the classifiers seem to decrease as the disagreement decreases, e.g., Figures~\ref{Estimated-Round-b}(g) to (i).

It will be an impressive result if the sufficient and necessary condition for disagreement-based approaches can be found. Toward this direction, we present a theoretical graph-based analysis on co-training in Section~\ref{sec:sufficient-necessary}, in which the classifier in each view is viewed as label propagation and thus co-training is viewed as a combinative label propagation over two views. Based on this analysis, we get the sufficient and necessary condition for co-training. Note that such graph-based analysis on co-training only cares the two graphs rather than where these graphs come from, and therefore it is also applicable to single-view disagreement-based approaches when there is only one view but two graphs can be obtained in different distance matrices. Recall that the analysis in Section~\ref{sec:upper-bound} provides a sufficient condition for disagreement-based approaches, however, it is different from the analysis in Section~\ref{sec:sufficient-necessary}. The analysis in Section~\ref{sec:sufficient-necessary} focuses on the transductive setting, since studying the sufficient and necessary condition is a very hard problem; while the analysis in Section~\ref{sec:upper-bound} focuses on the non-transductive setting, which is applicable to general learning process.

All previous theoretical analyses on co-training assumed that each view is sufficient to learn the target concept, however, in many real-world applications, due to feature corruption or various feature noise, neither view can provide sufficient information. So we present a theoretical analysis on co-training with insufficient views which is much more challenging but practical in Section~\ref{sec:insufficient-co-training}, especially when the two views provide diverse information. We prove that if the two views have large margin-based disagreement, co-training could succeed in outputting the approximation of the optimal classifier by exploiting unlabeled data even with insufficient views. We also give some implications for understanding the difference between co-training and co-regularization. In the analysis of Section~\ref{sec:insufficient-co-training}, we focus on the margin-based classifiers and assume that large margin leads to high label quality, since co-training with insufficient views is a much harder problem, we need some prior knowledge about how much information each instance can provide to learn the target concept.

Since the two classifiers in disagreement-based approaches are usually combined to
make predictions, we present a theoretical analysis to explain why and when the combination can be better than the individual classifiers in Section~\ref{sec:combination}. We focus on the margin-based classifiers and prove that when the individual classifiers have large disagreement, diverse margin output and small confidence risk, i.e., incorrect predictions have small margins, the combination would have low error rate.

Our theoretical result in Section~\ref{sec:general-analysis} indicates that data with two views are not necessary to improve learning performance by exploiting unlabeled data, but it does not mean that we do not need two views at all. When the data have two views, we can get better result. For example, if the data have two conditionally independent views, a single labeled example is sufficient to find the target concept \cite{Balcan:Blum2010,Zhou:Zhan:Yang2007}. It is noteworthy that in previous semi-supervised learning studies, the disagreement-based and graph-based approaches were developed separately. While our theoretical result in Section~\ref{sec:sufficient-necessary} provides a possibility of bringing them into a unified framework, which will be an interesting research direction.


\bibliography{analyze-cotraining}
\bibliographystyle{elsarticle-num}
\end{document}